\documentclass{article}

\usepackage{amsmath,amsfonts,bm}

\def\eqref#1{equation~\ref{#1}}

\def\1{\bm{1}}

\def\rva{{\mathbf{a}}}
\def\rvb{{\mathbf{b}}}
\def\rvc{{\mathbf{c}}}

\def\rvg{{\mathbf{g}}}

\def\rvu{{\mathbf{i}}}

\def\rvk{{\mathbf{k}}}

\def\rvo{{\mathbf{o}}}

\def\rvr{{\mathbf{r}}}

\def\rvu{{\mathbf{u}}}
\def\rvv{{\mathbf{v}}}

\def\rvx{{\mathbf{x}}}
\def\rvy{{\mathbf{y}}}
\def\rvz{{\mathbf{z}}}

\def\rmA{{\mathbf{A}}}
\def\rmB{{\mathbf{B}}}

\def\rmD{{\mathbf{D}}}
\def\rmE{{\mathbf{E}}}

\def\rmG{{\mathbf{G}}}
\def\rmH{{\mathbf{H}}}
\def\rmI{{\mathbf{I}}}

\def\rmK{{\mathbf{K}}}

\def\rmM{{\mathbf{M}}}

\def\rmT{{\mathbf{T}}}
\def\rmU{{\mathbf{U}}}
\def\rmV{{\mathbf{V}}}

\def\rmX{{\mathbf{X}}}

\def\rmZ{{\mathbf{Z}}}

\def\va{{\mathbf{a}}}
\def\vb{{\mathbf{b}}}
\def\vc{{\mathbf{c}}}

\def\ve{{\mathbf{e}}}

\def\vk{{\mathbf{k}}}

\def\vo{{\mathbf{o}}}

\def\vs{{\mathbf{s}}}

\def\vu{{\mathbf{u}}}
\def\vv{{\mathbf{v}}}
\def\vw{{\mathbf{w}}}
\def\vx{{\mathbf{x}}}

\def\vz{{\mathbf{z}}}

\def\mE{{\bm{E}}}

\def\mH{{\bm{H}}}

\def\mU{{\bm{U}}}

\def\mW{{\bm{W}}}

\DeclareMathAlphabet{\mathsfit}{\encodingdefault}{\sfdefault}{m}{sl}
\SetMathAlphabet{\mathsfit}{bold}{\encodingdefault}{\sfdefault}{bx}{n}

\newcommand{\E}{\mathbb{E}}

\newcommand{\R}{\mathbb{R}}
\newcommand{\Z}{\mathbb{Z}}
\newcommand{\N}{\mathbb{N}}

\newcommand{\Var}{\mathrm{Var}}

\DeclareMathOperator*{\argmax}{arg\,max}

\newcommand{\norm}[1]{\lVert #1 \rVert}

\newcommand{\col}{\begin{pmatrix}}
\newcommand{\con}{\end{pmatrix}}

\newcommand{\nin}{\notin}

\newcommand{\ReLU}{\operatorname{ReLU}}

\newcommand{\poly}{\operatorname{poly}}

\newcommand{\ip}[2]{\langle {#1}, {#2} \rangle}

\newcommand{\Sph}{\mathbb{S}^{d-1}}

\newcommand{\psione}{\psi_{1}}
\newcommand{\psitwo}{\psi_{2}}

\newcommand{\codeMatrix}{\rmH}
\newcommand{\noisyCodeMatrix}{\codeMatrix}
\newcommand{\decoder}{\rmM}
\newcommand{\outEmbedding}{\rvv}
\newcommand{\compressedOutEmbedding}{\outEmbedding}
\newcommand{\auxiliaryEmbedding}{\rvu}
\newcommand{\modelDim}{d}
\newcommand{\numVectors}{{|\rmK|}}
\newcommand{\codeDim}{m}
\newcommand{\gaussianMatrix}{\rmD}
\newcommand{\decodeDiff}{\rva}
\newcommand{\auxHold}{\rvb}
\newcommand{\decodingSphereX}{\sphereX}
\newcommand{\decodingSphereY}{\sphereY}

\newcommand{\outCompEmbedMat}{\rmV}
\newcommand{\targetDirectionMatrix}{\rmU}

\newcommand{\compressedAuxEmbedding}{\auxiliaryEmbedding}
\newcommand{\isoErr}{\rmE}

\newcommand{\keyEmbeddings}{\rvk}
\newcommand{\valueEmbeddings}{\rvv}
\newcommand{\numKV}{F}
\newcommand{\kvDim}{d}
\newcommand{\AOne}{\rmA}
\newcommand{\BOne}{\rmB}
\newcommand{\ATwo}{\rmA'}
\newcommand{\BTwo}{\rmB'}
\newcommand{\invertible}[1]{\mathrm{GL}(#1)}
\newcommand{\auxTheoryEmbed}{\rvu}
\newcommand{\tMatrix}{\rmT}
\newcommand{\tSubK}{\rmT_{\rvk}}
\newcommand{\embeddingC}{\rvc}

\newcommand{\deferredOutEmbedding}{\outEmbedding}
\newcommand{\deferredDecoder}{\decoder}

\newcommand{\sphereX}{\rvx}
\newcommand{\sphereY}{\rvy}
\newcommand{\probDelta}{\delta}
\newcommand{\maxDifference}{\mu}
\newcommand{\tBound}{t}
\newcommand{\normXi}{\rmZ}
\newcommand{\decodingNumVectors}{\numVectors}
\newcommand{\defCompOutEmbedding}{\tilde{\deferredOutEmbedding}}
\newcommand{\fixedCompressed}{\rvz}
\newcommand{\inSPM}{\rvz}

\newcommand{\vSet}{\bm{V}}
\newcommand{\inEmbedding}{\rvx}

\newcommand{\numKeys}{F}
\newcommand{\tokenDim}{d}
\newcommand{\hiddenDim}{h}
\newcommand{\encDim}{m}

\newcommand{\gaussMat}{\rmG}
\newcommand{\gaussRow}[1]{\rmG[#1]}
\newcommand{\gaussEle}[2]{\rmG[#1][#2]}
\newcommand{\target}{\rvo}
\newcommand{\encMat}{\rmE}
\newcommand{\decMat}{\rmD}
\newcommand{\gateMat}{\rmA}
\newcommand{\mMat}{\rmM}
\newcommand{\inVec}{\rvx}
\newcommand{\randVec}{\rvr}
\newcommand{\orthoMat}{\rmU}

\newcommand{\vvec}{\rvv}
\newcommand{\uvec}{\rvu}
\newcommand{\avec}{\rva}
\newcommand{\yvec}{\rvy}
\newcommand{\bvec}{\rvb}
\newcommand{\xvec}{\rvx}
\newcommand{\gvec}{\rvg}

\newcommand{\key}{\rvk}
\newcommand{\canCen}{\rvc}
\newcommand{\normDec}{\rmM}
\newcommand{\enc}{\text{\normalfont\bfseries enc}}
\newcommand{\dec}{\text{\normalfont\bfseries dec}}
\newcommand{\aMat}{\rmA}

\usepackage{microtype}
\usepackage{graphicx}
\usepackage{subcaption}
\usepackage{booktabs}

\usepackage{hyperref}

\usepackage[a4paper,margin=1in]{geometry}

\usepackage{amsmath}
\usepackage{amssymb}
\usepackage{mathtools}
\usepackage{amsthm}

\usepackage{graphicx, subcaption, caption}
\usepackage{wrapfig}
\usepackage{xcolor}
\usepackage{float}
\usepackage{booktabs, makecell}
\usepackage{enumitem}
\usepackage{bbm}
\usepackage{natbib}
\usepackage{algorithmic}
\usepackage{algorithm}
\usepackage{geometry}

\usepackage[capitalize,noabbrev]{cleveref}

\theoremstyle{plain}
\newtheorem{theorem}{Theorem}[subsection]
\newtheorem{definition}[theorem]{Definition}
\newtheorem{corollary}[theorem]{Corollary}
\newtheorem{construction}{Construction}[section]
\newtheorem{lemma}[theorem]{Lemma}

\newtheorem{remark}{Remark}
\newtheorem{proposition}[theorem]{Proposition}

\title{Constructing Efficient Fact-Storing MLPs for Transformers}

\makeatletter
\renewcommand{\thefootnote}{\fnsymbol{footnote}}
\newcommand{\thanksmark}[1]{\textsuperscript{\@fnsymbol{#1}}}
\newcommand{\corrauthormark}{\kern0.35em\textsuperscript{\@fnsymbol{2}}}
\makeatother

\author{
Owen Dugan\textsuperscript{1}
  \thanks{Equal first author}
  \corrauthormark
\quad
Roberto Garcia\textsuperscript{2}
  \thanksmark{1}
\quad
Ronny Junkins\textsuperscript{1}
  \thanksmark{1}
\quad
Jerry Liu\textsuperscript{2}
  \thanksmark{1}
\\[0.3ex]
Dylan Zinsley\textsuperscript{3}\quad
Sabri Eyuboglu\textsuperscript{1}\quad
Atri Rudra\textsuperscript{4}\quad
Chris R\'{e}\textsuperscript{1}
\\[1.5ex]
\textsuperscript{1}Computer Science Department, Stanford University \\
\textsuperscript{2}Institute for Computational \& Mathematical Engineering, Stanford University \\
\textsuperscript{3}Computer Science Department, University of Wisconsin–Madison \\
\textsuperscript{4}Computer Science and Engineering Department, University at Buffalo
}

\makeatletter
\renewcommand{\thefootnote}{\arabic{footnote}}
\makeatother

\begin{document}
\maketitle

\renewcommand{\thefootnote}{\fnsymbol{footnote}}
\footnotetext[2]{Corresponding author: \href{mailto:odugan@stanford.edu}{\texttt{odugan@stanford.edu}}}
\footnotetext[3]{Preprint, working draft version.}

\setcounter{footnote}{0}
\renewcommand{\thefootnote}{\arabic{footnote}}

\begin{abstract}
    The success of large language models (LLMs) can be attributed in part to their ability to efficiently store factual knowledge as key-value mappings within their MLP parameters.
    Recent work has proposed explicit weight constructions to build such fact-storing MLPs, providing an improved understanding of LLM fact storage mechanisms.
    In this paper, we introduce an MLP construction framework that improves over previous constructions in three areas: it 1) works for all but a measure zero set of feasible input-output pairs, 2) achieves asymptotically optimal parameter efficiency matching information-theoretic bounds for some embeddings, and 3) maintains usability within Transformers for factual recall.
    Through our improvements, we 1) discover a metric on value embeddings that characterizes facts-per-parameter scaling for both constructed and gradient-descent-trained MLPs, 2) identify a simple encoder-decoder mechanism that empirically matches gradient-descent MLP facts-per-parameter asymptotics across all the inputs and outputs we test, and 3) uncover a fundamental tradeoff between an MLP's fact-storage capacity and its usability within Transformers.
    Finally, we demonstrate a proof-of-concept application of fact-storing MLPs: modular fact editing on one-layer Transformers by \textit{replacing entire MLPs at once}.
\end{abstract}
\section{Introduction}
\label{sec:intro}

Large language models (LLMs) achieve remarkable performance across domains such as mathematics, science, and law \citep{AlphaProof2024_DMind,guha2023legalbench,saab2024capabilities}, in part because of their ability to store vast amounts of knowledge within their parameters~\citep{petroni2019language,meng2023locatingeditingfactualassociations}. As a result, there has been considerable interest in understanding the mechanism by which LLMs store knowledge.

A body of prior work seeks to understand how and where LLMs store knowledge by probing pretrained LLMs. These works observed that knowledge is often stored within Multi-Layer Perceptrons (MLPs) via key-value mappings (\textit{facts})~\citep{geva2021transformerfeedforwardlayerskeyvalue,dai2022knowledgeneuronspretrainedtransformers} and have explored LLM fact-editing by modifying MLP parameters~\citep{geva2022transformerfeedforwardlayersbuild, meng2023locatingeditingfactualassociations, nanda2023factfinding}. Another line of work measures the empirical fact storage capacity of LLMs~\citep{allenzhu2024physicslanguagemodels33, zucchet2025languagemodelslearnfacts, morris2025languagemodelsmemorize}, observing that their facts-per-parameter scaling is asymptotically optimal.
More recently, \citet{nichani2024understandingfactualrecalltransformers} further the understanding of MLP fact storage by introducing the first construction for fact-storing MLPs that provably comes within a polylog factor of matching the empirical facts-per-parameter scaling of LLMs.

Despite progress from recent constructions, particularly \citet{nichani2024understandingfactualrecalltransformers}, several key questions remain unanswered about the mechanics and properties of MLPs as fact-storage devices:
\begin{enumerate}[label=\textbf{Q\arabic*:}]
    \item \textbf{How do MLP input and output geometries affect fact-storage capacity?} 
    Existing fact-storing MLP constructions~\citep{nichani2024understandingfactualrecalltransformers} assume that inputs and outputs are uniformly distributed, even though MLPs in the wild have uncentered and non-uniform inputs and outputs (\Cref{sec:theory_main}).
    \label{item:desiderata_2}
    
    \item \textbf{How do MLPs achieve parameter-efficient fact-storage?}
    Existing constructions still fall short of explaining the fact-storage efficiency observed in practice. For instance, the theoretical guarantees in \citet{nichani2024understandingfactualrecalltransformers} suggest that their construction stores $O(\log^{11} F)$ fewer facts per parameter than the information-theoretic optimal for a fact set of size $F$.
    \label{item:desiderata_1}
    
    \item \textbf{How do fact-storing MLPs interface with the rest of the Transformer stack?} Prior work focuses on MLP constructions in isolation~\citep{bubeck2020networksizeweightssize, nichani2024understandingfactualrecalltransformers} or the capacity of a full Transformer stack at once \citep{allenzhu2024physicslanguagemodels33}. However, we still lack a clear understanding of how a transformer might learn to perform recall tasks using a fact-storing MLP.
    \label{item:desiderata_3}
\end{enumerate}

\begin{figure*}
    \centering
    \includegraphics[width=\linewidth]{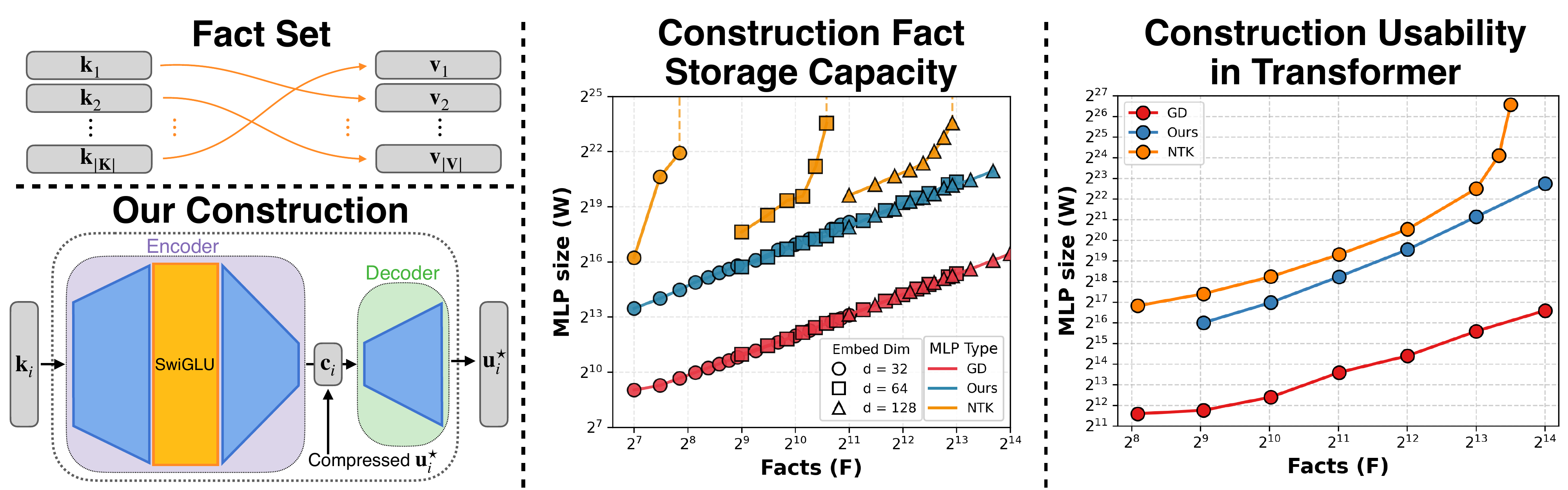}
    \caption{
    (\textbf{Left}) \textit{Top:} We formalize factual knowledge as discrete maps between key and value embeddings. \textit{Bottom:} Our construction consists of an \textit{encoder MLP} that exactly maps keys to compressed intermediate values, and a \textit{decoder linear layer} that linearly decompresses the intermediate values.
    (\textbf{Center}) We compare how the number of parameters ($y$-axis) needed to represent a fact set scales with the number of facts ($x$-axis). Our construction matches gradient-descent trained (GD) MLP asymptotics and requires $5$--$150\times$ fewer parameters than prior constructions.
    (\textbf{Right}) We compare how the number of parameters ($y$-axis) needed for an MLP to represent a fact set in a way that is \emph{usable within a transformer} scales with the number of facts ($x$-axis). Our constructed MLPs exhibit similar asymptotic scaling to GD MLPs, unlike NTK MLPs.
    \textit{Note:} NTK refers to the construction from \citet{nichani2024understandingfactualrecalltransformers}.
    }
    \label{fig:main-figure}
\end{figure*}

We address each of the above questions by improving over existing constructed fact-storing MLPs in a way that uncovers new insights into fact-storing MLPs more broadly. Together, our improvements form an MLP construction framework which produces MLPs that 1) work on all but a measure-zero set of feasible MLP inputs and outputs, 2) match asymptotic information theoretic lower bounds on parameter count for some embeddings, and 3) can be directly used by transformers for factual recall. These improvements allow us to 1) discover a metric on value embeddings that is predictive of MLP facts-per-parameter scaling for both our constructed MLPs and gradient-descent-trained MLPs (GD MLPs), 2) identify a simple encoder-decoder mechanism which is sufficient to empirically match GD MLP facts-per-parameter asymptotics across all of inputs and outputs we test, and 3) identify a fundamental capacity-usability tradeoff for MLPs inside transformers.

\textbf{Q1:}
In \Cref{sec:preconditioning}, 
we study the effect of desired output geometry on MLP capacity.
We improve the construction from \citet{nichani2024understandingfactualrecalltransformers}, improving facts-per-parameter scaling by $2$--$4\times$ and extending it to anisotropic output distributions through an output-whitening procedure.
These improvements provide an insight into MLP scaling: we propose a measure, the \emph{decodability}, which predicts fact-storage capacity for both constructed and GD MLPs with an $R^2$ greater than $97\%$.

\textbf{Q2:}
In \Cref{sec:theory_main}, we improve over existing constructions by providing an MLP construction framework requiring asymptotically fewer parameters than the lowest proven bounds for existing constructions, while also generalizing to nearly all feasible input and output distributions. Our closed-form constructed MLPs match the information-theoretic lower bound for some embeddings, empirically require $5$--$150\times$ fewer parameters than NTK MLPs, and are the first constructed MLPs to match GD MLP asymptotics regardless of input/output dimension.
This construction leads to a key insight about fact-storing MLPs: a simple encoder-decoder MLP framework using dimensionality reduction on the desired MLP outputs (e.g., \citet{johnson1984extensions}) can asymptotically match information-theoretically optimal facts-per-parameter scaling.

\textbf{Q3:}
In \Cref{sec:usability}, we improve existing constructions by identifying a set of modifications to the transformer architecture that enable training a transformer block to use fact storing MLPs for factual recall. We find that our transformer block can use our constructed MLPs, storing an amount of facts per parameter comparable to the information-theoretically optimal, unlike previous constructions.Additionally, we gain insight into fact-storing MLPs interactions with transformers by identifying a fundamental tradeoff between their capacity and usability in transformers.

Finally, in \Cref{sec:fact-editing}, inspired by our results on MLP usability within transformers, we demonstrate modular fact editing in 1-layer transformers as an application of fact-storing MLPs. If, given a transformer block, we modularly swap its fact-storing MLP with another one storing new facts, the transformer outputs the new facts accurately and only increases the cross-entropy loss of non-fact-related tokens by $\sim3$\% \textit{without any additional training}. Further, our modular MLP-swapping approach to fact editing doubles the \textit{fact-editing score} (defined in \Cref{sec:fact-editing}) of SoTA fact-editing weight updates (e.g. MEMIT \cite{memit}, Alpha-Edit \cite{fang2025alphaeditnullspaceconstrainedknowledge}, and ROME \cite{rome}) when editing 10\% of the fact set.

In summary, we present a construction that a) supports a broader class of embeddings than prior constructions, b) produces MLPs with asymptotically fewer parameters than the bounds proven for alternative constructions, and c) produces MLPs that are usable within transformers for factual recall. We use this construction to gain insights into 1) MLP fact-storage capacity's dependence on output geometry, 2) mechanisms behind MLP facts-per-parameter scaling, and 3) the tradeoff between MLP capacity and usability in transformers. By directly constructing MLPs to store facts, we provide a theoretical framework for studying fact storage and a path toward more robust fact manipulation in LLMs.

\section{Preliminaries}
\label{sec:definitions}
\subsection{Definitions}
\label{subsec:definitions_maintext}
We first formalize our notion of factual knowledge, which matches the definitions of~\citet{nichani2024understandingfactualrecalltransformers}. 

\paragraph{Formalizing Factual Knowledge.}
Inspired by prior work~\citep{nichani2024understandingfactualrecalltransformers,zoology,allenzhu2024physicslanguagemodels33}, we define a \textit{fact set} as a discrete mapping between integers. In particular, given a list of keys $K$ and a list of values $V$, a fact set is a function $f: [|K|]\to [|V|]$. For example, given $K = [$``France'', ``USA''$]$ and $V = [$``Washington, D.C.'', ``Paris''$]$, the fact set mapping countries to capitals would be $f(1) = 2,$ $f(2) = 1.$

Although we use human-interpretable examples of key-value maps above, our definition of fact sets applies broadly to transformer tasks. In particular, a language model specifies a fixed vocabulary and encodes maps between tokens as maps between integers, which is also representable in this framework.

Transformers interface with tokens through embedding tables. Motivated by this, we consider \textit{key embeddings} $\mathbf{K}\in \mathbb{R}^{|K| \times d}$ and \textit{value embeddings} $\mathbf{V}\in\mathbb{R}^{|V| \times d}$, which map keys and values, respectively, to vectors. We define $|\mathbf{K}|$ and $|\mathbf{V}|$ as the number of key and value embeddings, respectively, and we denote the $i$th key and value embedding as $\mathbf{k}_i$ and $\mathbf{v}_i$, respectively.
In the case of MLPs within transformers, key and value embeddings come from the internal representations of the surrounding transformer.

\paragraph{Storing a fact set.}
We say that a model $\mathbf{g}_\theta: \mathbb{R}^d\to \mathbb{R}^d$ \textit{stores a fact set} $f: [|\mathbf{K}|]\to [|\mathbf{V}|]$ given embeddings $\mathbf{K}$ and $\mathbf{V}$ if, for all $i \in [|\mathbf{K}|],$ and all $j\neq f(i)\in [|\mathbf{V}|]$,
\begin{equation}
    \label{eq:decoding-criterion-maintext}
    \begin{aligned}
    \langle \mathbf{g}_\theta(\mathbf{k}_i), \mathbf{v}_{f(i)}\rangle > \langle \mathbf{g}_\theta(\mathbf{k}_i), \mathbf{v}_{j}\rangle,
    \end{aligned}
\end{equation}
or, equivalently, $\langle \mathbf{g}_\theta(\mathbf{k}_i), \mathbf{v}_{f(i)} - \mathbf{v}_j\rangle > 0$. In the context of language modeling, this definition is equivalent to outputting the correct value token for each input key token under softmax decoding (see \Cref{sec:info_theory_bound}). For an MLP output $\mathbf{o}$, we refer to $\langle \mathbf{o}, \mathbf{v}_{i}\rangle$ as the \textit{score} of $\mathbf{o}$ with respect to the $i$th value.

We define the \emph{fact-storage cost} of key/value embeddings $\mathbf{K}$ and $\mathbf{V}$ given a model class $\mathbf{g}$ as the smallest number of model parameters needed to store \emph{all possible fact sets} over those embeddings:
\begin{equation}
    W(\mathbf{g}; \mathbf{K}, \mathbf{V}) =
    \min \left\{\#(\theta)\Bigg|\;
    \begin{aligned}
        &\forall f : [|\mathbf{K}|] \to [|\mathbf{V}|], \\
        &\exists\, \theta \; \text{s.t.} \; \mathbf{g}_\theta \text{ stores } f
    \end{aligned}
    \right\}.
    \label{def:complexity}
\end{equation}
A standard information-theoretic lower bound for fact storage cost~\citep{allenzhu2024physicslanguagemodels33}, which we prove for completeness in \Cref{sec:info_theory_bound}, is the following:
\begin{proposition}
    \label{thm: info_bounds_capacity-const}
    Assuming a constant number of bits per parameter, the fact-storage cost of embeddings $\mathbf{K}$ and $\mathbf{V}$ for \emph{any} model family $\mathbf{g}$ satisfies $W(\mathbf{g}; \mathbf{K}, \mathbf{V}) = \Omega(|\mathbf{K}|\log [|\mathbf{V}|])$.
\end{proposition}

Following prior work~\citep{allenzhu2024physicslanguagemodels33, zucchet2025languagemodelslearnfacts}, we define the \emph{fact-storage capacity} of a model as the maximum number of facts it can store for a given number of parameters. See \Cref{sec:info_theory_bound} for a formal definition.

\subsection{Related Work}
\label{sec:related_work}
A first body of prior work has attempted to understand and manipulate LLM knowledge storage by probing pretrained LLMs. \citet{geva2021transformerfeedforwardlayerskeyvalue,geva2022transformerfeedforwardlayersbuild} observed that knowledge is often stored within MLPs via key-value mappings. This discovery sparked a number of studies which attempt to reverse engineer the facts found in MLPs \citep{dai2022knowledgeneuronspretrainedtransformers,nanda2023factfinding}.

After identifying the facts stored by individual LLM MLPs, researchers naturally turned to editing this knowledge. Works such as \citet{dai2022knowledgeneuronspretrainedtransformers, meng2023locatingeditingfactualassociations, memit, model_edit_scaling, gu2024modeleditingharmsgeneral, fang2025alphaeditnullspaceconstrainedknowledge, sun2025mitigatingnegativeinterferencemultilingual} have developed increasingly more accurate, general, and targeted methods for editing of specific facts within LLM MLPs.

Building on the insights from probing LLMs, a second body of work attempts to formalize factual knowledge, often focusing on its scaling. Typically, these works treat knowledge as a key-value store and study the scaling of factual knowledge through associative recall synthetics \citep{allenzhu2024physicslanguagemodels33, zucchet2025languagemodelslearnfacts}, design choices which we also follow. Remarkably, these works consistently find empirically that trained LLMs store facts at the asymptotically optimal rate provided in \Cref{thm: info_bounds_capacity-const} \citep{allenzhu2024physicslanguagemodels33, zucchet2025languagemodelslearnfacts, morris2025languagemodelsmemorize}.

The discovery that trained MLPs store facts at the asymptotically optimal rate raises the question of how MLPs achieve such a scaling. In an attempt to answer this question, \citet{elhage2022toymodelssuperposition} have explored the geometric properties and learning dynamics of MLPs that store a large number of facts. Recently, \citet{nichani2024understandingfactualrecalltransformers} have taken an additional step toward uncovering the mechanisms underlying MLP fact storage; they propose a construction for fact-storing MLPs that comes within a (large) polylog factor of matching the asymptotic fact-scaling of LLM MLPs.

In this work, we improve upon the results of~\citet{nichani2024understandingfactualrecalltransformers} by
a) improving MLP fact-storage cost asymptotics,
b) handling more general input and output embeddings,
and c) enabling constructed MLPs to be usable within transformers.
We use insights from our construction to gain insight into fact-storing MLPs.

\section{Embedding Geometry and Fact-Storage Cost}
\label{sec:preconditioning}
In this section, we investigate how the fact-storage cost of an MLP depends on
the geometry of a fact set's value embeddings. We first gain insight into fact-storing MLPs by developing a metric on the value embeddings which is predictive of MLP fact-storage cost, achieving an $R^2 > 97\%$ (\Cref{subsec: rho}). Further, we use this insight to improve the NTK construction from \citet{nichani2024understandingfactualrecalltransformers}, by generalizing it to non-isotropic embeddings with an embedding by using an embedding whitening procedure. Moreover, we enhance gradient-descent-trained MLPs (GD MLPs), reducing its fact-storage cost for non-isotropic embeddings (\Cref{subsec: whitening}) using the same procedure.

\subsection{A Metric $\rho(\mathbf{V})$ that Predicts Fact-Storage Cost}\label{subsec: rho}
First, we introduce $\rho(\mathbf{V})$ to measure the \textit{decodability} of value embeddings $\mathbf{V}$.
Intuitively, $\rho(\mathbf{V})$ is the minimum normalized margin between the margin-optimal MLP outputs $\mathbf{U^*} \in{\mathbb{R}^{n,d}}$ and the value embeddings $\mathbf{V}\in{\mathbb{R}^{n,d}}$.

\begin{definition}
The decodability $\rho(\mathbf{V})$ of embeddings $\mathbf{V}$ is
\begin{equation}
\label{eq:decodability}
\rho(\mathbf{V}) = \max_{\mathbf{u_i} \in \mathbb{R}^{d}} \left[ \min_{i \neq j} \frac{\langle \mathbf{v}_i - \mathbf{v}_j, \mathbf{u}_i \rangle}{\|\mathbf{u}_i\|_2\|\mathbf{v}_i - \mathbf{v}_j\|_2} \right].
\end{equation}
\end{definition}

Given the margin-optimal output embeddings $\mathbf{u}_i$, $\rho(\mathbf{V})$ measures the minimum margin $\langle \mathbf{u}_i, \mathbf{v}_i \rangle - \langle \mathbf{u}_i, \mathbf{v}_j \rangle$ normalized by $\|\mathbf{u}_i\|_2$ and $\|\mathbf{v}_i - \mathbf{v}_j\|_2$\footnote{
A related notion is the \emph{coherence} of the value embeddings, 
defined as 
$\mu(\mathbf{V}) = \max_{i\neq j} 
\frac{\mid \langle \mathbf{v}_i, \mathbf{v}_j\rangle \mid}
{\|\mathbf{v}_i\|\;\|\mathbf{v}_j\|}$.  
When all $\mathbf{v}_i$ have unit norm, one can show 
that $\rho(\mathbf{V}) \ge \sqrt{[1-\mu(\mathbf{V})]/2}$.  
However, no corresponding \emph{upper} bound on $\rho(\mathbf{V})$ 
in terms of $\mu(\mathbf{V})$ exists in general (Appendix~\ref{sec:rho_relation_coherence}).
Empirically, coherence is not as predictive of fact-storage cost as $\rho(\mathbf{V})$ is
for either our constructed MLPs ($R^2 \approx 0.44$) or GD MLPs ($R^2 \approx 0.10$): see~\Cref{fig:coherence_appendix}.
This helps motivate the use of $\rho(\mathbf{V})$ rather than coherence as the relevant geometric predictor of decoding difficulty.
}. Such a normalization ensures that arbitrary scalings of $\mathbf{u}_i$ or $\mathbf{v}_i$ do not affect the decoding difficulty of $\mathbf{V}$, as one would expect. Notably, the quantity $\rho(\mathbf{V})$ also appears naturally in our decoder construction in \cref{subsec:construction_decoder}. 

\paragraph{$\rho(\mathbf{V})$ predicts fact storage capacity.}

In \Cref{fig:scaling-rho}a, we find empirically that fact-storage cost scales inversely with $\rho$ for both our constructed MLPs (presented in \Cref{sec:theory_main}) and GD MLPs. We show that $\rho$ is predictive of fact set difficulty ($R^2 > 97\%$), as measured by the size of MLP required to store a fact set, for both our constructed MLPs and GD MLPs. This ability to predict capacity for multiple types of fact-storing MLPs suggests that $\rho$ is not a construction-dependent quantity, and that it is instead a property of \textit{near-optimal} fact-storing MLPs.

\subsection{Defining Optimal MLP Outputs}\label{subsec: optimal outputs}
Interestingly, using $\mathbf{u}_i = \mathbf{v}_i$ is generally suboptimal for decoding to index $i$ of $\mathbf{V}$.

As an extreme case, consider the embeddings $\mathbf{v}_1 = \mathbf{e}_1$ and $\mathbf{v}_2 = 2\mathbf{e}_1$. If we wish to select an output that decodes to index 1, outputting $\mathbf{v}_1 = \mathbf{e}_1$ is incorrect and will instead decode to index 2. In fact, outputting $-\mathbf{e}_1$ is optimal, in the sense that it is the unit vector that maximizes the gap between its score with respect to $\mathbf{v}_1$ ($\text{score}_1 = \langle -\mathbf{e}_1, \mathbf{v}_1 \rangle = -1$) and its score with respect to $\mathbf{v}_2$ ($\text{score}_2 = \langle -\mathbf{e}_1, \mathbf{v}_2 \rangle = -2$).

Instead, we can define the margin-optimal output embeddings as the unit $\mathbf{u}_i$ that achieve the maximum value in the definition of $\rho(\mathbf{V})$: \begin{definition}
    \label{def:optimal_output_embeddings}
    The {\upshape margin-optimal output embeddings} (optimal output embeddings for short) $\mathbf{U}^\star\in\mathbb{R}^{|\mathbf{V}|\times d}$ for value embeddings $\mathbf{V}$ is
    \begin{equation}
        \label{eq:decodability}
        \mathbf{u}_i^\star(\mathbf{V}) = \argmax_{\mathbf{u} \in \mathbb{S}^{d-1}}
        \left[\min_{j}
        \frac{\langle \mathbf{v}_i - \mathbf{v}_j, \mathbf{u} \rangle}
        {\|\mathbf{v}_i - \mathbf{v}_j\|_2}\right].
    \end{equation}
\end{definition}

We can obtain $\mathbf{u}_i^\star$ as the solution to a convex program by relaxing the domain to $\|\mathbf{u}_i\|_2 \le 1$ (See Appendix \ref{app:extended_theory}).

Interestingly, $\mathbf{u}_i^\star$ is the spherical Chebyshev center \citep{spherical_chebyshev2024} of the set $S_i = \{\mathbf{v}_i - \mathbf{v}_j \mid j\neq i\}$. Similarly, $\rho(\mathbf{V})$ is the maximum of the spherical Chebyshev radii of the $S_i$. We explore the resulting bounds on $\rho(\mathbf{V})$ in Appendix \ref{app:extended_theory}.

\subsection{Embedding Whitening}\label{subsec: whitening}

Interestingly, the \textit{decodability} $\rho$ is \emph{not} invariant to affine transformations of the value embeddings, but MLPs \emph{are} equivariant to such transformations.
If the MLP $\mathbf{g}(x) = \mathbf{B}\,\mathrm{ReLU}(\mathbf{A}\mathbf{x}+\mathbf{b})$ stores a fact given the value embeddings $\{\mathbf{v}_i\}$,  
then for any invertible affine transformation of the value embeddings\footnote{Here $\mathrm{GL}(d)$ is the usual set of $d \times d$ real valued matrices with non-zero determinant.} \(
    T(\mathbf{v}) = \mathbf{M}\mathbf{v} + \mathbf{c} \,\, \text{for} \,\, \mathbf{M} \in \mathrm{GL}(d),\; \mathbf{c} \in \mathbb{R}^d,
\)
the reparameterized MLP $\tilde{\mathbf{g}}(\mathbf{x}) = \tilde{\mathbf{B}}\,\mathrm{ReLU}(\mathbf{A} \mathbf{x}+\mathbf{b})$ stores the fact set given value embeddings $\{T(\mathbf{v}_i)\}$, where
$\mathbf{\tilde B} = \mathbf{M}^{-1} \mathbf{B}$.\footnote{We prove this for completeness in \Cref{thm:affine_invariance_app}.}

This motivates the following procedure for improving the fact-storage cost of MLPs. 
Given embeddings $\mathbf{V} = \{\mathbf{v}_1, \ldots, \mathbf{v}_n\} \subset \mathbb{R}^d$, we search for an invertible affine transform $T(\mathbf{v})$ that maximizes the decodability of the transformed set:
\begin{equation}
    \label{eq:preconditioning_problem}
    \max_{\mathbf{M} \in \mathrm{GL}(d),\, \mathbf{c} \in \mathbb{R}^d} \;
    \rho (\{\,T(\mathbf{v}_i)\,\}_{i=1}^n).
\end{equation}
Let $\widetilde{\mathbf{V}} = \{\,T(\mathbf{v}_i)\,\}_{i=1}^{|\mathbf{V}|}$ denote the resulting embeddings, 
so that $\rho(\widetilde{\mathbf{V}}) \geq \rho(\mathbf{V})$. 
We then train or construct the MLP on $\widetilde{\mathbf{V}}$, then fold the affine transformation into the network parameters.

We find that a simple heuristic choice of transformation, where $\mathbf{M}$ is the whitening transform of the empirical covariance of $\mathbf{V}$ and $\vc$ is the negative of the mean of $\mathbf{V}$, often improves the decodability: see~\Cref{app:subsec_embeddings_theory} for formal bounds.
We refer to this procedure as \emph{embedding whitening}, and we refer to MLPs trained or constructed with and without embedding whitening as \textit{whitened} and \textit{non-whitened MLPs}, respectively.

\paragraph{Embedding whitening improves fact storage capacity.} In Figure~\ref{fig:scaling-rho}a, we find that embedding whitening improves constructed MLP fact-storage cost\footnote{For Figure~\ref{fig:scaling-rho}a, to obtain embeddings with small $\rho$, we use embeddings which are sampled uniformly from a unit sphere and then multiplied by an ill-conditioned transformation matrix. For this choice of embeddings, whitening exactly removes the dependence on $\rho$, but for other embeddings a dependence on $\rho$ may remain (See Appendix \ref{app:expt}).} for embeddings with low $\rho$ by up to $32\times$. However, as we will show in~\Cref{sec:usability}, whitening the embeddings results in MLPs with large Lipschitz constant that are harder to use within transformers.

\begin{figure*}[ht]
    \centering
    \includegraphics[width=0.99\linewidth]{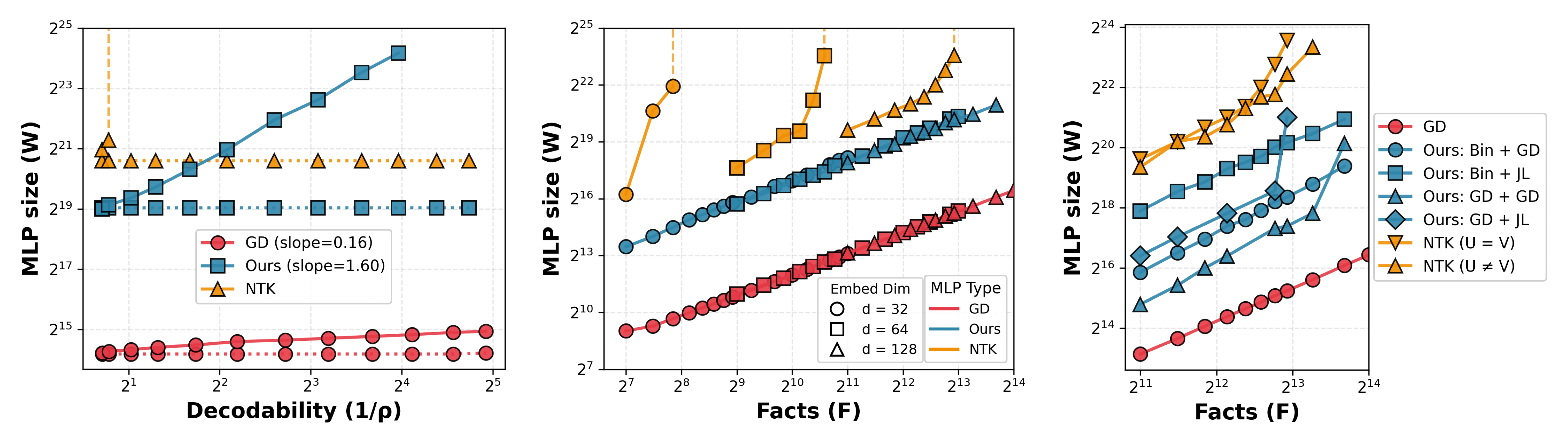}
   \caption{\textbf{(a)} For both GD and our constructed MLPs, $\rho$ is predictive ($R^2 > 0.97$) of MLP size for a fixed number of facts. Embedding whitening reduces our constructed MLPs' fact-storage cost by up to $32\times$ and allows NTK MLPs to generalize to highly anisotropic embeddings.
    \textbf{(b)} GD MLPs and our constructed MLPs exhibit consistent facts-per-parameter scaling as embedding dimension and number of facts vary jointly, whereas NTK MLPs exhibit asymptotically worse scaling as more facts are squeezed into a fixed embedding dimension (pictured for spherical embeddings). Our constructed MLPs have between $5$--$150\times$ lower fact-storage cost than NTK MLPs, while GD MLPs have $\sim\!20\times$ lower fact-storage cost than ours.
    \textbf{(c)} When training the encoder and decoder with gradient descent, the fact-storage cost gap to GD MLPs narrows from $\sim\!20\times$ to $\sim\!4\times$.
    }
    \label{fig:scaling-rho}
\end{figure*}

\section{MLP Constructions}
\label{sec:theory_main}
\begin{algorithm}[t]
\caption{Fact-Storing MLP Framework}
\label{alg:fact_mlp_framework}
\begin{algorithmic}[1]
\REQUIRE $\mathbf{K} \in \mathbb{R}^{|\mathbf{K}|\times d}$, $\mathbf{V} \in \mathbb{R}^{|\mathbf{V}|\times d}$, $f : [|\mathbf{K}|] \to [|\mathbf{V}|]$
\REQUIRE Hidden size $h$, compressed dim. $m$, activation $\sigma$
\STATE $(\mathbf{C}\in\mathbb{R}^{|\mathbf{V}|\times m}, \mathbf{D}\in\mathbb{R}^{d\times m}) \gets \textsc{Dec}(\mathbf{V}, m)$
\STATE $(\mathbf{A}, \mathbf{G}\in\mathbb{R}^{h\times d}, \mathbf{E}\in\mathbb{R}^{m\times h}) \gets \textsc{Enc}(\mathbf{K}, \mathbf{C}, f, h, \sigma)$
\STATE $\mathbf{MLP}(\mathbf{x}) \coloneqq \mathbf{D} \mathbf{E} \left(\sigma(\mathbf{G}\mathbf{x}) \odot (\mathbf{A}\mathbf{x})\right)$
\STATE \textbf{return}  $\mathbf{MLP}$
\end{algorithmic}
\end{algorithm}

We now present our framework for fact-storing MLPs (\Cref{alg:fact_mlp_framework}). The core insight of our framework is to define \textit{compressed output embeddings} $\mathbf{C}\in\mathbb{R}^{|\mathbf{V}|\times m}$ and to decompose the MLP into an \textit{encoder}, which maps keys $\mathbf{k}_i$ to compressed outputs $\mathbf{c}_{f(i)}$, and a \textit{decoder}, which decompresses $\mathbf{c}_{f(i)}$ into an output in $\mathbb{R}^d$ which decodes to $\mathbf{v}_{f(i)} \in \mathbb{R}^d$. This encoder-decoding framework is sufficient to match the asymptotic scaling of GD MLPs' fact-storage cost across a range of embeddings.

In \Cref{subsec:encoder_twohot} and \Cref{subsec:construction_decoder}, we present the details of the encoder and decoder portions of our frameworks, respectively. For each, we 1) present the encoder/decoder structure and objective, 2) demonstrate how an encoder/decoder can be obtained through gradient descent, and 3) present explicit, closed-form weight constructions with asymptotic analysis.

In \Cref{subsec:full_construction} we present the full construction and show that it provides tighter asymptotic fact-storage cost than has been proven for prior constructions, even matching the information-theoretic lower bounds in some cases. Finally, in \Cref{subsec:3_expts} we demonstrate empirically that 1) our construction has a lower fact-storage cost than prior constructions and 2) unlike prior constructions, our construction's \emph{fact-storage cost} scaling matches that of GD MLPs even when varying the number of facts or input-output dimensions independently.

\subsection{The Encoder}
\label{subsec:encoder_twohot}

Our encoder is a single-hidden layer MLP mapping key embeddings to compressed output embeddings.

\paragraph{Encoder Structure}
Our encoder is a gated MLP\footnote{For the rest of \Cref{sec:theory_main}, we drop biases for notational simplicity.}
\begin{equation*}
    \enc(\mathbf{x}) \;=\; \mathbf{E} \left(\sigma(\mathbf{G}\mathbf{x}) \odot (\mathbf{A}\mathbf{x})\right)
\end{equation*}
where $\mathbf{A}, \mathbf{G} \in \mathbb{R}^{h \times d}$, $\mathbf{E}\in \mathbb{R}^{m\times h}$, $\mathbf{x}\in\mathbb{R}^d$, and $\sigma:\mathbb{R}^h \to \mathbb{R}^h$ is an activation function.

Gated MLPs simplify our analysis and are now popular across frontier models~\citep{yang2025qwen3, dubey2024llama}. In \cref{app:extended_theory}, we extend to non-gated MLPs.

\paragraph{Encoder Framework Objective}
Given key embeddings $\mathbf{K}\in\mathbb{R}^{|\mathbf{K}|\times d}$, compressed output embeddings $\mathbf{C}\in\mathbb{R}^{|\mathbf{V}|\times m}$, and a mapping $f$, the objective of our encoder framework is to produce an MLP $\enc$ with a minimal number of parameters such that $\enc(\mathbf{k}_i) = \mathbf{c}_{f(i)}$ for all $i\in|\mathbf{K}|.$

\paragraph{Gradient-Descent Construction}
One strategy to build an encoder MLP is to use gradient descent (a \textit{GD Encoder}) by optimizing for $\enc$ in the Mean-Squared Error (MSE) objective
\begin{equation*}
    \mathcal{L}(\mathbf{K}, \mathbf{C}; \enc) = \sum_{i\in|\mathbf{K}|} ||\enc(\mathbf{k}_i) - \mathbf{c}_{f(i)}||^2.
\end{equation*}

\paragraph{Closed-Form Weight Construction}
Alternatively, we can construct an encoder via a closed-form weight construction. Our constructed encoder builds $m$ encoder gadgets\footnote{
We can set the down projection to $\mathbf{1}^\top$ without loss of generality by replacing $\mathbf{A}$ with $\textrm{diag}(\mathbf{E})\mathbf{A}$.}
\[
    \enc_j(\mathbf{x}) \;=\; \mathbf{1}_{\tilde h}^\top \left[\sigma(\mathbf{G}\mathbf{x}) \odot (\mathbf{A}\mathbf{x})\right],\quad\mathbf{G},\mathbf{A}\in\mathbb{R}^{\tilde h\times d},
\]
that map $\mathbf{k}_i$ to $\mathbf c_{f(i)}[j] \in \mathbb{R}$, respectively, where $\tilde h = h/m$. We will demonstrate that these gadgets require only $O(|\mathbf{K}|)$ parameters. By stacking all $m$ gadgets together, one for each target dimension $j$, we can construct $\mathbf c_{f(i)}$ with a total of $O(m|\mathbf{K}|)$ parameters, as shown in \Cref{alg:encoder_construction}.

\begin{algorithm}[t]
\caption{Encoder Construction (\textsc{Enc})}
\label{alg:encoder_construction}
\begin{algorithmic}[1]
\REQUIRE $\mathbf{K} \in \mathbb{R}^{|\mathbf{K}|\times d}$, $\mathbf{C} \in \mathbb{R}^{|\mathbf{V}|\times m}$, $f : [|\mathbf{K}|] \to [|\mathbf{V}|]$
\REQUIRE Hidden size $h$, activation $\sigma$
\STATE $\tilde h \coloneqq h/m$
\vspace{7pt}
\STATE \textbf{for} $j = 1$ \textbf{to} $m$ \textbf{do}
\STATE \hspace{1em} $\mathbf{o}^{(j)} \coloneqq [\mathbf{C}_{f(1), j},\ldots,\mathbf{C}_{f(|\mathbf{K}|), j}]  \in \mathbb{R}^{|\mathbf{K}|}$
\STATE \hspace{1em} $(\mathbf{A}^{(j)}, \mathbf{G}^{(j)}\in\mathbb{R}^{\tilde h\times d}) \gets \textsc{EncGad}(\mathbf{K}, \mathbf{o}^{(j)}, \tilde h, \sigma)$
\STATE \textbf{end for}
\vspace{7pt}
\STATE Stack encoder gadgets $\mathbf{A}, \mathbf{G}\in \mathbb{R}^{m\times d}$:
\begin{equation*}
    \mathbf{A} \coloneqq
    \begin{bmatrix}
    \mathbf{A}^{(1)} \\
    \vdots \\
    \mathbf{A}^{(m)}
    \end{bmatrix},
    \quad
    \mathbf{G} \coloneqq
    \begin{bmatrix}
    \mathbf{G}^{(1)} \\
    \vdots \\
    \mathbf{G}^{(m)}
    \end{bmatrix}
\end{equation*}
\STATE $\mathbf{E} \coloneqq \begin{bmatrix}
    \mathbf{1}_{1\times \tilde h} & \mathbf{0}_{1\times \tilde h} & \cdots & \mathbf{0}_{1\times \tilde h} \\
    \mathbf{0}_{1\times \tilde h} & \mathbf{1}_{1\times \tilde h} & \cdots & \mathbf{0}_{1\times \tilde h} \\
    \vdots & \vdots & \ddots & \vdots \\
    \mathbf{0}_{1\times \tilde h} & \mathbf{0}_{1\times \tilde h} & \cdots & \mathbf{1}_{1\times \tilde h}

\end{bmatrix}\in\mathbb{R}^{m\times h}$
\STATE \textbf{return} $(\mathbf{A}, \mathbf{G}, \mathbf{E})$
\end{algorithmic}
\end{algorithm}

\textit{Simple Two-Hot Encoder Gadget:} For clarity, we first present the encoder gadget in a simplified setting (Construction \ref{con:two_hot_unit_interval_main}), where the key embeddings are \textit{two-hot}, i.e.,  $\mathbf{K} =\{\mathbf{e}_{i} - \mathbf{e}_{j}\in \mathbb{R}^{d} \, | \, i\neq j \in [d]\}$, with $|\mathbf{K}| = d(d-1).$

Intuitively, Construction \ref{con:two_hot_unit_interval_main} involves two sequential steps: 1) pick a gating term that selects different portions of the input for different hidden neurons (in the case below, $\mathrm{ReLU}(\mathbf{I}_d\mathbf{x})$) and 2) find the $\mathbf{A}$ that fits the data. These two steps underlie our generalization of Construction \ref{con:two_hot_unit_interval_main} to arbitrary gating functions and embeddings.

\begin{construction}[Encoder, Two-Hot]
\label{con:two_hot_unit_interval_main}
Let \[h:\{(i,j)\,|\,i\neq j \mathrel{\in} [d]\} \to \mathbb{R}\]
be a function mapping each pair $(i,j)$ to the desired output for key embedding $\mathbf{e}_i - \mathbf{e}_j$.
Define $\enc(\mathbf{x}) \;=\; \mathbf{1}_d^\top \left[\mathrm{ReLU}(\mathbf{I}_d\mathbf{x}) \odot (\mathbf{A}\mathbf{x})\right]$, where $\mathbf{A} \in \mathbb{R}^{d \times d}$ with
\begin{align*}
    \mathbf{A}[p,q] \;=\;
    \begin{cases}
    0 & \text{if } p=q \\[4pt]
    -\,h(p,q) & \text{if } p\neq q.
    \end{cases}
\end{align*}
Then $\enc(\mathbf{e}_{i} - \mathbf{e}_{j}) = h(i,j)$ for all $i\neq j\in [d]$.  
This encoder has $2|\mathbf{K}| + O(d)$ parameters.\footnote{By a simple degrees-of-freedom argument, any MLP that can map $d^2-d = |\mathbf{K}|$ inputs each to an arbitrary real requires at least $|\mathbf{K}|$ parameters. Hence, the construction is asymptotically optimal in parameter count.}
\end{construction}
\begin{proof}[Proof:]
\begin{align*}
    &\;\mathrm{ReLU}(\mathbf{I}_d(\mathbf{e}_{i} - \mathbf{e}_{j})) \odot (\mathbf{A}(\mathbf{e}_{i} - \mathbf{e}_{j}))\\
    =&\;\mathbf{e}_{i} \odot (\mathbf{A}(\mathbf{e}_{i} - \mathbf{e}_{j}))\\
    =&\;(\mathbf{A}[i, i] - \mathbf{A}[i, j])\, \mathbf{e}_{i}\\
    =&\;h(i,j)\, \mathbf{e}_{i}.
\end{align*}
Finally, multiplying by $\mathbf{1}_d^\top$ extracts $h(i,j)$.
\end{proof}

\textit{A Generalized Gated Encoder Gadget:}
Following the two-hot example, our generalized gated encoder gadget will follow two simple steps: 1) pick $\mathbf{G}$, and 2) solve the resulting linear system for $\mathbf{A}$. The rest of this section will be dedicated to defining the linear system for $\mathbf{A}$ and providing conditions for a solution to exist.

Define\begin{align*}
    \mathbf{\Sigma} &= \sigma(\mathbf{G}\mathbf{K}^\top)\in\mathbb{R}^{h\times|\mathbf{K}|}\\
    \mathbf{o} &= [\mathbf c_{f(1)}[j],\ldots, \mathbf c_{f(|\mathbf{K}|)}[j]]^\top\\
    \mathbf{M}(\mathbf{\Sigma}, \mathbf{K}) &= [\textrm{\upshape diag}(\mathbf{\Sigma}_1)\mathbf{K},\; \ldots, \;\textrm{\upshape diag}(\mathbf{\Sigma}_{h})\mathbf{K}] \in \mathbb{R}^{|\mathbf{K}|\times dh}.
\end{align*}

The $\mathbf{A}$ matrices such that $\enc(\mathbf{k}_i) = \mathbf c_{f(i)}[j]$ for all $i\in|\mathbf{K}|$ are exactly the solutions to the linear system\footnote{We define $\textrm{vec}(\mathbf{A}) = [\mathbf{a}_1, \ldots, \mathbf{a}_h]^\top \in\mathbb{R}^{dh}.$}:
\[
\mathbf{M}(\mathbf{\Sigma}, \mathbf{K})\,\textrm{vec}(\mathbf{A}) = \mathbf{o}
\]
To obtain a construction, we need to choose $\mathbf{\Sigma}$ such that the system is solvable for every choice of $\mathbf{o}$, which is true if and only if $\mathbf{M}(\mathbf{\Sigma}, \mathbf{K})$ has full row-rank. 
Interestingly, this is true for generic $\mathbf{K}$ provided a simple rank condition on $\mathbf{\Sigma}:$
\begin{lemma}\label{lem:rank_condition}
    The matrix $\mathbf{M}(\mathbf{\Sigma}, \mathbf{K})$ has full row-rank for generic\footnote{I.e., for all $\mathbf{K}$ in a Zariski open set. The set of $\mathbf{K}$ not satisfying this condition is measure 0.} $\mathbf{K}$ if and only if
    \begin{equation}\label{eq: rank_cond_main}
        d\cdot \textrm{\upshape rank}(\mathbf{\Sigma}[:, S]) \ge |S|\quad\quad \forall S \subseteq [|\mathbf{K}|].
    \end{equation}
\end{lemma}

Further, for analytic $\sigma$, such a $\mathbf{\Sigma}$ is easy to find:
\begin{lemma}\label{lem:random_satisfies_rank}
    Let $\sigma: \mathbb{R}\to \mathbb{R}$ be a non-polynomial analytic activation. As long as $dh \ge |\mathbf{K}|$, for generic $\mathbf{K}\in\mathbb{R}^{|\mathbf{K}|\times d}$ and $\mathbf{G}\in\mathbb{R}^{h\times d}$, we have that $\mathbf{\Sigma} = \sigma(\mathbf{G}\mathbf{K}^\top)$ satisfies Equation \ref{eq: rank_cond_main}.
\end{lemma}

\begin{algorithm}[t]
\caption{Encoder Gadget Construction (\textsc{EncGad})}
\label{alg:encoder_gadget}
\begin{algorithmic}[1]
\REQUIRE $\mathbf{o} \in \mathbb{R}^{|\mathbf{K}|}$, generic $\mathbf{K} \in \mathbb{R}^{|\mathbf{K}|\times d}$
\REQUIRE Hidden size $h$ with $dh \ge |\mathbf{K}|$, analytic $\sigma$
\STATE Sample generic $\mathbf{G} \in \mathbb{R}^{h\times d}$ (e.g.\ i.i.d.\ Gaussian)
\STATE $\mathbf{\Sigma} \coloneqq \sigma(\mathbf{G}\mathbf{K}^\top)\in \mathbb{R}^{h\times |\mathbf{K}|}$
\vspace{7pt}
\STATE $\mathbf{M}
\coloneqq
\big[
\operatorname{diag}(\mathbf{\Sigma}_{1})\mathbf{K}, \cdots, 
\operatorname{diag}(\mathbf{\Sigma}_{h})\mathbf{K}
\big]\in \mathbb{R}^{|\mathbf{K}|\times (dh)}$
\STATE Solve for $\mathbf{v}\in \mathbb{R}^{dh}$ in $\mathbf{M} \, \mathbf{v} = \mathbf{o}$
\vspace{7pt}
\STATE $\mathbf{A} \coloneqq \begin{bmatrix}
    \mathbf{v}[1:d]\\
    \mathbf{v}[d+1:2d]\\
    \vdots\\
    \mathbf{v}[(h-1)d+1:hd]
\end{bmatrix} \in \mathbb{R}^{h\times d}$
\STATE \textbf{return} $(\mathbf{A}, \mathbf{G})$
\end{algorithmic}
\end{algorithm}

Putting these results together gives the more general construction in \Cref{alg:encoder_gadget}, proven in Appendix \ref{sec:new_encoding_results} along with generalizations to other activations functions $\sigma$ such as ReLU.

\paragraph{Asymptotic Analysis}

When $m$ copies of the generalized encoder gadget from \Cref{alg:encoder_gadget} are stacked to produce full output vectors, the full encoder contains $2m|\mathbf{K}| + O(md) + O(mh)$ parameters, which for $d,h = o(|\mathbf{K}|)$ is within a factor of two of the degrees-of-freedom lower bound of $m|\mathbf{K}|$ (up to lower order terms).

To our knowledge, our generalized encoder gadget is the first demonstration that gated MLPs can exactly memorize $N$ generic datapoints with $O(N)$ parameters, asymptotically matching the degrees-of-freedom lower bound.

In Appendix \ref{sec:new_encoding_results}, we show that our results extend to non-gated MLPs (up to an arbitrarily small $\epsilon$ error) by implementing a neural tangent kernel approximation similar to \citet{nichani2024understandingfactualrecalltransformers}. 
Interestingly, when this generalization is applied to ReLU MLPs, we obtain a construction which generalizes that from \citet{bubeck2020networksizeweightssize}.

Naively, if we allow $m=d$, the encoder alone could output the target embeddings exactly.
However, this construction would yield an MLP with $\Theta(d|\mathbf{K}|)$ parameters, which does not match the information-theoretic limit of $\Omega(|\mathbf{K}|\log |\mathbf{V}|)$ from \Cref{thm: info_bounds_capacity-const}.
As we explore in the next subsection, we can obtain a $\Theta(|\mathbf{K}|\log |\mathbf{V}|)$ construction by instead setting $m < d$ and picking \emph{compressed output embeddings} that can be approximately decoded into the optimal output embeddings.

\subsection{The Decoder and \texorpdfstring{$\rho$}{rho}}
\label{subsec:construction_decoder}

We next describe our decoder framework.

\paragraph{Decoder Structure}
The decoder consists of a single linear layer $\mathbf{dec}(\mathbf{x}) = \mathbf{D} \mathbf{x}$, where $\mathbf{D}\in \mathbb{R}^{d \times m}$ and $\mathbf{x}\in\mathbb{R}^m.$

\paragraph{Decoder Framework Objective}
Given value embeddings $\mathbf{V}\in\mathbb{R}^{|\mathbf{V}|\times d}$, the objective of our decoder framework is to produce 1) compressed output embeddings $\mathbf{C}\in\mathbb{R}^{|\mathbf{V}|\times m}$ and 2) a decoder $\mathbf{dec}$ such that \begin{equation}
    \label{eq:decoding_condition_maintext}
    \langle \mathbf{v}_i,\, \mathbf{dec}(\mathbf{c}_i) \rangle \;>\; \langle \mathbf{v}_j,\, \mathbf{dec}(\mathbf{c}_i) \rangle, \quad \forall i \neq j \in [|\mathbf{V}|],
\end{equation}
for a minimal value of $m$. We seek to minimize $m$ because the overall MLP parameter count is proportional to $m$.

\paragraph{Gradient Descent Construction}
We can easily construct such a pair of compressed output embeddings and a decoder linear layer using gradient descent (a \textit{GD Decoder}) by optimizing for $\mathbf{C}$ and $\mathbf{D}$ in the objective
\begin{equation*}
    \mathcal{L}(\mathbf{C}, \mathbf{D}, \mathbf{K}) = \sum_{i\neq j \in [|\mathbf{V}|]} \langle \mathbf{v}_i - \mathbf{v}_j,\, \mathbf{D} \mathbf{c}_i \rangle.
\end{equation*}

\paragraph{Closed-Form Weight Construction}
We will now provide a closed-form construction for such a decoder framework where $m = O(\log |\mathbf{V}|)$ with high probability for most embedding common embeddings distributions (e.g., normal, spherical, etc.). This gives $O(|\mathbf{K}|\log |\mathbf{V}|)$ parameters\footnote{We describe this in detail in Appendix \ref{app:subsec_bit_complexity}.} for the full encoder-decoder MLP.
\begin{construction}[Decoder Construction]
    \label{const: decoder_main}
    Sample an i.i.d. random Gaussian matrix $\mathbf{D}\in\mathbb{R}^{d\times m}.$ Then, define $\mathbf{c}_i = \mathbf{D}^\top \mathbf{u}_i^\star(\mathbf{V})$. For $m = O\left([\rho(\mathbf{V})]^{-2}\ \log |\mathbf{V}|\right)$, Equation \ref{eq:decoding_condition_maintext} holds with probability $>2/3$. Thus, $\mathbf{dec}(\mathbf{x}) = \mathbf{D}\mathbf{x}$ is a valid decoder construction with probability greater than $2/3$.
\end{construction}
\begin{proof}[Proof Sketch.]
    \(\langle \mathbf{v}_i - \mathbf{v}_j,\, \mathbf{D} \mathbf{c}_i \rangle = \langle \mathbf{D}^\top(\mathbf{v}_i - \mathbf{v}_j),\, \mathbf{D}^\top \mathbf{u}_i^\star \rangle.\)
    By Johnson-Lindenstrauss \citep{johnson1984extensions}, for $m =\Omega\left([\rho(\mathbf{V})]^{-2}\ \ln |\mathbf{V}|\right)$ and for all $i,j \in [|\mathbf{V}|]$,
    \[
        \text{sign}\left(\langle \mathbf{D}^\top(\mathbf{v}_i - \mathbf{v}_j),\, \mathbf{D}^\top \mathbf{u}_i^\star \rangle\right) = \text{sign}\left(\langle \mathbf{v}_i - \mathbf{v}_j,\, \mathbf{u}_i^\star \rangle\right)
    \]
    with probability $>2/3$. See \Cref{thm:NEW_main_decoding} for a full proof.
\end{proof}

\begin{algorithm}[t]
\caption{Decoder Construction (\textsc{Dec})}
\label{alg:decoder_construction}
\begin{algorithmic}[1]
\REQUIRE $\mathbf{V} = \in \mathbb{R}^{|\mathbf{V}|\times d}$, compressed dimension $m$
\STATE $\mathbf{U}^*\in\mathbb{R}^{|\mathbf{V}|\times d} \gets \textsc{OptimalOut}(\mathbf{V})$
\STATE Sample an i.i.d.\ Gaussian matrix $\mathbf{D} \in \mathbb{R}^{d\times m}$
\STATE $\mathbf{C}\coloneqq \mathbf{U}^\star \mathbf{D} \in \mathbb{R}^{|\mathbf{V}|\times m}$
\STATE \textbf{return} $(\mathbf{C}, \mathbf{D})$
\end{algorithmic}
\end{algorithm}

The decodability $\rho(\mathbf{V})$ (\Cref{eq:decodability}) quantifies how large $m$ needs to be as a function of how tightly clustered the value embeddings are. Notably, our construction applies to all feasible embeddings ($\rho(\mathbf{V}) > 0$).

\subsection{Full MLP Construction}
\label{subsec:full_construction}

Finally, we put the encoder and decoder together and describe our full fact MLP construction.
\begin{theorem}[Full Construction]\label{thm:capacity-final}
    For any fact set $f$, generic key embeddings $\mathbf{K}$, and value embeddings $\mathbf{V}$ with $\rho(\mathbf{V})>0$, construct $\enc$ as described in \Cref{subsec:encoder_twohot} and construct $\mathbf{dec}$ as described in \Cref{subsec:construction_decoder}. Our constructed fact MLP
    \[
        \mathbf{g}(\mathbf{x})\;=\;\mathbf{dec}(\enc(\mathbf{x})) = \mathbf{D}\,\mathbf{E}\  (\sigma(\mathbf{G} \mathbf{x})\odot (\mathbf{A}\mathbf{x}))
    \]
    stores $f$ given $\mathbf{K}$ and $\mathbf{V}$. Our constructed fact MLP has fact-storage cost $\Theta\left(\left[\rho(\mathbf{V})\right]^{-2}|\mathbf{K}|\log |\mathbf{V}|\right)$.
\end{theorem}

We compare our construction to other fact-storing MLP constructions in Table \ref{tab:theory}. For value embeddings with $\rho(\mathbf{V}) = \Omega(1)$, our construction is the first to match the asymptotic parameter count predicted by the information-theory lower bound (\Cref{thm: info_bounds_capacity-const}) and requires a $\log^{11} |\mathbf{V}|$ factor fewer parameters than \cite{nichani2024understandingfactualrecalltransformers}. Additionally, in the case of two-hot key and value embeddings (using Construction \ref{con:two_hot_unit_interval_main} for the encoder), our construction matches the information-theory lower bound (\Cref{thm: info_bounds_capacity-const}) in terms of bits.

\begin{table*}[t]
\centering
\caption{Comparison of construction fact storage costs and assumptions.
\citet{nichani2024understandingfactualrecalltransformers} assumes $|\mathbf K| = |\mathbf V|$. The na\"ive construction is detailed in \Cref{subsec:naive_const}.}
\label{tab:theory}
\scriptsize
\renewcommand{\arraystretch}{1.5}
\begin{tabular}{@{}lccll@{}}
\toprule
 & Parameters & Hidden Sizes & Assumptions on $\mathbf{K}$ & Assumptions on $\mathbf{V}$ \\ \midrule
Info-Theory Bound 
& $|\mathbf{K}| \log |\mathbf{V}|$ 
& $d^{-1} |\mathbf{K}| \log |\mathbf{V}|$ 
& None 
& None \\

\hline

Na\"ive
& $d |\mathbf{K}|$ 
& $|\mathbf{K}|$ 
& General Position 
& $\rho(\mathbf{V}) > 0$ \\

\citet{nichani2024understandingfactualrecalltransformers} 
& $|\mathbf{K}| \log^{12} |\mathbf{V}|$ 
& $d^{-1} |\mathbf{K}| \log^{12} |\mathbf{V}|$ 
& Uniform on $S^{d-1}$ 
& Uniform on $S^{d-1}$ \\

Ours 
& $[\rho(\mathbf{V})]^{-2}|\mathbf{K}| \log |\mathbf{V}|$ 
& $d^{-1} [\rho(\mathbf{V})]^{-2} |\mathbf{K}| \log |\mathbf{V}|$ 
& General Position 
& $\rho(\mathbf{V}) > 0$ \\
\bottomrule
\end{tabular}
\end{table*}

\subsection{Constructed and GD fact MLPs Empirical Scaling}\label{subsec:3_expts}

In \Cref{fig:scaling-rho} we show the fact-storage cost of our constructed MLPs, the constructed MLPs from \cite{nichani2024understandingfactualrecalltransformers} (NTK MLPs), and MLPs trained with gradient descent (GD MLPs) across a range of embeddings.

In \Cref{fig:scaling-rho}a, we demonstrate that our constructed MLP fact-storage cost scales inversely with $\rho$ at a rate matching the prediction from Construction \ref{thm:capacity-final}.

In \Cref{fig:scaling-rho}b, we show that for embeddings sampled from an i.i.d. uniform spherical distribution (\textit{spherical embeddings}), our MLPs empirically match the asymptotic fact-storage cost of GD MLPs unlike NTK MLPs.

Additionally, we ablate the effect of using gradient descent for the encoder and decoder of our construction: replacing our encoder construction with a gradient-descent-trained encoder (\textit{GD + JL}) increases our construction fact-storage capacity by  $\sim\!3\times$, replacing our decoder construction with a gradient-descent-trained decoder (\textit{Bin + GD}) increases our construction fact-storage capacity by $\sim\!4\times$, and replacing both our encoder and decoder constructions with gradient-descent-trained counterparts (\textit{GD + GD}) increases our construction fact-storage capacity by $\sim\!8\times$.

In \Cref{fig:scaling-rho}c, we show the fact-storage cost on spherical embeddings for $d\in\{32, 64, 128\}$ and variable $F = |\mathbf{K}| = |\mathbf{V}|$, specifically by setting $F = \alpha d^2$ for various $\alpha.$ We see that like GD MLPs, our construction exhibits the same scaling regardless of the choice of $d$. On the other hand, for each choice of $d$, NTK MLPs diverge for sufficiently large $\alpha$ and $F$,
indicating that NTK MLPs do not mimic the ability of fact MLPs to store large fact sets with small input-output dimension.

\section{Integrating fact-storing MLPs into Transformers}
\label{sec:usability}
\begin{figure*}[ht]
    \centering
    \includegraphics[width=0.99\linewidth]{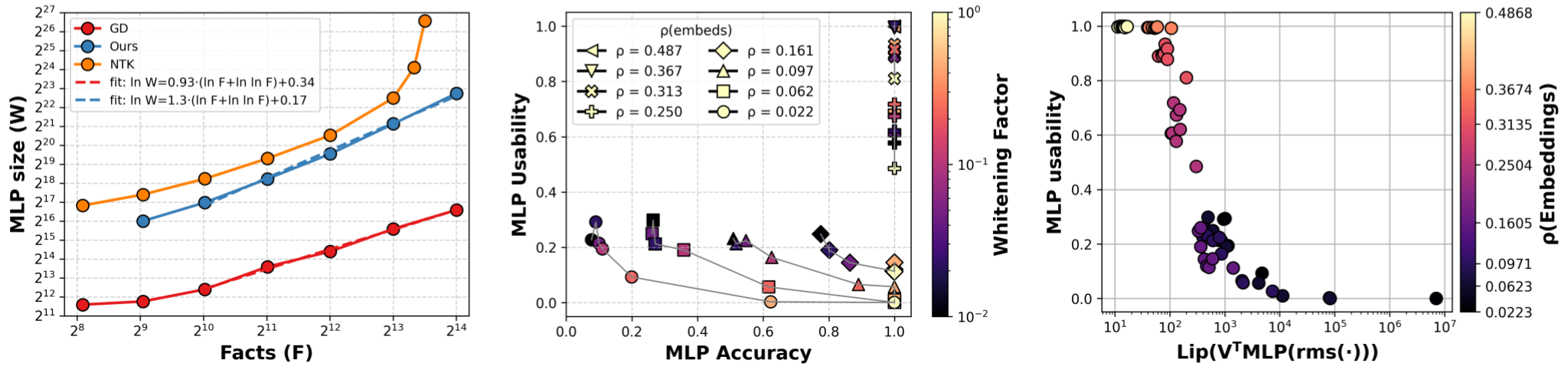}
   \caption{\textbf{(a) MLP size vs. fact-set size for MLPs with $\ge 99\%$ usability within Transformer.} We find that fact-storing MLPs are usable within 1-layer Transformers and that our constructed MLPs and GD MLPs exhibit similar $\ge 99\%$ usability scaling. \textbf{(b) MLP usability within Transformer v.s. MLP storage capacity.} We observe a tradeoff between MLP usability within a Transformer and the MLP's fact-storage capacity. \textbf{(c) MLP usability within Transformer v.s. its Lipschitz constant.} We observe that the measured Lipschitz constant is predictive of an MLP's usability within Transformers.}
    \label{fig:ssfr_main_fig}
\end{figure*}

We now investigate the extent to which fact-storing MLPs can be used by a transformer for factual recall. In \Cref{sec:ssfr}, we introduce the \emph{Synthetic Sequential Factual Recall} (SSFR) task, which formalizes the notion of transformer factual recall. We then find a small set of architectural modifications that enable vanilla transformers to use constructed MLPs for factual recall. Under this setup, we show that the number of MLP parameters required for a transformer to properly use the for factual recall grows at a comparable rate to the information-theoretically optimal one.

In \Cref{subsec:tradeoff}, we uncover a tradeoff between the capacity of an MLP to store facts and its usability for factual recall within transformers. We demonstrate that this tradeoff can be navigated through embedding whitening. In \Cref{subsec:conditioning}, we further show that an MLP’s Lipschitz constant serves as an indicator of its usability for factual recall by transformers.

Finally, in \cref{sec:fact-editing}, we explore using fact-storing MLPs within 1-layer transformers on a synthetic language-modeling (LM) task. We find that fact-storing MLPs within transformers can be swapped by MLPs storing entirely different fact sets, incurring only a $\sim$3\% cross-entropy increase on non-fact tokens while enabling the transformer to produce the new facts. Moreover, our MLP-swapping method outperforms prior fact-editing MLP updates, doubling their fact-editing score when editing 10\% of the fact set.

\subsection{Transformers can use fact-storing MLPs for factual recall}
\label{sec:ssfr}

We first demonstrate that fact-storing MLPs can be used for factual recall within a transformer. Further, we show that, together with GD MLPs, our construction is the first to be usable within a transformer while storing an amount of facts per parameter comparable to the information-theory optimal one.

\paragraph{Task.}
We introduce an associative-recall-style task \citep{zoology, nichani2024understandingfactualrecalltransformers}, which we term \emph{Synthetic Sequential Factual Recall} (SSFR), to test whether fact-storing MLPs can be used by transformers for factual recall. In SSFR, a transformer processes a sequence of ``junk'' tokens containing a single \emph{key} token and must predict the corresponding \emph{value} token at the end of the sequence. For example,
\[
\underbrace{*\ \%\ \&\ \#\ \$}_{\text{junk prefix}}
\ \underbrace{A}_{\text{key}}
\ \underbrace{\&\ \%\ *\ \$\ \#}_{\text{junk suffix}}
\ \rightarrow\ 
\underbrace{B}_{\text{value}}.
\]
This mirrors how, in a sentence such as ``The capital of France is Paris,'' the key and value (``capital of France'' and ``Paris'') are separated by an unrelated prefix and suffix (``The'' and ``is''). See Appendix \ref{appendix:ssfr-task} for details.

\paragraph{Training setup.}
Our goal is to evaluate to what extent fact-storing MLPs can be used by transformers on an SSFR task. To test this, we create a fact-storing MLP that stores the SSFR key-value mapping. We then freeze the fact-storing MLP and insert it into a single-layer transformer. Finally, we train the transformer to output the correct value for each SSFR sequence.

\paragraph{Metrics.}
To evaluate whether a transformer is actually using its fact-storing MLP for factual recall, as opposed to memorizing the facts in its attention weights, we define the \textit{fact-adaptive accuracy}. We take a transformer trained on SSFR and replace its fact-storing MLP with a new MLP storing a different fact set. We define the transformer's \textit{fact-adaptive accuracy} as the modified transformer's accuracy on the SSFR task corresponding to the fact set \textit{of the new MLP}. Intuitively, if a transformer has high fact-adaptive accuracy, it is using its fact-storing MLP for factual recall.

\paragraph{Fact-Storing MLPs are usable within transformers.}

We find that a simple set of modifications to the vanilla transformer architecture are sufficient for transformers to use both constructed and GD-trained MLPs for factual recall, achieving $>99\%$ fact-adaptive accuracy, while approximately using an information-theoretically optimal amount of parameters. \Cref{fig:ssfr_main_fig}a shows the minimum fact-storing MLP parameters required for a transformer using it to reach $99\%$ fact-adaptive accuracy as a function of fact-set size. Strikingly, our constructed and GD MLPs both exhibit empirical scaling similar to the theoretical optimum $\log W \approx \log F + \log \log F$, in contrast to NTK MLPs, whose fact-adaptive accuracy explodes for large fact sets. We attribute such a deterioration in fact-adaptive accuracy of NTK MLPs to their sharp decline in fact-storage capacity on large fact sets, as shown in \Cref{fig:scaling-rho}b. See Appendix \ref{appendix:mlp-size-vs-facts} for experimental details.

Concretely, we empirically find that i) tying transformer and MLP embeddings, ii) removing residual connections, iii) freezing the pre-MLP RMSNorm layer, and iv) freezing the \textit{value} and \textit{out-project} matrices of the attention layer to the identity matrix are sufficient for transformers to use fact-storing MLPs for factual recall.

Further, as observed in \cref{fig:app-fact-adaptive-front-full}, we find that the minimum MLP size needed to achieve $>99\%$ fact-adaptive accuracy for GD gated and non-gated MLPs is almost identical, suggesting that fact-storage within a transformer doesn't depend on the specific MLP architecture, but instead on its number of parameters.

\subsection{Tradeoff Between Capacity and Usability of an MLP}\label{subsec:tradeoff}

We uncover a tradeoff between an fact-storing MLP's \emph{storage capacity}, the fraction of facts of a fact set that it can successfully store, and \emph{usability}, the fraction of those stored facts that a transformer using the fact-storing can correctly retrieve, as can be seen in \Cref{fig:ssfr_main_fig}b and \Cref{fig:ssfr_relu}a. Formally, we define:

\begin{align*}
    \text{capacity} &= \frac{\text{\# facts MLP stores}}{\text{total \# facts}} \\
    \text{usability} &= \frac{\text{transformer fact-adaptive accuracy}}{\text{capacity}}.
\end{align*}

To study this capacity-usability tradeoff, we use our \emph{embedding whitening} technique from~\Cref{sec:preconditioning} but vary the strength $\alpha \in [0,1]$ of the empirical covariance whitening transform
\(
T(\mathbf{x}) = \mathbf{M}^\alpha \mathbf{x} + \mathbf{b}.
\)
For a fixed pair of transformer key and value embeddings, characterized by $\rho(\mathbf{K}) = \rho(\mathbf{V})$, we apply different whitening strengths $\alpha$, train an MLP to store a fact set using the corresponding MLP embeddings, and then train a Transformer to use that whitened MLP in SSFR.

We find that adjusting the whitening degree allows us to explore the tradeoff between usability and capacity. MLPs trained on less-whitened embeddings store fewer facts but are more usable by transformers, whereas MLPs trained on highly whitened embeddings store more facts but are harder for transformers to use. See Appendix \ref{appendix:usability-vs-storage-capacity} for experimental details.

\subsection{MLP Usability Depends on Lipschitz Constant} \label{subsec:conditioning}

In \Cref{subsec:tradeoff} we observe that whitened MLPs, with high fact storage capacity, tend to be less usable by transformers. Here, we find that the Lipschitz constant of an MLP serves as an indicator of its usability within a transformer. Concretely, given an MLP trained to represent a fact-set mapping from transformer key embeddings $\mathbf{K}\in\mathbb{R}^{|\mathbf{K}|\times d}$ to value embeddings $\mathbf{V}\in\mathbb{R}^{|\mathbf{V}|\times d}$, we look at:
\begin{equation}
    \text{Lip}(\mathbf{V}^T\text{MLP}(\text{rms}(\cdot)))
    \approx \operatorname{max}_i \sigma_1(\mathbf{J}(\mathbf{k}_i)),
    \label{eq:lip}
\end{equation}
where
\[
    \mathbf{J}(\mathbf{x}_i)
    = \frac{\partial \, \mathbf{V}^\top \mathrm{MLP}(\mathrm{RMSNorm}(\mathbf{x}_i))}{\partial \mathbf{x}_i}.
\]

As seen in ~\cref{fig:ssfr_main_fig}c and \cref{fig:ssfr_relu}a, increased MLP Lipschitz constant correlates with reduced MLP usability for factual recall. Intuitively, we believe this relationship arises due to optimization dynamics, similar to how training convergence under first-order optimizers depends on the largest Hessian singular value~\citep{2022arXiv220911920M}. We note there likely exist other MLP conditioning related metrics that can also capture this relationship. See Appendix \ref{appendix:usability-vs-conditioning} for experimental details.

\subsection{Language Modeling and Fact Editing with fact-storing MLPs} \label{sec:fact-editing}

\begin{figure}[h]
    \centering
    \includegraphics[width=0.31\textwidth]{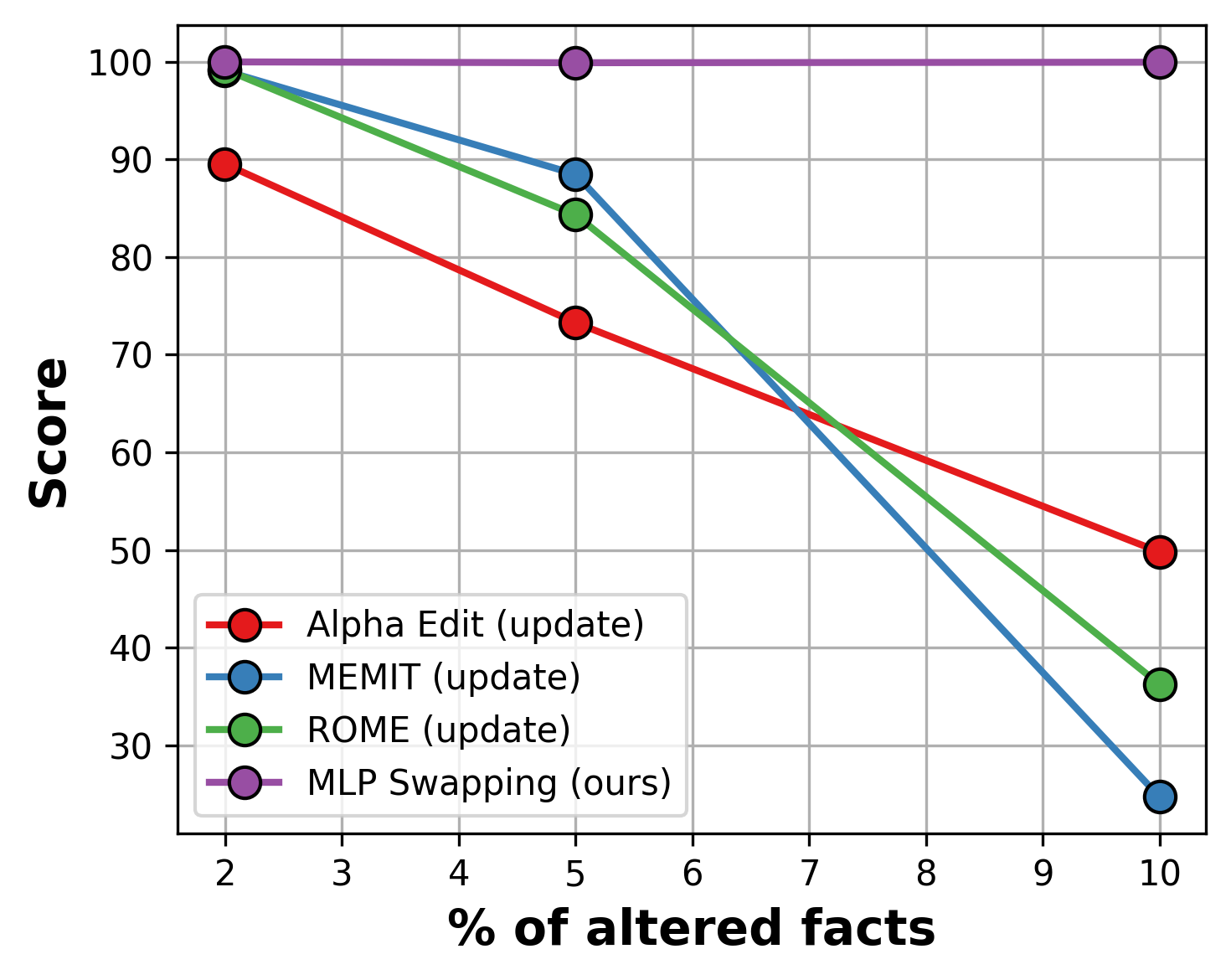}
    \caption{\textbf{Fact editing score as number of altered facts increases}. Fact editing via MLP swapping outperforms prior weight updates as the number of altered facts increase. The fact-editing score is computed as the geometric mean of the efficacy, specificity and paraphrase accuracies. }
    \label{fig:fact-editing}
\end{figure}

Finally, we explore whether fact-storing MLPs can be used by transformers for language modeling. On a synthetic task involving sentences about author-book relations (see Appendix \ref{appendix:lm-task}), we demonstrate that 1-layer transformers can use fact-storing MLPs for factual recall (\cref{fig:lm-fact-adaptive-fronteir}a). Remarkably, when we swap a transformer’s MLP for an entirely new fact-storing MLP, the transformer outputs the new facts with $>99\%$ accuracy while incurring less than a $\sim3\%$ increase in cross-entropy on non-fact tokens (Appendix \ref{fig:lm-fact-adaptive-fronteir}b). See Appendix \ref{appendix:lm-scaling} for experimental details.

Under the same setup, we show that transformers equipped with fact-storing MLPs can be modularly fact-edited. As shown in \Cref{fig:fact-editing}, our modular fact-editing procedure (MLP Swapping) consistently outperforms prior fact editing updates, including those of MEMIT \citep{memit}, ROME \citep{rome}, and Alpha Edit \citep{fang2025alphaeditnullspaceconstrainedknowledge}, doubling their fact-editing scores (defined in \Cref{fig:fact-editing}) on our 1-layer transformers when editing as little as 10\% of the facts stored in its MLP (see Appendix \ref{appendix:fact-editing}). These results suggest a path toward more robust and modular fact manipulation in LLMs.

\section{Discussion}\label{sec:discussion}
We have presented a construction that produces fact-storing MLPs with asymptotically fewer parameters than prior approaches, supports a broader class of embeddings, and can be used by transformers for factual recall. Using this construction, we characterized how output geometry affects fact-storage capacity, identified a simple encoder–decoder mechanism that matches information-theoretic facts-per-parameter scaling, and uncovered a capacity–usability tradeoff for fact-storing MLPs within transformers. These results offer a coherent framework for understanding how MLPs store and expose knowledge within transformers.

More broadly, our work outlines a constructive path forward for studying LLMs. Rather than relying solely on descriptive analyses of pretrained models, we show that explicitly building MLPs with interpretable, provable mechanisms can reveal principles that are otherwise difficult to extract from their learned weights. This constructive approach suggests several promising directions such as designing modular and robust memory systems, developing more parameter-efficient training and inference pipelines, and exploring whether similar constructions can shed light over LLM behaviors beyond factual recall.

In summary, by directly constructing MLPs that store and expose facts, we provide both a theoretical foundation and practical tools for understanding knowledge storage in transformers, as well as a path toward more interpretable and controllable mechanisms in large language models.

\section*{Acknowledgements}
\label{sec:acknowledgements}

The authors thank Neel Guha, Yasa Baig, Catherine Deng, Kelly Buchanan, Sam Buchanan, Avanika Narayan, Andy Dimnaku, Mayee Chen, Hermann Kumbong, Francois Chaubard, Jon Saad-Falcon, Stuart Sul, Alex Waitz, Dan Biderman, Ben Spector, Simran Arora and Michael Zhang for their helpful feedback and discussion.

The authors gratefully acknowledge the support of NIH under No. U54EB020405 (Mobilize), NSF under Nos. CCF2247015 (Hardware-Aware), CCF1763315 (Beyond Sparsity), CCF1563078 (Volume to Velocity), and 1937301 (RTML); US DEVCOM ARL under Nos. W911NF-23-2-0184 (Long-context) and W911NF-21-2-0251 (Interactive Human-AI Teaming); ONR under Nos. N000142312633 (Deep Signal Processing); Stanford HAI under No. 247183; NXP, Xilinx, LETI-CEA, Intel, IBM, Microsoft, NEC, Toshiba, TSMC, ARM, Hitachi, BASF, Accenture, Ericsson, Qualcomm, Analog Devices, Google Cloud, Salesforce, Total, the HAI-GCP Cloud Credits for Research program,  the Stanford Data Science Initiative (SDSI), and members of the Stanford DAWN project: Meta, Google, and VMWare. The U.S. Government is authorized to reproduce and distribute reprints for Governmental purposes notwithstanding any copyright notation thereon. Any opinions, findings, and conclusions or recommendations expressed in this material are those of the authors and do not necessarily reflect the views, policies, or endorsements, either expressed or implied, of NIH, ONR, or the U.S. Government.
OD is supported by the Hertz Foundation Fellowship, the Stanford Knight-Hennessy Scholarship, and the NSF GRFP.
JL is supported by the Department of Energy Computational Science Graduate Fellowship under Award Number DE-SC0023112.
AR's research is supported by NSF grant CCF\#2247014.

\bibliography{example_paper}
\bibliographystyle{icml2025}

\newpage

\appendix
\onecolumn
\section{Experiments}\label{app:expt}
\subsection{MLP Experiments}
Here we describe the experimental setup used for the MLP fact-storage capacity results in~\Cref{sec:preconditioning,sec:theory_main}.

\subsubsection{Task Setup}
\label{ssec:mlps_sweeps_task_setup_appendix}

\paragraph{Fact sets.}
Following the definition of the synthetic fact-storage task (\Cref{eq:decoding-criterion-maintext}), we index facts by $i \in [F]$.
Although fact-storage cost is defined as the smallest number of parameters needed to represent \emph{all possible fact sets} (\Cref{def:complexity}), in our experiments we approximate fact-storage cost as the smallest number of parameters needed to represent \emph{randomly sampled} bijective key-value maps $f : [F] \to [F]$.

\textit{Facts vs. embedding dimension.}
In our experiments, for each embedding dimension $d_{\text{model}}$, we set the number of facts to $F = \beta d_{\text{model}}^2$, where the multiplier $\beta=0.25$ unless otherwise specified.

Empirically, we find that the choice of $\beta$ does not affect the fact-storage capacity of gradient-descent-trained MLPs or our constructed MLPs. However, interestingly, larger values of $\beta$ significantly \emph{decrease} the fact-storage capacity of the MLP construction of~\citet{nichani2024understandingfactualrecalltransformers}: see~\Cref{ssec:mlp_ablations_appendix}.

\paragraph{Embeddings.}
Following prior work~\citep{nichani2024understandingfactualrecalltransformers}, key and value embeddings $\mathbf{K}, \mathbf{V} \in \mathbb{R}^{F \times d}$ are uniformly sampled from the unit sphere. Mirroring how word embeddings in LLMs work, our experiments \emph{tie keys and values}, i.e. $\mathbf{K} = \mathbf{V}$.

\textit{Anisotropic value embeddings.}
To vary the condition number of the value embeddings while preserving their geometric structure, we modify only the singular values of the embeddings matrix. We keep the left and right singular vectors fixed and apply a log-affine rescaling to the singular values so that the largest one is preserved and the smallest one is set to achieve a desired condition number $\kappa$.

\paragraph{Approximating MLP fact-storage cost via binary search.}
For each choice of $(d,F,\kappa,\text{MLP family})$, we determine the minimum number of parameters needed to perfectly store a randomly-sampled fact set given randomly-sampled embeddings. To do so, we perform a one-dimensional binary search over a single scalar hyperparameter characterizing the ``size'' of the MLP. The hyperparameter we sweep over depends on the family of MLPs we evaluate:
\begin{itemize}
    \item For gradient-descent-trained (GD) and NTK MLPs~\citep{nichani2024understandingfactualrecalltransformers}, we search over the hidden dimension $h$.
    \item For our constructed MLPs, we either search over the decoder dimension $m$ or the \emph{encoder width multiplier}.
\end{itemize}
See~\Cref{ssec:mlp_sweeps_architectures_appendix} for details about each of the MLP variants we evaluate.

\subsubsection{Metrics}

\paragraph{Accuracy-based success criterion.}
We evaluate models using the same dot-product scoring rule used in the definition of fact storage (\Cref{eq:decoding-criterion-maintext}), which we restate here for convenience.
Given a trained model $\mathbf{g}_\theta$ and embeddings $(\mathbf{K}, \mathbf{V})$, the predicted value index for a key $i \in [F]$ is
\[
\hat{f}(i)
\;=\;
\arg\max_{j \in [F]}\;
\langle \mathbf{g}_\theta(\mathbf{k}_i),\, \mathbf{v}_j \rangle,
\]
i.e.\ the index achieving the highest score with respect to the MLP output.

The fact-storage \emph{accuracy} of $\mathbf{g}_\theta$ on a fact set $f : [F] \to [F]$ is then
\[
\mathrm{Acc}
\;=\;
\frac{1}{F}
\sum_{i \in [F]}
\mathbf{1}\!\left[\hat{f}(i) = f(i)\right].
\]

Within our binary searches, we declare that a model successfully stores a fact set if it achieves an accuracy of at least $1-\varepsilon_{acc}$. For our MLP fact-storage capacity experiments, we set $\varepsilon_{acc}=0$ unless otherwise stated.

When multiple random seeds are used for a given binary search experiment (e.g. where the randomness is over the choice of fact set and embeddings), we aggregate by taking the \emph{minimum accuracy across seeds} before comparing to this threshold.
The binary search then returns the smallest number of parameters for which the aggregated accuracy is at least $1-\varepsilon_{acc}$.

\subsubsection{MLP architectures and variants}
\label{ssec:mlp_sweeps_architectures_appendix}

Here we summarize all MLP variants evaluated in the capacity sweeps, corresponding to the methods compared in~\Cref{fig:scaling-rho}c and described formally in~\Cref{sec:theory_main}.
Each configuration consists of (i) a choice of MLP variant (gradient-descent-trained, our explicit construction, or the NTK construction of \citet{nichani2024understandingfactualrecalltransformers}), (ii) variant-specific configuration details, including optional use of margin-optimal outputs for NTK MLPs and encoder-decoder settings for our construction models, and (iii) optional embedding whitening.

We start by describing each MLP variant and variant-specific configuration details:

\begin{itemize}
\item \textbf{Gradient-descent-trained (GD) MLPs.}
GD MLPs use the standard two-layer gated MLP (SwiGLU-style) architecture described in~\Cref{subsec:encoder_twohot}, with an ``up'' projection $\mathbb{R}^d \!\to\! \mathbb{R}^h$ followed by a ``down'' projection $\mathbb{R}^h \!\to\! \mathbb{R}^d$.  
Given an input $\mathbf{x}\in\mathbb{R}^d$, the block computes
\[
\mathbf{g}_\theta(\mathbf{x})
\;=\;
W_{\mathrm{down}}\!\left(
  \sigma(W_{\mathrm{gate}}\mathbf{x} + \mathbf{b}_{\mathrm{gate}})
  \odot
  (W_{\mathrm{up}}\mathbf{x} + \mathbf{b}_{\mathrm{up}})
\right)
+ \mathbf{b}_{\mathrm{down}},
\]
where $W_{\mathrm{up}}, W_{\mathrm{gate}} \in \mathbb{R}^{h\times d}$, $W_{\mathrm{down}}\in\mathbb{R}^{d\times h}$, $\sigma$ is Swish, and $\odot$ denotes element-wise multiplication.

Models are trained with full-batch gradient descent using Adam and a cosine-annealed learning rate schedule (initial rate $10^{-3}$, final rate $10^{-6}$) for up to $20{,}000$ epochs with early stopping. We use the cross-entropy objective formed from dot-product logits $\mathbf{g}_\theta(\mathbf{K})\mathbf{V}^\top$, matching the decoding rule of~\Cref{eq:decoding-criterion-maintext}.

In the sweeps, the hidden dimension $h$ is the sole capacity parameter, which means binary search identifies the smallest $h$ for which the trained GD MLP achieves perfect fact-storage accuracy.

\item \textbf{Our constructed MLPs.}
Our construction decomposes the fact-storing MLP into an \emph{encoder} and a \emph{decoder}, each of which admits both an explicit construction and a learnable gradient-descent–based alternative.
For completeness, we summarize all variants evaluated in the sweeps.

\textit{Encoder variants.}
\begin{itemize}
    \item \textbf{Binning / explicit (Bin) encoder.}  
    This is the encoder defined in~\Cref{subsec:encoder_twohot} and~\Cref{alg:encoder_construction}, built by stacking $m$ closed-form encoder gadgets (\Cref{alg:encoder_gadget}).
    Each gadget solves a linear system to map keys to the $j$th coordinate of the compressed code $\mathbf{C}$; the full encoder has the gated form
    \[
       \enc(\mathbf{x}) \;=\; \mathbf{E}\big( \sigma(\mathbf{G}\mathbf{x}) \odot (\mathbf{A}\mathbf{x}) \big).
    \]
    This encoder is fully explicit and requires no training.

    \item \textbf{Gradient-descent-trained (GD) encoder.}  
    Instead of constructing $(\mathbf{A},\mathbf{G},\mathbf{E})$ analytically, we train a gated encoder $g_\theta:\mathbb{R}^d\to\mathbb{R}^m$ via full-batch gradient descent to fit the compressed codes $\mathbf{C}$.
    Given keys $\mathbf{K}$ and targets $\mathbf{C}$ permuted by $f$, we minimize
    \[
    \mathcal{L}_{\mathrm{enc}}
    =
    \frac{1}{F}\sum_{i=1}^F
      \big\| g_\theta(\mathbf{k}_i + \eta_i) - \mathbf{c}_{f(i)} \big\|_2^2,
    \qquad
    \eta_i\sim\mathcal{N}(0, \varepsilon_{\mathrm{key}}^2 I_d),
    \]
    with $\varepsilon_{\mathrm{key}}=10^{-7}$.  
    The encoder uses the same gated MLP architecture as the explicit encoder, but with hidden dimension
    \[
        h = \left\lceil m\,(F/d)\cdot\mathrm{encoder\_width\_multiplier} \right\rceil
    \]
    (where the encoder width multiplier $=1$ by default), and is trained for $1000$ Adam updates with learning rate $10^{-2}$.
    After training, $g_\theta$ is used as the encoder and produces the hidden codes used by the decoder.
\end{itemize}

\textit{Decoder variants.}
\begin{itemize}
    \item \textbf{Johnson-Lindenstrauss (JL) decoder.}  
    This is the explicit decoder of~\Cref{subsec:construction_decoder} and~\Cref{alg:decoder_construction}.
    We sample a Gaussian matrix $\mathbf{D}\in\mathbb{R}^{d\times m}$ and set compressed codes $\mathbf{C}=\mathbf{U}^\star \mathbf{D}$, where $\mathbf{U}^\star$ is the \emph{margin-optimal output embeddings} (\Cref{def:optimal_output_embeddings}).
    For $m = \Theta(\rho(\mathbf{V})^{-2}\log|\mathbf{V}|)$, the JL decoder satisfies the decoding inequalities with high probability.

    \item \textbf{Gradient-descent-trained (GD) decoder.}  
    We replace the random projection with learnable compressed codes $\mathbf{C} \in\mathbb{R}^{F\times m}$ and a learnable decoding matrix $\mathbf{M}\in\mathbb{R}^{m\times d}$. 
    Predicted values are $\hat{\mathbf{V}} = \mathbf{C} \mathbf{M}$ with dot-product scores $S=\hat{\mathbf{V}}\mathbf{V}^\top$.  
    We train $(\mathbf{C},\mathbf{M})$ using full-batch Adam (with a learning rate of $1$, cosine decay to $0.01$, and $1000$ steps) with cross-entropy loss over the scores:
    \[
        \mathcal{L}_{\mathrm{dec}} = \mathrm{CE}(S,f).
    \]
    After training, we normalize the rows of $\mathbf{C}$ and $\mathbf{M}$ for numerical stability, and $(\mathbf{C},\mathbf{M})$ replaces the analytic JL decoder in the full construction.
\end{itemize}

Each constructed MLP is uniquely identified by its encoder/decoder pair (Bin+JL, GD+JL, Bin+GD, GD+GD).

In the sweeps, the decoder width $m_{\text{dec}}$ is the capacity parameter for the Bin+JL and Bin+GD construction variants. For the GD+JL and GD+GD variants, we use a two-step procedure. First, we sweep over the decoder width $m$, obtaining the smallest value $\hat{m}$ for which the constructed MLP achieves perfect fact-storage accuracy. Next, we fix $m = \hat{m}$ and further sweep over the \emph{encoder width multiplier} to find the smallest value in the range $[0, 2]$ for which the MLP achieves perfect accuracy.

\item \textbf{NTK MLPs.}
We also evaluate the Hermite-feature construction of~\citet{nichani2024understandingfactualrecalltransformers}, which we refer to throughout as ``NTK MLPs''.

Given key embeddings $\mathbf{K}\in\mathbb{R}^{F\times d}$, value embeddings $\mathbf{V}\in\mathbb{R}^{F\times d}$, and a mapping $f:[F]\to[F]$, the NTK MLP of width $h$ is constructed as in~\Cref{alg:ntk_mlp_construction}.

\begin{itemize}
\item
We first (optionally) replace $\mathbf{V}$ by the \emph{minimum-margin output embeddings} $\mathbf{U}^\star$: in our ablations, we find this improves fact-storage capacity by $2$-$4\times$ (\Cref{fig:ntk_trainU_rho_appendix}).

\item
We then apply the construction from~\citet{nichani2024understandingfactualrecalltransformers}. Crucially, although~\citet{nichani2024understandingfactualrecalltransformers}'s Theorem 2 describes a \emph{non-gated} MLP construction, in fact their work first defines a \emph{gated} MLP, then uses an NTK argument to show that a non-gated MLP can be used to approximate the gated MLP by rescaling the magnitudes of the MLP weights. In our experiments, we find the non-gated MLP exhibits large Lipschitz constant, making it impractical to use within a Transformer; as such, we directly implement their gated MLP without the NTK approximation.
\end{itemize}

The resulting gated MLP has the form
\[
\mathbf{g}_{\mathrm{NTK}}(\mathbf{x})
\;=\;
\mathbf{P}\Big(
  \sigma(\mathbf{W}_{\mathrm{gate}}\mathbf{x})
  \odot
  (\mathbf{W}_{\mathrm{up}}\mathbf{x})
\Big),
\]
with $\sigma$ equal to the chosen activation.
In our experiments, mirroring the GD and our constructed MLPs, we use $\sigma = \text{Swish}$.

In the sweeps, the hidden dimension $h$ is the sole capacity parameter for NTK MLPs, and we perform binary search over $h$ exactly as for GD MLPs.

Note that~\citet{nichani2024understandingfactualrecalltransformers} proposes their construction for uniformly spherically distributed key and value embeddings that are \emph{not tied}; in our experiments, we evaluate how well the NTK MLP construction can generalize to more realistic settings, such as tied + anisotropic embeddings.

\begin{algorithm}[t]
\caption{NTK MLP Construction}
\label{alg:ntk_mlp_construction}
\begin{algorithmic}[1]
\REQUIRE Keys $\mathbf{K} \in \mathbb{R}^{F \times d}$, values $\mathbf{V} \in \mathbb{R}^{F \times d}$, mapping $f:[F]\to[F]$
\REQUIRE Hidden width $h$, activation choice $\sigma$, Hermite degree $k$, finite-difference step $\varepsilon$ (for plain MLP)
\REQUIRE Flag \texttt{margin\_optimal} (whether to use $\mathbf{U}^\star$)

\vspace{4pt}
\IF{\texttt{margin\_optimal} is True}
  \STATE $\mathbf{V} \gets \mathbf{U}^\star$ \COMMENT{margin-optimal output embeddings}
\ENDIF

\vspace{4pt}
\STATE Sample gate weights $\mathbf{W}_{\mathrm{gate}} \sim \mathcal{N}(0,1)^{h\times d}$
\STATE Sample $\mathbf{P}_{\mathrm{raw}} \sim \mathcal{N}(0,1)^{d\times h}$ and normalize each column to unit norm to obtain $\mathbf{P}$

\vspace{4pt}
\STATE $\mathbf{Z} \gets \mathbf{K}\mathbf{W}_{\mathrm{gate}}^\top \in \mathbb{R}^{F\times h}$ \COMMENT{project inputs}
\STATE Choose Hermite degree $k$ (from activation or configuration)
\STATE $\mathbf{H} \gets \widehat{\mathbf{H}}_k(\mathbf{Z}) \in \mathbb{R}^{F\times h}$ \COMMENT{degree-$k$ normalized Hermite features}

\vspace{4pt}
\STATE $\mathbf{Y} \gets [\mathbf{V}_{f(0)};\ldots;\mathbf{V}_{f(F-1)}] \in \mathbb{R}^{F\times d}$ \COMMENT{reorder values by $f$}
\STATE $\mathbf{A} \gets \mathbf{Y}\mathbf{P} \in \mathbb{R}^{F\times h}$ \COMMENT{feature coefficients}
\STATE $\mathbf{W}_{\mathrm{up}} \gets \frac{1}{h} (\mathbf{H} \odot \mathbf{A})^\top \mathbf{K} \in \mathbb{R}^{h\times d}$

\vspace{6pt}
\textbf{return} the gated MLP:
\[
  \mathbf{g}(\mathbf{x})
  \;=\;
  \mathbf{P}\Big(
    \sigma(\mathbf{W}_{\mathrm{gate}}\mathbf{x})
    \odot
    (\mathbf{W}_{\mathrm{up}}\mathbf{x})
  \Big)
\]

\end{algorithmic}
\end{algorithm}

\end{itemize}

\paragraph{Computing margin-optimal output embeddings.}
For both our constructed MLPs and the NTK baseline, we optionally replace the original value embeddings $\mathbf{V}\in\mathbb{R}^{F\times d}$ by a new set $\mathbf{U}^\star$ obtained by maximizing the dot-product decoding margin (as in~\Cref{def:optimal_output_embeddings}).  
Specifically, for each $i$ we solve the convex optimization problem
\[
\max_{\|u\|_2\le 1} \;\min_{j\neq i}
\frac{\langle \mathbf{v}_i - \mathbf{v}_j,\; u\rangle}{\|\mathbf{v}_i - \mathbf{v}_j\|_2},
\]
and denote the optimizer by $u_i^\star$.  
We solve these problems using ADMM.

\paragraph{Embedding whitening.}
For anisotropic value embeddings, we optionally apply a ZCA whitening preconditioning step prior to training or construction.  
Given an embedding matrix $E\in\mathbb{R}^{F\times d}$ (keys or values), we estimate its second-moment matrix
\[
\Sigma \;=\; \frac{1}{F} E^\top E,
\qquad 
\tilde{\Sigma} \;=\; \Sigma + \varepsilon I_d
\]
with a small ridge $\varepsilon\!\approx\!10^{-6}$ to ensure invertibility.  
Let $\tilde{\Sigma} = Q \Lambda Q^\top$ be the eigendecomposition, where $Q$ is orthonormal and $\Lambda=\mathrm{diag}(\lambda_1,\dots,\lambda_d)$ with $\lambda_i>0$.  
Full ZCA whitening corresponds to the transform
\[
W_{\mathrm{zca}} \;=\; Q\,\Lambda^{-1/2}\,Q^\top.
\]
We also investigate \emph{interpolating} between no whitening and full whitening using a strength parameter $\alpha\in[0,1]$:
\[
W_\alpha \;=\; W_{\mathrm{zca}}^{\alpha}.
\]

Before training or construction, we replace $E$ by the whitened embeddings $E_{\mathrm{white}} = E\,W_\alpha$.
The inverse transform $W_\alpha^{-1}$ is then folded into the final linear block of the resulting MLP, so that the MLP output remains in the original embedding basis.

\subsubsection{Ablations}
\label{ssec:mlp_ablations_appendix}

\paragraph{Effect of margin-optimal output embeddings on NTK MLPs.}
\Cref{fig:scaling-rho} shows that NTK MLPs fail to achieve perfect fact storage once the value embeddings become sufficiently anisotropic.  
Here, we investigate whether applying the NTK construction to the \emph{margin-optimal output embeddings} $\mathbf{U}^\star$ improves its robustness.
As shown in \Cref{fig:ntk_trainU_rho_appendix}, although replacing the raw value embeddings by $\mathbf{U}^\star$ improves fact-storage capacity by a factor of $2$-$4\times$, the NTK construction still breaks down once the condition number exceeds a moderate threshold.
In contrast, both GD MLPs and our constructed MLPs maintain consistent scaling across a broad range of anisotropic embeddings.

\begin{figure}[h]
    \centering
    \includegraphics[width=0.55\linewidth]{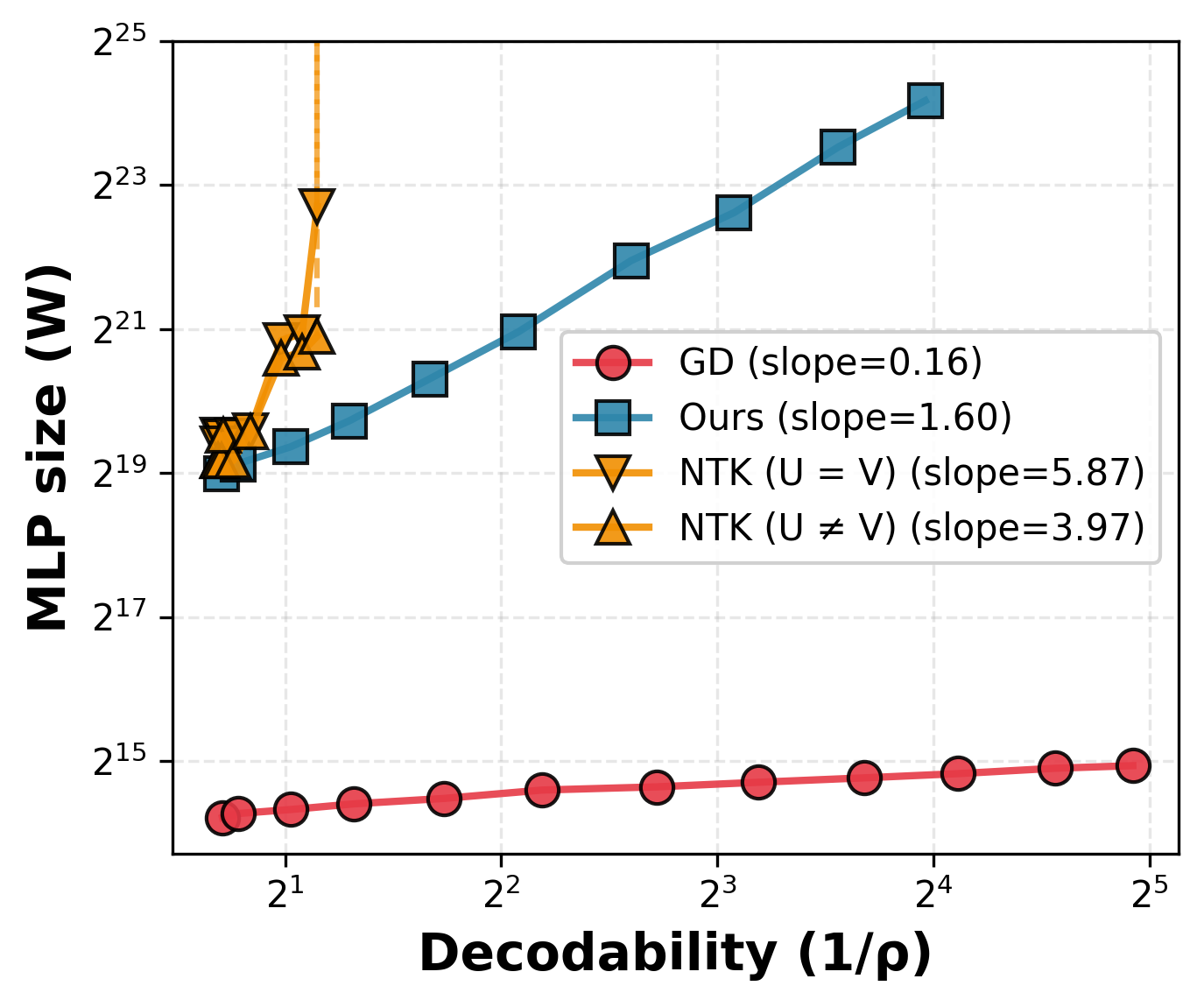}
    \caption{NTK MLPs fail to achieve perfect fact storage for sufficiently anisotropic output embeddings. Using the margin-optimal output embeddings for the NTK construction improves fact-storage capacity by up to $4\times$, but does not improve robustness to anisotropic embeddings.}
    \label{fig:ntk_trainU_rho_appendix}
\end{figure}

\paragraph{Coherence exhibits weak predictive power for fact-storage capacity.}
\Cref{fig:coherence_appendix} compares fact-storage capacity against the coherence of the embedding matrix, a commonly used measure of geometric spread.  
Unlike our decodability statistic $\rho(\mathbf{V})$, coherence does not strongly correlate with the number of parameters needed to store a fixed number of facts; this is true for both GD MLPs ($R^2 = 0.10$) and our constructed MLPs ($R^2 = 0.44$).
This supports our use of $\rho$, rather than coherence or related spectral heuristics, as a natural predictor of separability for the decoder and, ultimately, of fact-storage capacity.

\begin{figure}[h]
    \centering
    \includegraphics[width=0.55\linewidth]{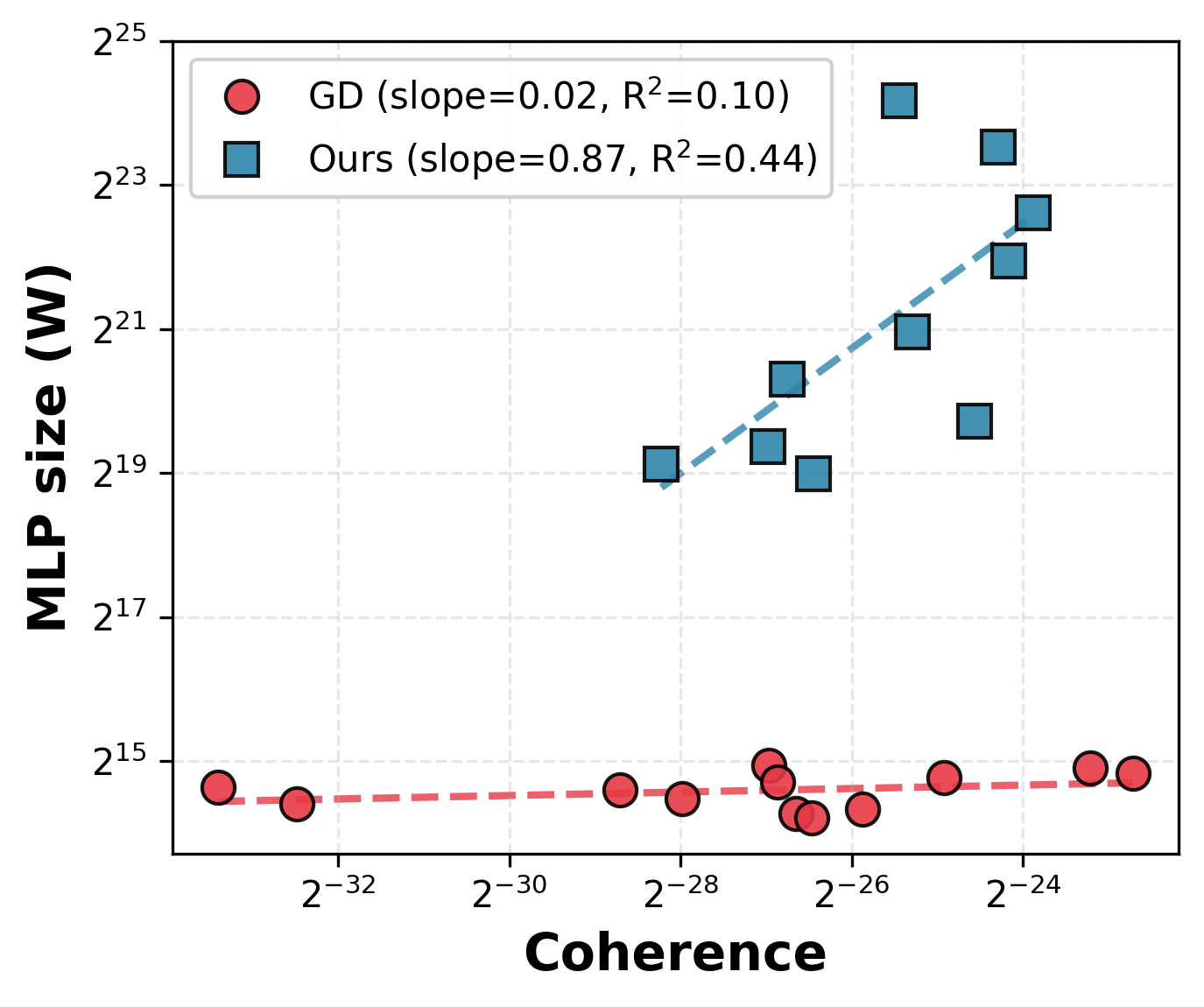}
    \caption{Unlike our decodability metric, $\rho$, coherence is not strongly predictive of fact-storage capacity for GD nor our constructed MLPs.}
    \label{fig:coherence_appendix}
\end{figure}

\newpage
\subsection{SSFR Experiments}
\subsubsection{SSFR Task}\label{appendix:ssfr-task}
We introduce the SSFR task to evaluate a model’s ability to retrieve facts stored in its weights. In this task, the model is presented with a sequence containing a single key token surrounded by “junk” tokens and is required to output the corresponding value token according to the task's \textit{fact set}. 

Formally, let $f : \mathcal{S}_k \to \mathcal{S}_v$ be a fact set over tokens $\mathcal{S}_k \, \cup \,  \mathcal{S}_v$. Let $\mathcal{J}= \{(j^{\text{prefix}}_1, j^{\text{suffix}}_1), (j^{\text{prefix}}_2, j^{\text{suffix}}_2), \dots\}$ be the set containing junk prefixes and suffixes tuples. The SSFR task is then defined as the set of sequences:
\[
    \mathcal{S}_{SSFR}[f] = \{\text{concat}(j_{\text{prefix}},\, k,\, j_{\text{suffix}}, f(k)) \;|\; k\in \mathcal{S}_k, \;\; (j^{\text{prefix}}, j^{\text{suffix}})  \in \mathcal{J}\}.
\]

The model’s task, given a sequence from $\mathcal{S}_{SSFR}[f]$, is then to predict $f(k)$ as the final token of the sequence. For example, given the sequence
\[
\underbrace{*\hspace{0.7em} \%\hspace{0.7em} \& \hspace{0.7em} \# \hspace{0.7em} \$}_{\text{junk prefix}} \ 
 \underbrace{\hspace{0.7em} A \hspace{0.7em}}_{\text{key}} \ 
\underbrace{*\hspace{0.7em} \%\hspace{0.7em} \& \hspace{0.7em} \# \hspace{0.7em} \$}_{\text{junk suffix}}
 \underbrace{\hspace{0.7em} B \hspace{0.7em}}_{\text{value}} \ 
\]
from $\mathcal{S}_{SSFR}[f]$, the model's task is to predict the final token $B=f(A)$.

In practice, across all of our experiments, the junk prefix and junk suffixes have a length between 8 and 16. Further, the amount of junk prefixes and suffixes tuples we use, i.e. $|\mathcal{J}|$, is 16. Finally, we reserve 16 additional tokens (to those representing the keys and values of the fact-set), as the junk tokens.

\subsubsection{Training Setup}\label{app:ssfr-training-setup}
The setup we use to train transformers using fact-storing MLPs in all SSFR experiments is as follows:
\begin{enumerate}
    \item Randomly sample the transformer embeddings for the key, value and junk tokens from a standard normal distribution. We optionally ill-condition the embeddings, as in the MLP fact-storage capacity experiments (Appendix~\ref{ssec:mlps_sweeps_task_setup_appendix}). We do not ill-condition embeddings unless stated otherwise.
    \item Randomly sample a fact set.
    \item Compute the MLP embeddings. To obtain the MLP key embeddings, we just project all the transformer key embeddings to the unit sphere (since the transformer stack forwards them through a normalization layer before feeding them to the MLP). The MLP value embeddings stay the same as the transformer value embeddings.
    \item Construct or train with gradient-descent a fact-storing MLP that stores the fact set under the MLP embeddings.
    \item Train the modified transformer, as outlined in \cref{sec:ssfr}, with frozen key and value transformer embeddings, in the SSFR task corresponding to the fact set we sampled.
\end{enumerate}

\paragraph{Constructed / GD MLPs Setup.}
Across our SSFR experiments, we use constructed and GD fact-storing MLPs as outlined in Appendix~\ref{ssec:mlp_sweeps_architectures_appendix}.

\paragraph{Transformer Setup.}
Across all our SSFR experiments we use a modified 1-layer GPT2 transformer~\citep{radford2019language, Karpathy2022} with RoPE~\citep{su2023roformerenhancedtransformerrotary} positional embeddings, frozen key and value transformer embeddings, RMSNorm normalization layers, single-head attention. Moreover, as outlined in \cref{sec:ssfr}, we tie the transformer and MLP embeddings, remove residual connections, freeze the RMSNorm before the MLP (so that it just projects to the unit sphere) and freeze the \textit{value} and \textit{out-project} matrices of the attention layer to the identity matrix. Across all experiments, we train transformers on a total of 4.8M sequences randomly sampled from the SSFR task, or until convergence, using an AdamW optimizer, with a learning rate of $2\times10^{-4}$ unless stated otherwise.

\subsubsection{MLP Size v.s. Facts}\label{appendix:mlp-size-vs-facts}
\begin{figure}[h]
    \centering
    \includegraphics[width=0.6\textwidth]{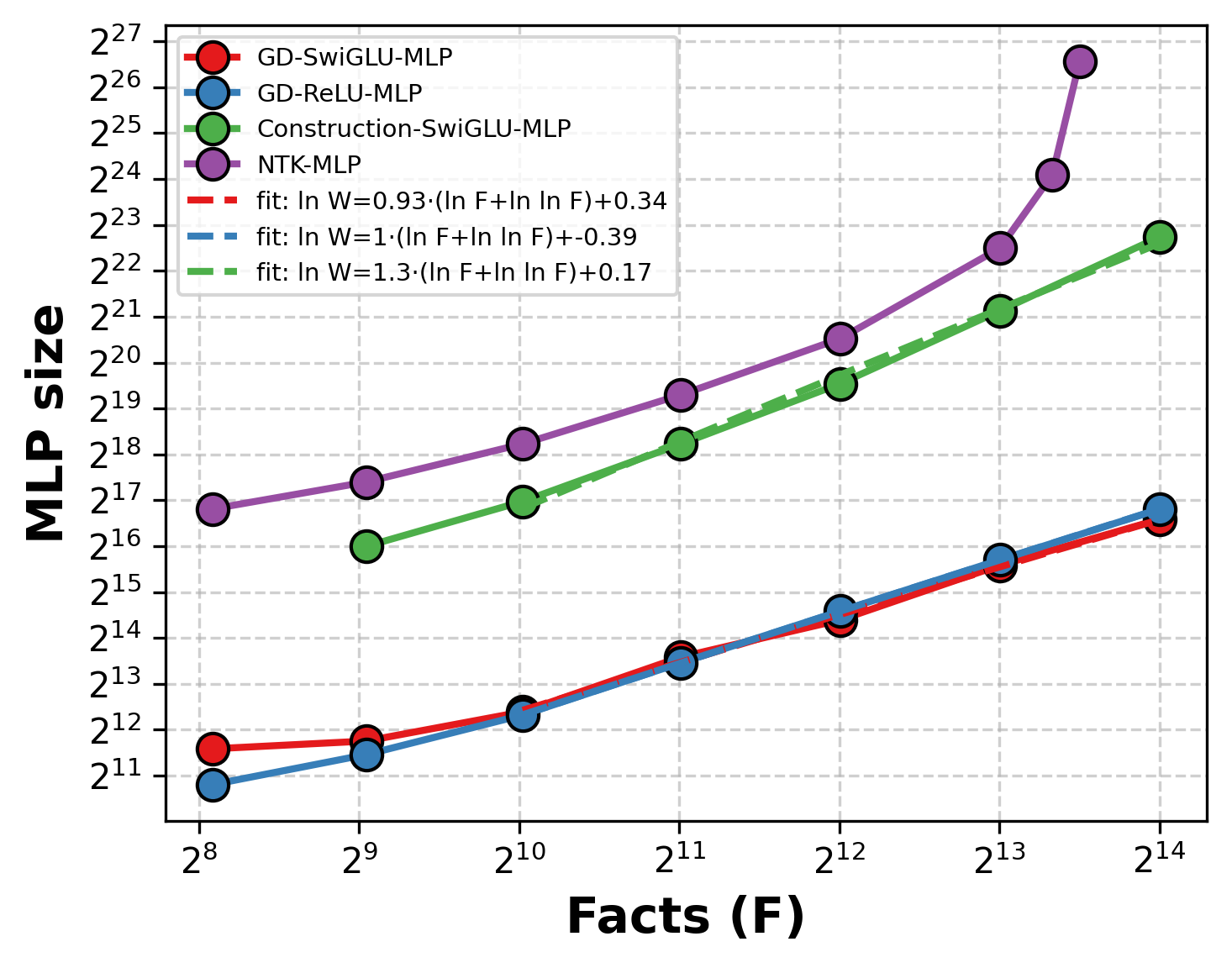}
    \caption{MLP size vs. fact-set size for MLPs with $\ge 99\%$ usability within a Transformer, including ReLU MLPs.}
    \label{fig:app-fact-adaptive-front-full}
\end{figure}

In our MLP size (W) v.s. Facts (F) scaling experiments, presented in \cref{sec:ssfr} and observed in \cref{fig:ssfr_main_fig}.a and \cref{fig:app-fact-adaptive-front-full}, we seek to find the smallest MLP size such that the MLP is usable for factual recall by a transformer. We determine whether an MLP is usable by a transformer by testing whether its \textit{fact-adaptive accuracy} is $>99\%$. To this end, we take a transformer using a fact-storing MLP with embedding-dimension $d=128$ and run a binary search to find the minimum hidden size $h$ needed to store every fact-set size $F\in\{2^8, \dots, 2^{14}\}$. In this binary search, to reduce noise, we run each experiment corresponding to an MLP size with 4 seeds and take the maximum \textit{fact-adaptive accuracy} out of them. We then report the total MLP size v.s. \# of Facts curve outlined by our binary search results.

\subsubsection{MLP Usability v.s. Capacity}\label{appendix:usability-vs-storage-capacity}
\begin{figure}[h]
    \centering
    \includegraphics[width=0.9\textwidth]{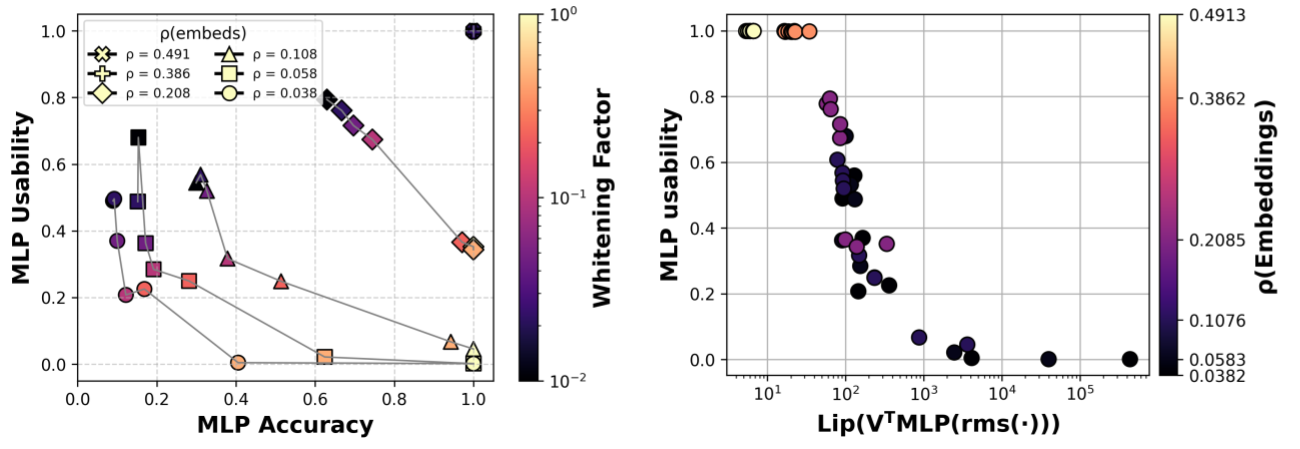}
    \caption{\textbf{(a) MLP usability within Transformer v.s. MLP storage capacity for a ReLU MLP.} We observe a tradeoff between MLP usability within a Transformer and the MLP's fact-storage capacity. \textbf{(b) MLP usability within Transformer v.s. its Lipschitz constant for a ReLU MLP.} We observe that the measured Lipschitz constant is predictive of an MLP's usability within Transformers.}
    \label{fig:ssfr_relu}
\end{figure}

In our MLP Usability v.s. Accuracy experiments, we study the effect of embedding whitening on the usability v.s. accuracy tradeoff of GD fact MLPs (trained with Cross-Entropy loss), as outlined in \cref{subsec:tradeoff}. Concretely, we look at transformers using SwiGLU and ReLU fact MLPs, with $d=128$ and hidden size $m=1.1h^*$, where $h^*$ is the hidden dimension size found in our scaling experiments from \cref{fig:app-fact-adaptive-front-full}.

Concretely, for SwiGLU MLP's we study ill-conditioned transformer embeddings with $\kappa(\mathbf{K}_{t})=\kappa(\mathbf{V}_{t}) \in \{1.1\times10^{0},\;1.0\times10^{1},\;2.5\times10^{1},\;5.0\times10^{1},\;2.5\times10^{2},\;1.0\times10^{3},\;1.0\times10^{4},\;1.0\times10^{6}\}
$, yielding a varied spectrum of $\rho$ values, as observed in \cref{fig:ssfr_main_fig}.b.

In addition, for ReLU MLPs, we look at transformer embeddings with $\kappa(\mathbf{K}_{t})=\kappa(\mathbf{V}_{t}) \in \{1.1\times10^{0},\;1.0\times10^{1},\;1.0\times10^{2},\;1.0\times10^{3},\;1.0\times10^{4},\;1.0\times10^{5}\}$, yielding a varied spectrum of $\rho$ values, as observed in \cref{fig:ssfr_relu}.a.

Further, for every $\rho$, we study the whitening degrees $\alpha\in\{0.0, 0.01, 0.022, 0.046, 0.1, 0.22, 0.46, 1.0\}$. To reduce noise, for every combination of $\alpha, \rho$, we run experiments for the learning rates $lr\in\{2\times10^{-6}, 2\times10^{-5}, 2\times10^{-4}, 2\times10^{-3}, 2\times10^{-2}\}$ with 4 seeds each, keeping the transformer with the largest \textit{fact-adaptive} accuracy.

\subsubsection{MLP Usability v.s. Lipschitz constant} \label{appendix:usability-vs-conditioning}
In our MLP Usability v.s. Lipschitz constant experiments, we study the variation of MLP Usability v.s. an approximation of the Lipschitz constant, as outlined in \cref{subsec:conditioning} and observed in \cref{fig:ssfr_main_fig}.c and \cref{fig:ssfr_relu}.b. Concretely, for every transformer obtained in our MLP Usability v.s. Accuracy experiments \cref{subsec:tradeoff}, we approximate its fact-storing MLP's Lipchitz constant as the maximum out of 100 random $\mathbf{k}_i$ samples of \cref{eq:lip}.

\subsection{Language Modeling Experiments}
\subsubsection{Authors and Books Dataset}\label{appendix:lm-task}
We introduce a simple language modeling (LM) task to evaluate a transformer’s ability to perform next-token prediction while recalling factual information. In this task, the model is presented with a natural-language sentence expressing a \((\textit{book}, \textit{author})\) relation and is required to predict each subsequent token in the sequence. Notably, we curate this dataset using author-books relations from the Goodreads Book Graph Dataset \citep{authors_dataset}.

Formally, let \(f : S_k \to S_v\) be the authors \textit{fact set}, where  
\(S_k = \{\text{``It''},\ \text{``1984''},\ \text{``And Then There Were None''},\ \ldots\}\) is the set of book titles (keys) and  
\(S_v = \{\text{``Stephen King''},\ \text{``George Orwell''},\ \text{``Agatha Christie''},\ \ldots\}\) is the set of corresponding authors (values).  
To simplify analysis, we select exactly one book per author.  
Let \(J = \{(\text{``The author of''},\ \text{``is''}),\ (\text{``Who is the author of''},\ \text{``? It is''}),\ \ldots\}\) denote the set of natural-language template prefix--suffix pairs.  
The LM task given \(f\) can then be defined as:
\[
    \mathcal{S}_{LM}[f] = \{\text{concat}(t_{\text{prefix}},\ k,\ t_{\text{suffix}}, f(k)) \ |\ (t_{\text{prefix}}, t_{\text{suffix}}) \in J,\ k \in S_k \}.
\]

For example, given the sequence:
\[
\underbrace{\text{The author of}}_{\text{template prefix}}\ 
\underbrace{1984}_{\text{key}}\ 
\underbrace{\text{is}}_{\text{template suffix}}\ 
\underbrace{\text{George Orwell}}_{\text{value}}
\]
from \(\mathcal{S}_{LM}[f]\), the model’s task is to perform next-token prediction \textit{at every position} in the sentence. This LM task allows us to study factual recall in a more natural language modeling setting, complementing the SSFR setup.

\subsubsection{Training Setup}
The setup we use to train transformers using fact-storing MLPs in the Language Modeling experiments is the same as that outlined in \cref{app:ssfr-training-setup}. However, instead of using a random fact set, we use the authors and books fact-set and use uniformly sampled embeddings.

\paragraph{GD MLP Setup.}
Notably, in our LM experiments, we only use GD trained fact-storing MLPs, which are trained in a MSE objective (as opposed to a Cross-Entropy objective) to store the fact set under arg-max decoding. Concretely, these MLPs are trained to minimize $L_{MLP}(\mathbf{K}, \mathbf{V}, f) \propto \sum_{i=1}^{|\mathbf{K}|}||MLP(\mathbf{k}_i) - \mathbf{v}_{f(i)}||_2^2$.

\paragraph{Transformer Setup}
In our LM experiments, we use a similar setup as that outlined in \cref{app:ssfr-training-setup}, with some additional modifications we find empirically helpful:
\begin{itemize}
    \item Replace the state-mixer of the transformer with a Mixture-of-Experts (MoE) module with 2 experts and an MLP router. Concretely, we use a \textit{fact-expert}, which is the frozen fact-storing MLP and a \textit{language-expert}, which is a trainable low-rank linear layer. Intuitively, this MoE setup enables the transformer to selectively use the fact-storing MLP only for factual recall.
    \item Parametrize the \textit{query} and \textit{key} projections in the attention module with MLPs.
\end{itemize}

\subsubsection{MLP Size v.s. Facts}\label{appendix:lm-scaling}
\begin{figure}[h]
    \centering
    \includegraphics[width=0.9\textwidth]{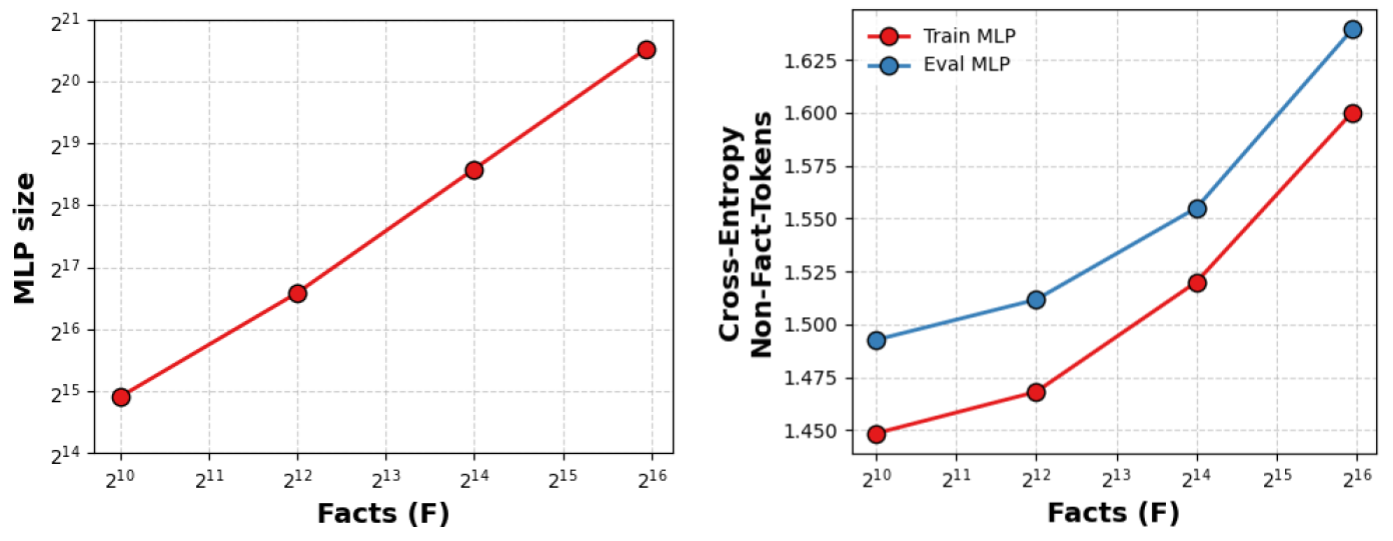}
    \caption{\textbf{(a) MLP size vs. fact-set size for MLPs with $\ge 99\%$ usability in LM task within a transformer.} Notably, fact MLPs are usable within transformers for Language Modeling. \textbf{(b) CE Loss of on non-fact tokens on a LM task for the transformers in \cref{fig:lm-fact-adaptive-fronteir}.a after swapping their fact-storing MLP for different one.} Notably, the CE Loss of the transformers decays minimally ($\sim3$\%) when replacing the original MLP (train) with another one storing a different fact-set (eval).}
    \label{fig:lm-fact-adaptive-fronteir}
\end{figure}

Similar to \cref{appendix:mlp-size-vs-facts}, we perform MLP size (W) v.s. Facts (F) scaling experiments for our transformers, equipped with GD fact MLPs, in the LM task. Concretely, we take transformers equipped with SwiGLU fact MLPs with $d=256$ and use a binary search with 4 seeds per experiment to determine to find the smallest MLP size W s.t. a transformer can use such an MLP for factual recall on a fact set of size F. As can be observed in \cref{fig:lm-fact-adaptive-fronteir}.a, our transformers can use fact-storing MLPs for factual recall with reasonable scaling in facts per parameter. Furthermore, each of these transformers only suffer a small decay of $\sim3\%$ in average Cross-Entropy loss for the non-fact tokens of the LM task (e.g. "The",  "author" "of", etc.) when their MLP is swapped by another one storing a different fact-set (i.e. a different mapping from books to authors).

\subsubsection{Fact Editing} \label{appendix:fact-editing}

We evaluate fact-editing methods in the same setting used for our Language Modeling experiments. Concretely, we use the model obtained in those experiments storing 16,000 author-book facts, each represented by 16 rephrases.

To study how different fact-editing approaches behave, we divide the fact set into two subsets: a \emph{preserved fact set}, whose facts the editor should maintain, and an \emph{altered fact set}, whose facts the editor should modify. We run experiments using several combinations of preserved/altered fact set sizes: $\{(6554, 1638), (3277, 819), (1311, 327)\}$, which are subsets of the original fact set of 16,000 facts.

We evaluate each editing method using three standard metrics. \emph{Specificity} measures accuracy on the altered-fact set, indicating how well the method performs the intended edits. \emph{Efficacy} measures accuracy on the preserved-fact set, capturing whether the method avoids unintended side effects. \emph{Paraphrase} evaluates the accuracy on paraphrases of the altered facts, measuring how well edits generalize beyond the training prompts. We also report a \emph{Score}, defined as the harmonic mean of these three metrics.

We compare four editing methods. Our method, \emph{MLP swapping}, trains an MLP to store the full altered-fact set and swaps it into the transformer in place of the original fact-storing MLP. The remaining three methods: MEMIT \citep{memit}, AlphaEdit  \citep{fang2025alphaeditnullspaceconstrainedknowledge}, and ROME \citep{rome}, are existing weight-update-based editors, which are set up to alter the \emph{altered fact set} and preserve the \emph{preserved fact set}. Because these methods are designed for large language models and real-world text, we adapt them to our simplified 1-layer transformer setup. For each, we perform a grid search over its hyperparameters and report the accuracies corresponding to the configuration achieving the best overall score.
\begin{itemize}
    \item \textbf{MEMIT:} We search over $\texttt{train\_steps} \in \{10, 25, 100\}$, $\texttt{lr} \in \{0.005, 0.05, 0.5\}$, $\lambda \in \{1.5\times10^4, 1.5\times10^3, 1.5\times10^2, 1\}$, and $\texttt{clip\_norm} \in \{0.5, 0.75, 1\}$.
    \item \textbf{AlphaEdit:} We search over $\texttt{train\_steps} \in \{10, 25, 100\}$, $\texttt{lr} \in \{0.005, 0.05, 0.5\}$, $\texttt{clip\_norm} \in \{0.5, 0.75, \text{None}\}$, and $\texttt{singular\_value\_tolerance} \in \{10^{-2}, 1, 10\}$.
    \item \textbf{ROME:} We search over $\texttt{train\_steps} \in \{10, 25, 100\}$, $\texttt{lr} \in \{0.005, 0.05, 0.5\}$, $\texttt{wd} \in \{1.5\times10^{-3}, 1.5\times10^{-4}, 0\}$, and $\texttt{early\_stopping\_loss} \in \{5\times10^{-2}, \text{None}\}$.
\end{itemize}

For these methods, we apply residual updates to the output of the MLP inside the MoE module on the final token of the input prompt. We find this appropriate since our transformer has a single layer, so the fact-storing MLP directly precedes the logits without any intervening attention layers. Moreover, we do not introduce random token prefixes when computing residual vectors. Instead, we use a single templated prompt per fact. In addition, for ROME, we omit the KL-divergence term from the residual computation given the simplicity of our dataset, where each subject (author) appears in only one relation, mapping uniquely to a book.

\newpage
\section{Theoretical Results}\label{app:extended_theory}
This section is organized as follows: 
\begin{enumerate}
    \item In \Cref{sec:notaion_external_res} we discuss notation and external results that will be useful throughout the appendix.
    \item In \Cref{sec:info_theory_bound} we provide additional preliminary information on softmax decoding and fact storage capacity in support of \Cref{subsec:definitions_maintext}.
    \item In \Cref{theory_sec_3.1} we detail our encoding construction in support of \Cref{subsec:encoder_twohot}.
    \item In \Cref{sec:new_decoding_results} we prove bounds on $\rho$, and detail our decoding construction in support of \Cref{subsec:construction_decoder}.
    \item In \Cref{theory_sec_3.3} we prove our full construction in support of \Cref{subsec:full_construction}.
    \item In \Cref{app:subsec_embeddings_theory} we explore the interaction between $\rho$ and transformations on embeddings in support of \Cref{sec:preconditioning}. 
    \item In \Cref{app:subsec_bit_complexity} we prove that our construction has bounded bit complexity.  
    \item In \Cref{app:subsec_spherical_shevyshev} we prove bounds on the spherical Chebyshev value. 
    \item In \Cref{app:extended_theory_proofs} we collect deferred proofs from the previous sections. 
\end{enumerate}

\subsection{Notation and External Results}\label{sec:notaion_external_res}
All vectors are denoted by bold lower case letters ($e.g.,\vx$), and matrices by bold uppercase letters ($e.g.,\rmV$). All vectors are assumed to be in column form and indices will start from 1. We denote $\mathbb{S}^{\modelDim-1}$ to be the unit sphere in $\R^d$.

For a set $\mU=[\vu_1, \hdots, \vu_N^{\top}]$ with $\vu_i\in\mathbb S^{d-1}$, set
\[
\rho(\rmU;\rmV)\;:=\;\min_{i\neq j}\ \frac{\langle \vv_i-\vv_j,\;\vu_i\rangle}{\|\vv_i-\vv_j\|_2},
\qquad
\rho(\rmV)\;:=\;\max_{\mU}\ \rho(\rmU;\rmV).
\]

We use the former definition of $\rho$ in several sections of the appendix as it is somewhat easier to work with. 

We generally abbreviate $\|x\|_2$ to $\|x\|$; other norms are explicitly marked. We occasionally use $|\cdot|$ to denote the number of rows in a matrix (ie. $| \rmK |$ = \# of rows in $\rmK$). Additionally, note that $O(d)$ is the set of $d\times d$ orthonormal matrices and is distinguishable from Big-O notation by the type of its elements (eg. $\rmU \in O(d)$). 

A random vector $x \in \R^d$ is \emph{rotationally invariant} if
\[
Vx \sim x \qquad \forall V \in O(d),
\]
i.e., its distribution depends only on $\|x\|_2$ and not on its direction (e.g.\ $x \sim \mathcal N(0,I_d)$). When we say the keys are rotationally invariant, we mean they are i.i.d.\ draws from such a distribution.

\subsubsection{The Bubeck Result}

Fix some dataset $\mathcal{D}=\{(\vx_i, y_i)\}_{i \in [n]} \subset (\mathbb R^d \times \mathbb R)^n$. Let $\mathcal F_k$ be the set of functions of the form 

\begin{align*}
    f(\vx) = a^\top \ReLU(\mW\vx + b)
\end{align*}
where $\va = (a_1, ..., a_k)^\top \in \mathbb R^k$, $b = (b_1, ..., b_k)^\top \in \mathbb R^k$, and $\mW \in \mathbb R^{k \times d}$ with rows $\vw_1^\top, ..., \vw_k^\top$. Denote $\mathbf y = (y_1 , ..., y_n)$ and $\mathbf f = (f(\vx_1), ..., f(\vx_n))$ with $f \in \mathcal F, \; f:\mathbb R^d \rightarrow \mathbb R$. Note that this is equivalent to the definition in \citep{bubeck2020networksizeweightssize}.

We will use the following result from \citep{bubeck2020networksizeweightssize}: 
\begin{theorem}\label{thm:bubeck_result}
    Let $(\vx_i)_{i \in [n]}$ be in general position in $\mathbb R^d$ (i.e., any hyperplane contains at most $d$ points). Then there exists $f \in \mathcal F_{4 \cdot \lceil \frac{n}{d} \rceil}$ such that $\mathbf f = \mathbf y$.
\end{theorem}
We now provide a proof sketch of the result to provide intuition. For a full proof, see Proposition 4 of \citep{bubeck2020networksizeweightssize}.

\begin{proof}
Split the $n$ samples into $r=\lceil n/d\rceil$ disjoint sets of indices $S_1,\dots,S_r$ of size $d$ (last may be smaller).
By general position, for each block $S$ there is a hyperplane $H_S=\{\vx:\ \vz_S\!\cdot \vx=b_S\}$ that contains exactly $\{\vx_i:i\in S\}$.

Define the function, for small enough $\delta>0$:
\[
g_{\vz,\vv,b,\delta}(\vx)\ :=\ \frac{\ReLU((\vz+\delta \vv)\!\cdot \vx-b)-\ReLU(\vz\!\cdot \vx-b)}{\delta}.
\]
If $\delta$ preserves the signs of $\vz\!\cdot \vx_i-b$ for all data (i.e., if no input crosses the ReLU boundary), then
\[
g_{\vz,\vv,b,\delta}(\vx_i)=\begin{cases} \vv\!\cdot \vx_i,& \vz\!\cdot \vx_i>b,\\ 0,& \vz\!\cdot \vx_i<b.\end{cases}
\]
Set
\[
h_{\vz,\vv,b,\delta}(\vx):=g_{\vz,\vv,b - \tau,\delta}(\vx) - g_{\vz,\vv,b + \tau,\delta}(\vx),
\]
for small enough $\tau > 0$. We then have that $h_{\vz_S,\vv,b_S,\tau, \delta}(\vx_i)=\vv\!\cdot \vx_i$ for $\vx_i \in S$ and $0$ otherwise. Choices which always work are $0 < \tau < \frac{1}{2} \min_{i \nin S} |\vu_S \cdot \vx_i - b_S|$ and $\delta \leq \frac{1}{2} \min_{i \in [n]} \min_{\sigma \in \{-1, 1\}} \frac{|\vz_S \cdot \vx_i - (b_S + \sigma \tau)|}{|\vv \cdot \vx_i|}$. 

Pick $S_i$ such that $X_{S_i}$ (the matrix collecting all $\vx_j \in S_i$), by general position of $\vv_i$s, has full rank for all $i$. For each block $S$, solve $X_S \vv_S=y_S$ and define
$f_S(\vx):=h_{\vz_S,\vv_S,b_S,\tau, \delta}(\vx)$.
Then $f_S(\vx_i)=y_i$ for $i\in S$ and $0$ for $i\notin S$.

Finally,
\[
f(\vx):=\sum_{t=1}^r f_{S_t}(\vx)\ \in\ \mathcal F_{4r}=\mathcal F_{4\lceil n/d\rceil}
\quad\text{and}\quad
f(\vx_i)=y_i\ \ \forall i\in[n].
\]
\end{proof}

\subsubsection{Johnson-Lindenstrauss Inner Product Preservation}
We will use the following result from \citep{kalavasis2024replicablelearninglargemarginhalfspaces}. 

We say that a random matrix $\mathbf A \in \mathbb R^{k \times d}$ is a $\textit{JL-matrix}$ if either $\mathbf A_{i,j} \sim_{i.i.d} \mathcal N(0, 1/k)$ or $\mathbf A_{i,j} \sim_{i.i.d} U\{-1 / \sqrt{k}, 1 / \sqrt{k}\}$. 

\begin{corollary} \label{jl-ip-result}
    Fix $\epsilon, \delta_{\text{JL}}\in(0,1)$. Let $\rmA\in\R^{k\times d}$ be a $JL$-matrix for $k=\Omega(\epsilon^{-2}\log(\frac{1}{\delta_{\text{JL}}})).$ Then for any $x,z\in\R^{d}$, 
    \[ \Pr_{\rmA}[|\rvz^{\top}\rvx-(\rmA\rvz)^{\top}\rmA\rvx|>\epsilon\|\rvz\|\cdot\|\rvx\|]\le \delta_{\text{JL}}.\]
\end{corollary}

\subsubsection{Sub-gaussian rows}
We will use the following result from \citep{vershynin2018high}.
\begin{theorem} \label{sub-gaussian-rows}
    Let $\rmA$ be an $N\times n$ matrix whose rows $\rmA_i$ are independent sub-gaussian isotropic random vectors in $\R^n$. Then for every $t\ge 0$, with probability at least $1-2\exp(-ct^2)$ one has
    \[\sqrt{N}-C\sqrt{n}-t\le s_{\text{min}}(\rmA)\le \sqrt{N}+C\sqrt{n}+t.\]
    Here $C=C_K, c=c_K >0$ depend only on the subgaussian norm $K=\max_{i}\|\rmA_i\|_{\psi_{2}}$ of the rows. 
\end{theorem}

\subsection{Additional Details on \Cref{subsec:definitions_maintext}}\label{sec:info_theory_bound}
In \Cref{subsec:definitions_maintext} we define what it means for a model to store a fact set. Here, we describe why this is equivalent to outputting the correct value token under softmax decoding, and for completeness provide a proof of \Cref{thm: info_bounds_capacity-const}. We use the definition of softmax decodability as follows.

\begin{definition}\label{def:EHSM-decode-exact}
Let $\codeMatrix\in\mathbb{R}^{\numVectors\times\modelDim}$. A family of output embeddings
$\{\outEmbedding_i\}_{i=1}^{\numVectors}\subset\mathbb{R}^d$ is \emph{softmax-decodable} if there exists a matrix $\rmM\in\R^{\modelDim\times\codeDim}$ such that for all $i$,
\begin{align}\label{eq:alpha-bound}
    \left|\left|\textbf{softmax}_j\left(\langle \rmM\cdot \noisyCodeMatrix[i], \outEmbedding_j\rangle\right) - \ve_i\right|\right|_\infty < \alpha.
    \end{align} \label{eq:softmax-hypothesis}
    \footnote{As a reminder, here the $\textbf{softmax}_j\left(\left\langle \vz, \textbf{y}_j \right\rangle \right)$ notation means a vector of length $n$ where the $\ell$-th coordinate is given, for some arbitrary $\vz\in\R^d$, by: 
    \begin{align*}
        \left(\textbf{softmax}_j\left(\left\langle \vz, \outEmbedding_j \right\rangle\right) \right)_\ell = \frac{\exp\left(\left\langle \vz, \outEmbedding_\ell \right\rangle\right)}{\sum_{k = 1}^n \exp\left(\left\langle \vz, \outEmbedding_k \right\rangle\right)}
    \end{align*}
    
    }
    for some $\frac{1}{2} > \alpha > 0$.
    \end{definition}

In the notation of \Cref{subsec:construction_decoder}, we have $\noisyCodeMatrix[i] := \gaussianMatrix \compressedAuxEmbedding_i$. The following lemma shows that this is equivalent to the provided ``dot-product'' version.

\begin{lemma}\label{lem:softmax_reformulation}
    A set of output embeddings $\{\compressedOutEmbedding_i\}$ is \textit{softmax-decodable} if and only if there exists an $\rmM$ such that, for every $i\neq j$,
    \(
        \left\langle \rmM\cdot \noisyCodeMatrix[i], \compressedOutEmbedding_i\right\rangle > \left\langle \rmM\cdot \noisyCodeMatrix[i], \compressedOutEmbedding_j\right\rangle.
    \)
\end{lemma}
\begin{proof}
    See \Cref{pf:softmax_reformulation}
\end{proof}

The following theorem is a formalized version of \Cref{thm: info_bounds_capacity-const}.

\begin{theorem}[Information‑theoretic capacity bounds]\label{thm:info_capacity}
Let an MLP have $W$ trainable real weights, each stored with a fixed
precision of $p$ bits; write $B = pW = \Theta(W)$ for the total number
of bits that can be set by training.  Let $F$ be the number of
(key,value) pairs (``facts’’) we wish to memorize.

\begin{enumerate}
    \item \textbf{Multi‑valued facts.}  
          If every key may take any of the $F$ values—
          i.e.\ the fact set is a function
          $f:[F]\to[F]$—then any such table representable by the
          network satisfies
          \[
              F \;=\; O\!\bigl(\tfrac{W}{\log W}\bigr).
          \]
    \item \textbf{Binary facts.}  
          If every key is mapped to a bit
          ($f:[F]\to\{0,1\}$) the capacity bound tightens to
          \[
              F \;=\; O(W).
          \]
\end{enumerate}
\end{theorem}

\begin{proof}
Let $\mathcal H$ be the set of hypothesis functions the parameterised
family can express. Because each of the $B=\Theta(W)$ bits can be chosen independently,
\[
        |\mathcal H|
        \;\;\le\;\;
        2^{B}
        \;=\;2^{\Theta(W)} .
\]

In the case of multi-valued facts, there are $F^{F}$ distinct functions
$[F]\to[F]$.
Representability of \emph{all} such maps demands
\[
       2^{\Theta(W)}
       \;\ge\;
       F^{F}.
\]
Taking $\log_{2}$ and rearranging:
\[
       F\log_2 F
       \;=\;O(W)
       \quad\Longrightarrow\quad
       F = O\!\bigl(\tfrac{W}{\log_2 W}\bigr),
\]
since $\log_2 F = \Theta(\log_2 W)$ whenever $F$ grows at most
polynomially in $W$.

For binary facts there are only $2^{F}$ possibilities, so the same
counting gives
\[
        2^{\Theta(W)}
        \;\ge\;
        2^{F}
        \quad\Longrightarrow\quad
        F = O(W).
\]
\end{proof}

\subsection{Additional Details for \Cref{subsec:encoder_twohot}}\label{theory_sec_3.1}
\subsubsection{A Na\"ive Construction}\label{subsec:naive_const}
We briefly describe a na\"ive construction, which we compare to ours in \Cref{tab:theory}. Let $\mathbf K=\{\mathbf k_i\}_{i=1}^{|\mathbf K|}\subset\mathbb R^{d}$ and stack input embeddings as columns $\tilde{\mathbf K}=[\,\mathbf k_1\ \cdots\ \mathbf k_{|\mathbf K|}\,]\in\mathbb R^{d\times |\mathbf K|}$. Consider
\[
g(\vx)\;=\;\rmV\,\mathrm{ReLU}(\tilde{\mathbf K}^\top \vx - \vb),\qquad \rmV\in\mathbb R^{d\times |\mathbf K|},\ \ \vb\in\mathbb R^{|\mathbf K|}.
\]
For each $j$, define $\alpha_j:=\langle \mathbf k_j,\mathbf k_j\rangle$ and $\beta_j:=\max_{i\neq j}\langle \mathbf k_j,\mathbf k_i\rangle$, and assume $\alpha_j>\beta_j$. Choose any $b_j\in(\beta_j,\alpha_j)$ and set $a_i:=\alpha_i-b_i>0$. Then
\[
\mathrm{ReLU}(\tilde{\mathbf K}^\top \mathbf k_i - \vb)\;=\;a_i\,\mathbf e_i,
\]
so taking
\[
\rmV\;=\;\big[\,\mathbf v_{f(1)}/a_1\ \ \mathbf v_{f(2)}/a_2\ \ \cdots\ \ \mathbf v_{f(H)}/a_H\,\big]
\]
gives exact retrieval $g(\mathbf k_i)=\mathbf v_{f(i)}$. However, the hidden size is $|\mathbf K|$, and the parameter count is $\Theta(d |\mathbf K|)$ which is much too large.

\subsubsection{Two-hot Construction}\label{con:two_hot_unit_interval}
\begin{construction}[Encoder Construction, Two-Hot]
Fix a dimension \(d\ge 2\) and let \(\{\ve_1,\dots,\ve_d\}\subset\R^d\) be the standard basis.
Define the key set
\[
\mathbf{K}\;:=\;\bigl\{\vk_{i,j}\;=\;\ve_i-\ve_j:\ i\ne j,\ i,j\in[d]\bigr\},\qquad |\mathbf{K}|=d(d-1).
\]
Let \(h:\{(i,j)\,|\,i\neq j,\ i,j\in[d]\}\to[0,1]\) prescribe a target scalar for each key \(\vk_{i,j}\).
Define the (one-hidden-layer) encoder \(\enc:\R^d\to\R\) by
\[
\enc(\vx)\;=\;\mathbf{1}^\top\,\mathrm{ReLU}\bigl(\rmA\,\vx-\mathbf{1}\bigr),
\]
where \(\mathbf{1}\in\R^d\) is the all-ones vector, \(\mathrm{ReLU}\) acts elementwise, and the weight matrix \(\rmA\in\R^{d\times d}\) is
\[
\rmA[p,q]\;=\;
\begin{cases}
1 & \text{if } p=q,\\[4pt]
-\,h(p,q) & \text{if } p\neq q.
\end{cases}
\]
Then, for every \(i\neq j\in[d]\),
\[
\enc(\vk_{i,j})\;=\;h(i,j).
\]

\begin{proof}
Fix \(i\neq j\) and consider \(\vk_{i,j}=\ve_i-\ve_j\).
For each coordinate \(p\in[d]\),
\begin{align*}
    \bigl(\rmA\,\vk_{i,j}-\mathbf{1}\bigr)[p]
\;&=\;\rmA[p,i]-\rmA[p,j]-1 \\ 
\;&=\;
\begin{cases}
1-(-h(i,j))-1 \;=\; h(i,j), & p=i,\\[4pt]
(-h(j,i))-1-1 \;=\; -\,h(j,i)-2, & p=j,\\[4pt]
(-h(p,i))-(-h(p,j))-1 \;=\; h(p,j)-h(p,i)-1, & p\notin\{i,j\}.
\end{cases}
\end{align*}
Since \(h(\cdot,\cdot)\in[0,1]\), we have:
(i) the \(i\)-th coordinate equals \(h(i,j)\ge 0\);
(ii) the \(j\)-th coordinate is \(\le -2\) and thus strictly negative; and
(iii) for \(p\notin\{i,j\}\), \(h(p,j)-h(p,i)-1\le 1-0-1=0\), hence these coordinates are nonpositive.
Applying \(\mathrm{ReLU}\) elementwise zeroes out all nonpositive coordinates and preserves the \(i\)-th coordinate, yielding
\[
\mathrm{ReLU}\bigl(\rmA\,\vk_{i,j}-\mathbf{1}\bigr)[p]
\;=\;
\begin{cases}
h(i,j), & p=i,\\
0, & p\neq i.
\end{cases}
\]
Finally, summing with \(\mathbf{1}^\top\) gives \(\enc(\vk_{i,j})=\mathbf{1}^\top \mathrm{ReLU}(\rmA\,\vk_{i,j}-\mathbf{1})=h(i,j)\), as claimed.
\end{proof}
\end{construction}

\paragraph{Remark}
In the above proof, we say that $h$ outputs values in $[0,1]$ without loss of generality. Because the domain of $h$ is finite, let
\(a:=\min_{i\neq j} h(i,j)\) and \(b:=\max_{i\neq j} h(i,j)\).
Set \(\Delta:=b-a\) (take \(\Delta=1\) if \(a=b\)) and define the normalized function
\[
\tilde h(i,j)\;=\;\frac{h(i,j)-a}{\Delta}\in[0,1].
\]
Build the encoder above for \(\tilde h\), yielding
\(\widetilde{\enc}(\vk_{i,j})=\tilde h(i,j)\).
Recover \(h\) exactly with the 1D transform:
\[
\enc_h(\vx)\;=\;a+\Delta\cdot \widetilde{\enc}(\vx).
\]
This post-composition changes only \(O(1)\) top-layer parameters and does not affect the gating argument, so we may assume \(\operatorname{range}(h)\subset[0,1]\) without loss of generality.

\subsubsection{Discussion of Nichani et al.'s polylog factor}

Throughout the paper, we compare our construction with that given by \citet{nichani2024understandingfactualrecalltransformers}. Here, we discuss why the number of parameters of the \citet{nichani2024understandingfactualrecalltransformers} construction is at least $\Omega(|\mathbf K| \log^{12} |\mathbf V|)$. For comparability, we use some notation such as $m,d$ from \citet{nichani2024understandingfactualrecalltransformers}.

\citet{nichani2024understandingfactualrecalltransformers}'s result for a one-layer MLP with non-linear activation is presented in their Theorem 9 in Appendix B. Their theorem statement  is as follows for $\mathbf V, \mathbf W \in \mathbb R^{m \times d}$:

\noindent\textbf{Assumption 3.} $\sigma$ is a polynomial of degree $q$. Furthermore, if $\sigma(z)=\sum_{k=0}^{q} c_k h_k(z)$ is the Hermite decomposition of $\sigma$, then $c_k \ne 0$ for all $0 \le k \le q$.

\medskip

\noindent\textbf{Theorem 9~\citep{nichani2024understandingfactualrecalltransformers}.} Let $\epsilon \in (0,1)$ be a fixed constant. Assume that $d \ge N^{\epsilon}$ and $N \ge C_1(\epsilon)$, where $C_1(\epsilon)$ is a constant depending only on $\epsilon$. Assume that $q$ in Assumption 3 satisfies $q=\frac{C_2}{\epsilon}$ for some $C_2>2$. Then, if
\[
md \;\gtrsim\; N\bigl(C_3 \log(MN/\delta)\bigr)^{C_4/\epsilon},
\]
with probability $1-\delta$ over the draw of the embeddings, there exists $\mathbf V,\mathbf W$ such that
\[
\operatorname*{arg\,max}_{y\in[M]}\; \mathbf u_y^{\top}\mathbf V^{\top}\sigma(\mathbf W \mathbf e_x) \;=\; f^{*}(x)
\tag{20}
\]
for all $x \in [N]$.

Mapping their notation to ours, we have $N := |\mathbf K|$ and $M := 32$. In Theorem~9, they require $md \gtrsim\; N\,C_3^{\,q}\,\log^{\,4q+4}\!\Bigl(\tfrac{1}{\delta'}\Bigr)$ where $\delta'=\tfrac{\delta}{MN}$. This gives $\frac{md}{N}
\gtrsim\; C_3^{\,q}\,\bigl(\log(MN/\delta)\bigr)^{\,4q+4}$ and for $\delta=N^{-c}$ for a constant $c>0$,
\begin{equation*}
\log\!\Bigl(\tfrac{MN}{\delta}\Bigr)
=\Theta(\log N)
\quad\Longrightarrow\quad
\frac{md}{N} \;\gtrsim\; C_3^{\,q}\,(\log N)^{\,4q+4}.
\end{equation*}
Using their dimensional regime $d\ge N^{\varepsilon}$ gives $\log N=\Theta(\log d)$. 
In addition, they assume that $|\mathbf K| = |\mathbf V|$, so
\begin{equation*}
\log N \;\asymp\; \log|\mathbf V|
\quad\Longrightarrow\quad
md \;\gtrsim\; C_3^{\,q}\, |\mathbf K|(\log|\mathbf V|)^{\,4q+4}.
\end{equation*}
Since $q=\tfrac{C_2}{\varepsilon}>2$ implies $4q+4\ge 12$, we have
\begin{equation*}
\#\text{Parameters}\simeq\;md\;\gtrsim\; |\mathbf K|\,\log^{12}\!|\mathbf V|.
\end{equation*}

\newpage

\subsection{Additional Details for \Cref{subsec:encoder_twohot}}\label{sec:new_encoding_results}
This section is divided into three parts:

\begin{enumerate}
    \item In \Cref{subapp:enc_overview}, we provide an overview of our encoder architecture, desiderata, and more. We describe how we break the encoder into gated or non-gated encoder gadgets, each of which output one component of the final result.
    \item In \Cref{subapp:gated_enc_theory}, we describe the gated encoder gadget in more detail and prove that it works for asymptotically optimal parameter counts.
    \item In \Cref{subapp:non-gated_enc_theory}, we describe the non-gated encoder gadget in more detail. We show how we can construct the non-gated encoder gadget using the gated encoder gadget algorithm, and we illustrate how, in the special case of a ReLU encoder, we obtain a generalization of the Baum network from \citet{bubeck2020networksizeweightssize}.
\end{enumerate}

\subsubsection{Overview of the Encoder}\label{subapp:enc_overview}

Our encoder is a single-hidden layer MLP mapping key embeddings to compressed output embeddings.

\paragraph{Encoder Structure}
Our encoder is a either a gated MLP
\begin{equation*}
    \enc(\mathbf{x}) \;=\; \mathbf{E} \left(\sigma(\mathbf{G}\mathbf{x} + \mathbf{b}_G) \odot (\mathbf{A}\mathbf{x} + \mathbf{b}_A)\right) + \mathbf{b}_E,
\end{equation*}
or a non-gated MLP \begin{equation*}
    \enc(\mathbf{x}) \;=\; \mathbf{E} \sigma(\mathbf{A}\mathbf{x} + \mathbf{b}_A) + \mathbf{b}_E
\end{equation*}
with $\mathbf{A}, \mathbf{G} \in \mathbb{R}^{h \times d}$, $\mathbf{E}\in \mathbb{R}^{m\times h}$, $\mathbf{b}_A,\mathbf{b}_G\in\mathbb{R}^{h}$, $\mathbf{b}_E\in\mathbb{R}^{m},$  $\mathbf{x}\in\mathbb{R}^d$, and $\sigma: \R\rightarrow\R$. 

Gated MLPs simplify our analysis and are now popular across frontier models~\citep{yang2025qwen3technicalreport, dubey2024llama}. In \Cref{subapp:non-gated_enc_theory}, we extend our arguments to non-gated encoders.

\paragraph{Encoder Framework Objective}
Given key embeddings $\mathbf{K}\in\mathbb{R}^{|\mathbf{K}|\times d}$, compressed output embeddings $\mathbf{C}\in\mathbb{R}^{|\mathbf{V}|\times m}$, and a mapping $f$, the objective of our encoder framework is to produce an MLP $\enc$ with a minimal number of parameters such that $\mathbf{enc}(\mathbf{k}_i) = \mathbf{c}_{f(i)}$ for all $i\in[|\mathbf{K}|].$

\paragraph{Construction}
Our constructed encoder builds $m$ encoder gated or non-gated gadgets, for each $j\in[m]$:
\begin{align*}
    \enc_j(\mathbf{x}) \;&=\; \mathbf{1}_{\tilde h}^\top \left[\sigma(\mathbf{G}^{(j)}\mathbf{x}+\vb_{G}^{(j)}) \odot (\mathbf{A}^{(j)}\mathbf{x}+\vb_{A}^{(j)})\right]+b_{E}^{(j)};
\end{align*} 
or alternatively,
 \begin{align*}
    \enc_j(\mathbf{x}) \;&=\; \mathbf{E}^{(j)} \sigma(\mathbf{A}^{(j)}\mathbf{x} + \vb_{A}^{(j)}) + b_{E}^{(j)}\\
    \\
    \text{with}\quad\quad\mathbf{G}^{(j)},\mathbf{A}^{(j)}&\in\mathbb{R}^{\tilde h\times d},\quad  \mE^{(j)}\in\mathbb{R}^{1\times \tilde h} ,\quad  \vb_G^{(j)},\vb_A^{(j)}\in\mathbb{R}^{\tilde h}, \quad b_{E}^{(j)}\in \mathbb{R}
\end{align*}
that map $\mathbf{k}_i$ to $\mathbf c_{f(i)}[j] \in \mathbb{R}$, respectively, where $\tilde h = h/m$. We can set the down projection to $\mathbf{1}^\top$ in the gated encoder gadget without loss of generality by replacing $\mathbf{A}^{(j)}$ with $\textrm{diag}(\mathbf{E}^{(j)})\mathbf{A}^{(j)}$ and $\mathbf{b}_A^{(j)}$ with $\textrm{diag}(\mathbf{E}^{(j)})\mathbf{b}_A^{(j)}$. We will apply a similar technique in the case of the non-gated encoder gadget, but it is more involved.

We will demonstrate that these gadgets require only $O(|\mathbf{K}|)$ parameters. By stacking all $m$ gadgets together, one for each target dimension $j$, we can construct $\mathbf c_{f(i)}$ with a total of $O(m|\mathbf{K}|)$ parameters, as shown in \Cref{alg:encoder_construction_app}.

We will describe the gated and non-gated encoder gadgets in Appendix~\ref{subapp:gated_enc_theory} and \ref{subapp:non-gated_enc_theory}, respectively. We will drop the $j$ indexing everywhere for notational simplicity.

\begin{algorithm}[t]
\caption{Encoder Construction (\textsc{Encoder})}
\label{alg:encoder_construction_app}
\begin{algorithmic}[1]
\REQUIRE Key embeddings $\mathbf{K} \in \mathbb{R}^{|\mathbf{K}|\times d}$, Compressed output embeddings $\mathbf{C} \in \mathbb{R}^{|\mathbf{V}|\times m}$, Fact-mapping $f : [|\mathbf{K}|] \to [|\mathbf{V}|]$
\REQUIRE Hidden size $h$, activation $\sigma$, gated MLP flag $\textsc{gated}$, bias flag $\textsc{bias}$, tolerance $\delta$
\STATE $\tilde h \coloneqq h/m$
\vspace{7pt}
\STATE \textbf{for} $j = 1$ \textbf{to} $m$ \textbf{do}
\STATE \hspace{1em} $\mathbf{o}^{(j)} \coloneqq [\mathbf{C}_{f(1), j},\ldots,\mathbf{C}_{f(|\mathbf{K}|), j}]  \in \mathbb{R}^{|\mathbf{K}|}$
\STATE \hspace{1em} \textbf{if} \textsc{gated}:
\STATE \hspace{2em} $\enc_j(\vx)\coloneqq \mE^{(j)} \left(\sigma(\mathbf{G}^{(j)}\mathbf{x} + \vb_G^{(j)}) \odot (\mathbf{A}^{(j)}\mathbf{x} + \vb_A^{(j)})\right) + b_E^{(j)} \gets \textsc{GatedEncoderGadget}(\mathbf{K}, \mathbf{o}^{(j)}, \tilde h, \sigma, \textsc{bias})$
\STATE \hspace{1em} \textbf{else}:
\STATE \hspace{2em} $\enc_j(\vx)\coloneqq \mE^{(j)} \sigma(\mathbf{A}^{(j)}\mathbf{x} + \vb_A^{(j)}) + b_E^{(j)}  \gets \textsc{EncoderGadget}(\mathbf{K}, \mathbf{o}^{(j)}, \tilde h, \sigma, \textsc{bias}, \delta)$
\STATE \textbf{end for}
\vspace{7pt}
\STATE Stack $\mathbf{A} \coloneqq
    \begin{bmatrix}
    \mathbf{A}^{(1)} \\
    \vdots \\
    \mathbf{A}^{(m)}
    \end{bmatrix} \in \mathbb{R}^{h\times d}$, $\mathbf{b}_A \coloneqq
    \begin{bmatrix}
    \mathbf{b}_A^{(1)} \\
    \vdots \\
    \mathbf{b}_A^{(m)}
    \end{bmatrix} \in \mathbb{R}^{h}$, and $\mathbf{b}_E \coloneqq
    \begin{bmatrix}
    b_E^{(1)} \\
    \vdots \\
    b_E^{(m)}
    \end{bmatrix} \in \mathbb{R}^{m}$
\STATE $\mathbf{E} \coloneqq \begin{bmatrix}
    \mE^{(1)} & \mathbf{0}_{1\times \tilde h} & \cdots & \mathbf{0}_{1\times \tilde h} \\
    \mathbf{0}_{1\times \tilde h} & \mE^{(2)} & \cdots & \mathbf{0}_{1\times \tilde h} \\
    \vdots & \vdots & \ddots & \vdots \\
    \mathbf{0}_{1\times \tilde h} & \mathbf{0}_{1\times \tilde h} & \cdots & \mE^{(m)}
\end{bmatrix}\in\mathbb{R}^{m\times h}$
\vspace{7pt}
\STATE \textbf{if} \textsc{gated}:
\STATE \hspace{1em} Stack $\mathbf{G} \coloneqq
    \begin{bmatrix}
    \mathbf{G}^{(1)} \\
    \vdots \\
    \mathbf{G}^{(m)}
    \end{bmatrix} \in \mathbb{R}^{h\times d}$ and $\mathbf{b}_G \coloneqq
    \begin{bmatrix}
    \mathbf{b}_G^{(1)} \\
    \vdots \\
    \mathbf{b}_G^{(m)}
    \end{bmatrix} \in \mathbb{R}^{h}$
\vspace{7pt}
\STATE \textbf{if} \textsc{gated}:
\STATE \hspace{1em} $\enc(\vx)\coloneqq \mathbf{E} \left(\sigma(\mathbf{G}\mathbf{x} + \mathbf{b}_G) \odot (\mathbf{A}\mathbf{x} + \mathbf{b}_A)\right) + \mathbf{b}_E$
\STATE \textbf{else}
\STATE \hspace{1em} $\enc(\vx)\coloneqq \mathbf{E} \sigma(\mathbf{A}\mathbf{x} + \mathbf{b}_A) + \mathbf{b}_E$
\vspace{7pt}
\STATE \textbf{return} $\enc$
\end{algorithmic}
\end{algorithm}

\begin{algorithm}[t]
\caption{Gated Encoder Gadget Construction (\textsc{GatedEncoderGadget})}
\label{alg:gated_encoder_gadget_app}
\begin{algorithmic}[1]
\REQUIRE $\mathbf{o} \in \mathbb{R}^{|\mathbf{K}|}$, generic $\mathbf{K} \in \mathbb{R}^{|\mathbf{K}|\times d}$
\REQUIRE Hidden size $h$ with $dh \ge |\mathbf{K}|$, analytic $\sigma$, bias flag $\textsc{bias}$
\STATE Sample generic $\mathbf{G} \in \mathbb{R}^{h\times d}$ (e.g., i.i.d.\ Gaussian)
\STATE \textbf{if} \textsc{bias}:
\STATE \hspace{1em} Sample arbitrary $\mathbf{b}_G \in \mathbb{R}^{h}$ (e.g., all zeros)
\STATE \textbf{else}:
\STATE \hspace{1em} $\mathbf{b}_G \coloneqq \mathbf{0}_h \in \mathbb{R}^{h}$ (e.g., all zeros)
\vspace{7pt}
\STATE $\mathbf{\Sigma} \coloneqq \sigma(\mathbf{G}\mathbf{K}^\top + \mathbf{b}_G)\in \mathbb{R}^{h\times |\mathbf{K}|}$
\vspace{7pt}
\STATE \textbf{if} \textsc{bias}:
\STATE \hspace{1em} $\tilde d \coloneqq d+1$
\STATE \hspace{1em} $\mathbf{\tilde K} \coloneqq [\mathbf{K}, \mathbf{1}_{|\mathbf{K}|}]\in \mathbb{R}^{|\mathbf{K}|\times \tilde d}$
\STATE \textbf{else}:
\STATE \hspace{1em} $\tilde d \coloneqq d$
\STATE \hspace{1em} $\mathbf{\tilde K} \coloneqq \mathbf{K}\in \mathbb{R}^{|\mathbf{K}|\times \tilde d}$
\vspace{7pt}
\STATE $\mathbf{M}
\coloneqq
\big[
\operatorname{diag}(\mathbf{\Sigma}_{1})\mathbf{\tilde K}, \cdots, 
\operatorname{diag}(\mathbf{\Sigma}_{h})\mathbf{\tilde K}
\big]\in \mathbb{R}^{|\mathbf{K}|\times (dh)}$
\STATE \textbf{if} \textsc{bias}:
\STATE \hspace{1em} $D \coloneqq dh + 1$
\STATE \hspace{1em} $\mathbf{\widetilde M} \coloneqq [\mathbf{M}, \mathbf{1}_{|\mathbf{K}|}]\in \mathbb{R}^{|\mathbf{K}|\times \tilde D}$
\STATE \textbf{else}:
\STATE \hspace{1em} $\tilde D \coloneqq dh$
\STATE \hspace{1em} $\mathbf{\widetilde M} \coloneqq \mathbf{M}\in \mathbb{R}^{|\mathbf{K}|\times \tilde D}$
\vspace{7pt}

\STATE Solve for $\mathbf{v}\in \mathbb{R}^{dh}$ in $\mathbf{\widetilde M} \, \mathbf{v} = \mathbf{o}$
\vspace{7pt}
\STATE $\mathbf{A} \coloneqq \begin{bmatrix}
    \mathbf{v}[1:\tilde{d}-1]\\
    \mathbf{v}[\tilde d+1:2\tilde{d}-1]\\
    \vdots\\
    \mathbf{v}[(h-1)\tilde d+1:h\tilde{d}-1]
\end{bmatrix} \in \mathbb{R}^{h\times d}$
\STATE \textbf{if} \textsc{bias}:
\STATE \hspace{1em} $\mathbf{b}_A \coloneqq \begin{bmatrix}
    \mathbf{v}[\tilde d]\\
    \mathbf{v}[2\tilde d]\\
    \vdots\\
    \mathbf{v}[h\tilde d]
\end{bmatrix} \in \mathbb{R}^{h}$
\STATE \hspace{1em} $b_E \coloneqq \mathbf{v}[D] \in \mathbb{R}$
\STATE \textbf{else}:
\STATE \hspace{1em} $\mathbf{b}_A \coloneqq \mathbf{0}_h \in \mathbb{R}^{h}$
\STATE \hspace{1em} $b_E \coloneqq 0 \in \mathbb{R}$

\vspace{7pt}
\STATE $\enc(\vx)\coloneqq \mathbf{1}_h \left(\sigma(\mathbf{G}\mathbf{x} + \vb_G) \odot (\mathbf{A}\mathbf{x} + \vb_A)\right) + b_E$
\STATE \textbf{return} $\enc$
\end{algorithmic}
\end{algorithm}

\begin{algorithm}[t]
\caption{Encoder Gadget Construction (\textsc{EncoderGadget})}
\label{alg:encoder_gadget_app}
\begin{algorithmic}[1]
\REQUIRE $\mathbf{o} \in \mathbb{R}^{|\mathbf{K}|}$, generic $\mathbf{K} \in \mathbb{R}^{|\mathbf{K}|\times d}$
\REQUIRE Hidden size $h$ with $dh \ge |\mathbf{K}|$, analytic $\sigma$, bias flag $\textsc{bias}$, tolerance $\delta$
\STATE $\enc(\vx)\coloneqq \mathbf{1}_{1\times h/2} \left(\frac{d\sigma}{d x}(\mathbf{G}\mathbf{x} + \vb_G) \odot (\mathbf{A}\mathbf{x} + \vb_A)\right) + b_E \gets \textsc{GatedEncoderGadget}(\mathbf{K}, \mathbf{o}, h/2, \frac{d\sigma}{d x}, \textsc{bias})$
\vspace{7pt}
\STATE \textbf{for} $i=1$ \textbf{to} $[|\mathbf{K}|]$ \textbf{do}
\STATE \hspace{1em} $S_i\coloneqq \left\{\epsilon\, \Bigg|\, \left|[\epsilon^{-1}/2, -\epsilon^{-1}/2]\sigma\left(\begin{bmatrix}
    \mathbf{G} + \text{diag}(\epsilon) \mathbf{A}\\
    \mathbf{G} - \text{diag}(\epsilon) \mathbf{A}\\
\end{bmatrix}\mathbf{k}_i + \begin{bmatrix}
    \mathbf{b}_G + \epsilon \odot \mathbf{b}_A\\
    \mathbf{b}_G - \epsilon \odot \mathbf{b}_A\\
\end{bmatrix}\right) - \enc(\mathbf{k}_i)\right| \le \delta \right\}$
\STATE \textbf{end for}
\STATE Pick any $\epsilon \in \bigcap_{i=1}^{|\mathbf{K}|} S_i$
\vspace{7pt}
\STATE $\mathbf{A}\coloneqq \begin{bmatrix}
    \mathbf{G} + \text{diag}(\epsilon) \mathbf{A}\\
    \mathbf{G} - \text{diag}(\epsilon) \mathbf{A}\\
\end{bmatrix}\in\mathbb{R}^{h\times d}$
\STATE $\mathbf{b}_A\coloneqq \begin{bmatrix}
    \mathbf{b}_G + \epsilon \odot \mathbf{b}_A\\
    \mathbf{b}_G - \epsilon \odot \mathbf{b}_A\\
\end{bmatrix}\in\mathbb{R}^{h}$
\STATE $\mathbf{E}\coloneqq [\epsilon^{-1}/2, -\epsilon^{-1}/2]\in\mathbb{R}^{1\times h}$
\vspace{7pt}
\STATE $\enc(\vx)\coloneqq \mathbf{E}\sigma(\mathbf{A}\mathbf{x} + \vb_A) + b_E$
\STATE \textbf{return} $\enc$
\end{algorithmic}
\end{algorithm}

\subsubsection{Gated Encoder Theory}\label{subapp:gated_enc_theory}

Our gated encoder gadget will follow two simple steps: 1) pick $\mathbf{G}$, and 2) solve the resulting linear system for $\mathbf{A}$. The rest of this section will be dedicated to defining the linear system for $\mathbf{A}$ and providing conditions for a solution to exist.

Define\begin{align*}
    \mathbf{\Sigma} &= \sigma(\mathbf{G}\mathbf{K}^\top + \vb_G \mathbf{1}_{|\rmK|}^{\top})\in\mathbb{R}^{h\times|\mathbf{K}|}\\
    \mathbf{o} &= [\mathbf c_{f(1)}[j],\ldots, \mathbf c_{f(|\mathbf{K}|)}[j]]^\top
\end{align*}
where $\vb_G=\bm{0}$ if $\enc$ has no biases.

If $\enc$ has no biases, further define 
\begin{align*}
    \mathbf{M}(\mathbf{\Sigma}, \mathbf{K}) &= [\textrm{\upshape diag}(\mathbf{\Sigma}_1)\mathbf{K},\; \ldots, \;\textrm{\upshape diag}(\mathbf{\Sigma}_{h})\mathbf{K}] \in \mathbb{R}^{|\mathbf{K}|\times dh}\\
    \textrm{vec}(\mathbf{A}) &= [\mathbf{a}_1, \ldots, \mathbf{a}_h]^\top \in\mathbb{R}^{dh}.
\end{align*}
The $\mathbf{A}$ matrices such that $\enc(\mathbf{k}_i) = \mathbf c_{f(i)}[j]$ for all $i\in|\mathbf{K}|$ are exactly the solutions to the linear system \[\mathbf{M}(\mathbf{\Sigma}, \mathbf{K})\,\textrm{vec}(\mathbf{A}) = \mathbf{o}.\]
The above holds since once $\mathbf{\Sigma}$ entries are fixed, the encoder output is linear in the entries of $\mathbf{A}$.

If instead $\enc$ does have biases, define 
\begin{align*}
    \tilde d &= d+1\\
    D &= h\tilde d + 1\\
    \mathbf{\tilde K} &= [\mathbf{K}, \mathbf{1}_{|\mathbf{K}|}]\in \mathbb{R}^{|\mathbf{K}|\times \tilde d}\\
    \mathbf{\widetilde M}(\mathbf{\Sigma}, \mathbf{K}) &= [\textrm{\upshape diag}(\mathbf{\Sigma}_1)\mathbf{\tilde K},\; \ldots, \;\textrm{\upshape diag}(\mathbf{\Sigma}_{h})\mathbf{\tilde K}, \mathbf{1}_{|\mathbf{K}|}] \in \mathbb{R}^{|\mathbf{K}|\times D}\\
    \textrm{vec}(\mathbf{A}, \mathbf{b}_A, b_E) &= [\mathbf{a}_1, \mathbf{b}_A[1], \ldots, \mathbf{a}_h, \mathbf{b}_A[h], b_E]^\top \in\mathbb{R}^{D}.
\end{align*}
The $\mathbf{A}$, $\mathbf{b}_A$, and $b_E$ such that $\enc(\mathbf{k}_i) = \mathbf c_{f(i)}[j]$ for all $i\in|\mathbf{K}|$ are exactly the solutions to the linear system \[\mathbf{\widetilde M}(\mathbf{\Sigma}, \mathbf{K})\,\textrm{vec}(\mathbf{A}, \mathbf{b}_A, b_E) = \mathbf{o}.\]

To obtain a construction, it is sufficient to choose $\mathbf{\Sigma}$ such that the system is solvable for every choice of $\mathbf{o}$, which is true if and only if $\mathbf{M}(\mathbf{\Sigma}, \mathbf{K})$ or $\mathbf{\tilde M}(\mathbf{\Sigma}, \mathbf{K})$ has full row-rank. Since $\mathbf{\tilde M}(\mathbf{\Sigma}, \mathbf{K})$ always has full row rank if $\mathbf{M}(\mathbf{\Sigma}, \mathbf{K})$ does (because $\mathbf{\tilde M}(\mathbf{\Sigma}, \mathbf{K})$ is a submatrix of $\mathbf{\tilde M}(\mathbf{\Sigma}, \mathbf{K})$ with the same number of rows), we focus below on proving $\mathbf{M}(\mathbf{\Sigma}, \mathbf{K})$ has full row rank. Tighter bounds can be obtained for the bias case by considering $\mathbf{\tilde M}(\mathbf{\Sigma}, \mathbf{K})$ directly, but they do not affect parameter-count asymptotics (or even constant multipliers).

\paragraph{Rank condition on $\mathbf{\Sigma}$}
Interestingly, the above is true for generic $\mathbf{K}$ provided a simple rank condition on $\mathbf{\Sigma}.$ We start with the following definitions.
\begin{definition}
    Given a set $S$, define a $d$-partition of $S$ as a tuple of sets $\mathcal{I} = (I_1,\ldots,I_d)$ with $I_1,\ldots,I_d\subseteq[|S|]$ satisfying $I_i \cap I_j = \emptyset$ for all $i \neq j\in[d]$. Define a complete $d$-partition of $S$ as a $d$ partition also satisfying $\bigcup_{i\in[d]} I_i = S.$
\end{definition}
\begin{definition}\label{def:katri_k}
    Let $I_1,\ldots,I_d$ be a $d$-partition of $[|\mathbf{K}|]$ and let $\mathbf{a}\in\mathbb{R}^{|\mathbf{K}|}.$ Define $\mathbf{K}(\mathbf{a}, I_1,\ldots,I_d) \in \mathbb{R}^{|\mathbf{K}| \times d}$ according to the rule
    \[
        \mathbf{K}(\mathbf{a}, I_1,\ldots,I_d)[i,j] = \mathbf{a}[i]\mathbbm{1}\{i\in I_j\}.
    \]
    We abbreviate $\mathbf{K}(I_1,\ldots,I_d)\equiv \mathbf{K}(\mathbf{1}_{|\mathbf{K}|}, I_1,\ldots,I_d)$.
\end{definition}
Next, we provide the following lemmas characterizing the rank of $\mathbf{M}(\mathbf{\Sigma}, \mathbf{K})$ and $\mathbf{\widetilde M}(\mathbf{\Sigma}, \mathbf{K})$.
\begin{lemma}\label{lem:katri_rao_rank_equality}
    Let $I_1,\ldots,I_d$ be a $d$-partition of $[|\mathbf{K}|]$, pick any $\mathbf{\Sigma}\in \mathbb{R}^{h\times |\mathbf{K}|}$, and pick any $\mathbf{a}\in\mathbb{R}^{|\mathbf{K}|}$ with $\mathbf{a}[i]\neq 0$ for all $i\in [|\mathbf{K}|]$. Then
    \[
    \textrm{\upshape rank}(\normDec(\mathbf{\Sigma}, \mathbf{K}(\alpha, I_1,\ldots,I_d))) = \sum_{j=1}^d \operatorname{rank}\!\big(\boldsymbol{\Sigma}[:, I_j]\big).
    \]
\end{lemma}
\begin{proof}
    We define $\mathbf{K}\coloneqq \mathbf{K}(\alpha, I_1,\ldots,I_d)$ for notational simplicity.

    The columns of $\normDec$ can be re-grouped to form $d$ blocks of size $|\mathbf{K}| \times h$. Let $\normDec_j$ be the $j$-th new block, $j \in [d]$. This block contains all columns from $\normDec$ that were constructed using $\mathbf{K}[:, j]$ and can be written as
    $\normDec_j = \textrm{\upshape diag}(\mathbf{K}[:, j])\mathbf{\Sigma}^\top.$
    
    The matrix $\textrm{\upshape diag}(\mathbf{K}[:,j])$ acts as a row-selector. It zeroes out all rows of $\mathbf{\Sigma}^\top$ except for those with indices in $I_j$. Thus, $\textrm{\upshape col}(\normDec_i) \perp \textrm{\upshape col}(\normDec_j)$ for all $i,j\in[d],$ so
    \[
        \dim\left(\textrm{\upshape col}(\normDec(\mathbf{\Sigma}, \mathbf{K}))\right) = \dim\left(\bigoplus_{j=1}^d \textrm{\upshape col}(\normDec_j)\right) = \sum_{j=1}^d \operatorname{rank}\!\big(\normDec_j\big).
    \]
    
    Furthermore, \begin{align*}
        \operatorname{rank}(\normDec_j) &= \operatorname{rank}(\textrm{\upshape diag}(\mathbf{K}[:, j])\mathbf{\Sigma}^\top)\\
        &= \operatorname{rank}(\textrm{\upshape diag}(\mathbf{K}[I_j, j])\mathbf{\Sigma}^\top[I_j,:])\\
        &= \operatorname{rank}(\mathbf{\Sigma}^\top[I_j,:])\\
        &= \operatorname{rank}(\mathbf{\Sigma}[:,I_j]).
    \end{align*}
    Thus
    \[
    \textrm{\upshape rank}(\normDec(\mathbf{\Sigma}, \mathbf{K})) = \sum_{j=1}^d \operatorname{rank}\!\big(\normDec_j\big) = \sum_{j=1}^d \operatorname{rank}\!\big(\boldsymbol{\Sigma}[:, I_j]\big),
    \]
    as desired.
\end{proof}

\begin{lemma}\label{lem:katri_rao_matroid_union}
    For generic $\mathbf{K}$, we have that
    \begin{align}
        \textrm{\upshape rank}(\mathbf{M}(\mathbf{\Sigma}, \mathbf{K})) &= \min_{S\subseteq [|\mathbf{K}|]}\Big[|\mathbf{K}| - |S| + d\cdot \textrm{\upshape rank}(\mathbf{\Sigma}[:, S])\Big]\equiv R(\mathbf{\Sigma}).\label{eq:rank_eq1}
    \end{align}
    More specifically, the set $\mathcal{K} = \{\mathbf{K}\, |\, \text{\upshape rank}(\mathbf{M}(\mathbf{\Sigma}, \mathbf{K})) = R(\mathbf{\Sigma})\}$ is a non-empty Zariski open set (i.e. its complement is an algebraic set) and hence has full measure.
\end{lemma}
\begin{proof}
    For the full proof, see \Cref{pf:katri_rao_matroid_union}. A sketch of the proof is as follows.

    We first show that $\mathcal{K}$ is a Zariski open set. We show this by demonstrating that the $\mathbf{K}$ contained in $\mathcal{K}$ are exactly those for which not all $R(\mathbf{\Sigma})$th order minors of $\mathbf{M}(\mathbf{\Sigma}, \mathbf{K})$ are 0.

    Thus, we simply need to show that $\mathcal{K}$ is non-empty. Fortunately, by noting that Equations \ref{eq:rank_eq1} matches the form of the the Matroid Union Theorem~\citep{oxley2011matroid}, we can use the Matroid Union Theorem to construct an explicit $\mathbf{K}$ contained in $\mathcal{K}$, thus completing the proof.
\end{proof}

\begin{lemma}\label{lem:rank_condition_app}
    The set $\mathcal{K} = \{\mathbf{K}\, |\, \text{\upshape rank}(\mathbf{M}(\mathbf{\Sigma}, \mathbf{K})) = |\mathbf{K}|\}$ is a non-empty Zariski open set (and hence has full measure) if and only if
    \begin{equation}\label{eq: rank_cond_1_app}
        d\cdot \textrm{\upshape rank}(\mathbf{\Sigma}[:, S]) \ge |S|\quad\quad \forall S \subseteq [|\mathbf{K}|].
    \end{equation}
\end{lemma}
\begin{proof}
    $(\implies)$ Follows immediately from Lemma~\ref{lem:katri_rao_matroid_union}.
    
    $(\impliedby)$ Conversely, suppose there exists a subset $S \subseteq [|\mathbf{K}|]$ such that
\[
d\,\operatorname{rank}(\mathbf{\Sigma}[:,S]) < |S|.
\]
Then
\[
R(\mathbf{\Sigma})
= \min_{T \subseteq [|\mathbf{K}|]} \bigl[\,|\mathbf{K}| - |T| + d\cdot \operatorname{rank}(\mathbf{\Sigma}[:,T])\,\bigr]
\le |\mathbf{K}| - |S| + d\cdot \operatorname{rank}(\mathbf{\Sigma}[:,S])
< |\mathbf{K}|.
\]
By Lemma~\ref{lem:katri_rao_matroid_union}, there exists a non-empty Zariski open set
$\mathcal{K}_0$ such that for all $\mathbf{K}\in\mathcal{K}_0$,
\[
\operatorname{rank}(\mathbf{M}(\mathbf{\Sigma},\mathbf{K})) = R(\mathbf{\Sigma}) < |\mathbf{K}|.
\]
Therefore the full-rank locus
\[
\mathcal{K}_{\mathrm{full}}
:= \{\mathbf{K} : \operatorname{rank}(\mathbf{M}(\mathbf{\Sigma},\mathbf{K})) = |\mathbf{K}|\}
\]
is contained in the complement of $\mathcal{K}_0$, which is a proper Zariski closed set.
Hence $\mathcal{K}_{\mathrm{full}}$ cannot be a non-empty Zariski open set.
\end{proof}

Further, for analytic $\sigma$, such a $\mathbf{\Sigma}$ is easy to find. To show this, we first start with the following standard lemmas (proofs given for completeness):
\begin{lemma}\label{lem:indep_func_to_full_rank_matrix}
Let $f_1,\dots,f_r$ be linearly independent real-valued functions on some set $S$. Then there exist
points $\avec^{(1)},\dots,\avec^{(r)}\in S$ such that the $r\times r$ matrix $\mathbf{M} = \bigl(f_i(\avec^{(j)})\bigr)_{1\le i,j\le r}$
has rank $r$ (equivalently, is invertible).
\end{lemma}
\begin{proof}
    See \Cref{pf:indep_func_to_full_rank_matrix}.
\end{proof}

\begin{lemma}\label{lem:rank_determines_independence_sigma}
    Let $\sigma$ be a non-polynomial analytic function and define $f_\lambda(t) = \sigma(\lambda t)$. Further, define $\mathcal{S} = \operatorname{span}\{ f_\lambda | \lambda \in \mathbb{R} \}.$ The dimension of $\mathcal{S}$ is infinite.
\end{lemma}
\begin{proof}
    See \Cref{pf:rank_determines_independence_sigma}.
\end{proof}
\begin{lemma}\label{lem:d-rank_gives_rank_partial}
    Given a non-polynomial analytic function $\sigma:\mathbb{R}\to\mathbb{R},$ for generic $\mathbf{x}\in\mathbb{R}^{d_1}$ and $\mathbf{y}\in\mathbb{R}^{d_2}$, we have that
    \begin{align}
        \text{\upshape rank}(\sigma(\mathbf{x}\mathbf{y}^\top)) &= \min\{d_1, d_2\}.
    \end{align}
    More specifically, the set
    \begin{align*}
        \mathcal{S} &= \left\{(\mathbf{x}, \mathbf{y})\, \Big|\, \text{\upshape rank}(\sigma(\mathbf{x}\mathbf{y}^\top)) = \min\{d_1, d_2\}\right\}
    \end{align*}
    is the complement of a proper analytic subvariety of $\mathbb{R}^{d_1}\times \mathbb{R}^{d_2}$.
\end{lemma}
\begin{proof}
    We first show that the set $\mathcal{S}$ is the complement of an algebraic subvariety in a similar approach to the proof of \Cref{lem:katri_rao_matroid_union}. Thus, all that remains is to show that $\mathcal{S}$ is non-empty.

    \textbf{Case 1, $d_1\ge d_2$:}
    By \Cref{lem:rank_determines_independence_sigma} there exists a choice of $\mathbf{x}\in\mathbb{R}^{d_1}$ such that $\{\sigma(\mathbf{x}[i]\cdot y)\}_{i=1}^{d_1}$ are independent functions of $y$. Thus, by \Cref{lem:indep_func_to_full_rank_matrix}, we can choose $\mathbf{y}\in\mathbb{R}^{d_2}$ such that the matrix $\sigma(\mathbf{x}\mathbf{y}^\top)$ has rank $\min\{d_1, d_2\}$.

    \textbf{Case 2, $d_1< d_2$:}
    By \Cref{lem:rank_determines_independence_sigma} there exists a choice of $\mathbf{y}\in\mathbb{R}^{d_2}$ such that $\{\sigma(x\cdot \mathbf{y}[i])\}_{i=1}^{d_2}$ are independent functions of $x$. Thus, by \Cref{lem:indep_func_to_full_rank_matrix}, we can choose $\mathbf{x}\in\mathbb{R}^{d_1}$ such that the matrix $\sigma(\mathbf{x}\mathbf{y}^\top)$ has rank $\min\{d_1, d_2\}$.
    
    This demonstrates that $\mathcal{S}$ is nonempty, completing the proof.
\end{proof}
The above lemma can be naturally generalized:
\begin{lemma}\label{lem:d-rank_gives_rank}
    Given a non-polynomial analytic function $\sigma:\mathbb{R}\to\mathbb{R},$ for generic $\mathbf{x}\in\mathbb{R}^{d_1}$ and $\mathbf{y}\in\mathbb{R}^{d_2}$ we have that
    \begin{align}
        \text{\upshape rank}(\sigma(\mathbf{x}\mathbf{y}^\top)[S_1, S_2]) &= \min\{|S_1|, |S_2|\}\quad\quad \forall S_1\subseteq [d_1],\, S_2\subseteq [d_2].
    \end{align}
    More specifically, the set
    \begin{align*}
        \mathcal{S} &= \left\{(\mathbf{x}, \mathbf{y})\, \Big|\, \text{\upshape rank}(\sigma(\mathbf{x}\mathbf{y}^\top)[S_1, S_2]) = \min\{|S_1|, |S_2|\}\quad\quad \forall S_1\subseteq [d_1],\, S_2\subseteq [d_2]\right\}
    \end{align*}
    is the complement of a proper analytic subvariety of $\mathbb{R}^{d_1}\times \mathbb{R}^{d_2}$.
\end{lemma}
\begin{proof}
    See \Cref{pf:d-rank_gives_rank}.
\end{proof}

Finally, we combine \Cref{lem:katri_rao_rank_equality} and \Cref{lem:d-rank_gives_rank} to obtain the following characterization for when $\mathbf{M}$ has full row rank.
\begin{lemma}[Full-row-rank condition for non-polynomial analytic activations]\label{lem:full_M_rank_result_analytic}
    Let $\sigma: \mathbb{R}\to \mathbb{R}$ be a non-polynomial analytic function. If $dh \ge |\mathbf{K}|$, then for generic 
    $\mathbf{K}\in\mathbb{R}^{|\mathbf{K}|\times d}$ and $\mathbf{G}\in\mathbb{R}^{h\times d}$, the matrix
    \[
        \mathbf{M}\big(\sigma(\mathbf{G}\mathbf{K}^\top),\mathbf{K}\big) \;\in\; \mathbb{R}^{|\mathbf{K}|\times (dh)}
    \]
    has full row rank $|\mathbf{K}|$.
    The tuples for which full row rank fails form a proper analytic subvariety of the
    ambient parameter space.
\end{lemma}
\begin{proof}
    A more careful combination of the proofs of \Cref{lem:katri_rao_matroid_union,lem:rank_condition_app,lem:d-rank_gives_rank}. Full proof given in \Cref{pf:full_M_rank_result_analytic}.
\end{proof}

\Cref{lem:full_M_rank_result_analytic} is the last piece we need to prove the full encoder gadget theorem:
\begin{theorem}\label{thm:gated_encoder_gadget_works}
    Let $\sigma: \mathbb{R}\to \mathbb{R}$ be a non-polynomial analytic activation. If $dh \ge |\mathbf{K}|$ and $\text{\upshape rank}[\sigma]\ge h$, then following \Cref{alg:gated_encoder_gadget_app} with $\textsc{bias}$ either $\textsc{True}$ or $\textsc{False}$ produces an MLP $\mathbf{enc}(\vx)\coloneqq \mathbf{1}_h \left(\sigma(\mathbf{G}\mathbf{x}) \odot (\mathbf{A}\mathbf{x})\right)$ which satisfies $\mathbf{enc}(\vk_i) = o_i$ for all $i\in[|\mathbf{K}|].$
\end{theorem}
\begin{proof}
By Lemma~\ref{lem:full_M_rank_result_analytic}, under the stated conditions (no-bias or biased case) and for generic draws of $\mathbf{G}$ (setting $\mathbf{b}_G = \mathbf{0}_{h}$), the corresponding matrix $\mathbf{M}(\mathbf{\Sigma},\mathbf{K})$ or $\mathbf{\widetilde M}(\mathbf{\Sigma},\mathbf{K})$ have full row rank. Hence, for any target vector $\mathbf{o}$, the linear system in $\textrm{vec}(\mathbf{A})$ (or $\textrm{vec}(\mathbf{A},\mathbf{b}_A,b_E)$) is solvable, and the parameters returned by \Cref{alg:gated_encoder_gadget_app} satisfy $\enc(\vk_i)=o_i$ for all $i\in[|\mathbf{K}|]$.
\end{proof}

\subsubsection{Non-Gated Encoders Reduce to Gated Encoders} \label{subapp:non-gated_enc_theory}

In Appendix \ref{sec:new_encoding_results}, it is shown that these results extend to non-gated MLPs (up to an arbitrarily small $\delta$ error) by implementing a neural tangent kernel (NTK) approximation similar to \citet{nichani2024understandingfactualrecalltransformers}. Interestingly, when this generalization is applied to ReLU MLPs, a construction is obtained which generalizes that from \citet{bubeck2020networksizeweightssize} while utilizing up to 4$\times$ fewer parameters\footnote{In fact, this generalization of \citet{bubeck2020networksizeweightssize} matches the degrees-of-freedom-based parameter count lower bound up to lower order terms.}. Additionally, while it is possible to use the encoder construction from \citet{bubeck2020networksizeweightssize} directly in the full fact-storing construction, we found that the resulting MLPs are not usable by transformers, whereas the MLPs constructed herein are.

The construction, detailed in \Cref{alg:encoder_gadget_app}, approximates a gated MLP that uses the activation's derivative, $\sigma'$, with a standard non-gated MLP that uses $\sigma$. This is achieved in three steps:

\begin{enumerate}
    \item \textbf{Construct a ``Derivative" Gadget:} First, \Cref{alg:encoder_gadget_app} (Line 1) calls \Cref{alg:gated_encoder_gadget_app} to find the parameters of an intermediate gated gadget. This call uses a hidden size of $h/2$ (where $h$ is the hidden size required by \Cref{alg:encoder_gadget_app}) and replaces the activation $\sigma$ with its derivative, $\frac{d\sigma}{dx}$. Let the parameters returned by this call be $(\mathbf{G}_{\text{deriv}}, \mathbf{b}_{G, \text{deriv}}, \mathbf{A}_{\text{deriv}}, \mathbf{b}_{A, \text{deriv}}, b_{E})$ where $\mathbf{G}_{\text{deriv}}, \mathbf{A}_{\text{deriv}} \in \mathbb{R}^{(h/2) \times d}$ and $\mathbf{b}_{G, \text{deriv}}, \mathbf{b}_{A, \text{deriv}} \in \mathbb{R}^{h/2}$. The resulting encoder (which \Cref{alg:encoder_gadget_app} temporarily calls $\enc(\vx)$ on Line 1) is:
    \begin{equation*}
        \enc_{\text{deriv}}(\mathbf{x}) = \mathbf{1}_{h/2}^\top \left(\sigma'(\mathbf{G}_{\text{deriv}}\mathbf{x} + \mathbf{b}_{G, \text{deriv}}) \odot (\mathbf{A}_{\text{deriv}}\mathbf{x} + \mathbf{b}_{A, \text{deriv}})\right) + b_{E}
    \end{equation*}
    This $\enc_{\text{deriv}}$ is constructed to map $\mathbf{k}_i$ to the target output $o_i$ for all $i \in [|\mathbf{K}|]$.

    \item \textbf{Find Approximation Parameter $\boldsymbol{\epsilon}$:} Second (Lines 3-6), the algorithm finds a small vector $\boldsymbol{\epsilon} \in \mathbb{R}^{h/2}$. This $\boldsymbol{\epsilon}$ is chosen such that a central difference approximation of $\enc_{\text{deriv}}$ (using $\sigma$) is within a tolerance $\delta$ of the target values $o_i \approx \enc_{\text{deriv}}(\mathbf{k}_i)$ for all keys $\mathbf{k}_i$.

    \item \textbf{Construct Final Non-Gated Gadget:} Finally (Lines 8-12), the algorithm uses the intermediate parameters and $\boldsymbol{\epsilon}$ to define the parameters of the \emph{final} non-gated MLP, which has the target hidden size $h = 2 \times (h/2)$. The parameters for the returned $\enc(\mathbf{x})$ are:
    \begin{align*}
        \mathbf{A} &\coloneqq \begin{bmatrix}
            \mathbf{G}_{\text{deriv}} + \operatorname{diag}(\boldsymbol{\epsilon}) \mathbf{A}_{\text{deriv}} \\
            \mathbf{G}_{\text{deriv}} - \operatorname{diag}(\boldsymbol{\epsilon}) \mathbf{A}_{\text{deriv}}
        \end{bmatrix} \in \mathbb{R}^{h\times d} \\
        \mathbf{b}_A &\coloneqq \begin{bmatrix}
            \mathbf{b}_{G, \text{deriv}} + \boldsymbol{\epsilon} \odot \mathbf{b}_{A, \text{deriv}} \\
            \mathbf{b}_{G, \text{deriv}} - \boldsymbol{\epsilon} \odot \mathbf{b}_{A, \text{deriv}}
        \end{bmatrix} \in \mathbb{R}^{h} \\
        \mathbf{E} &\coloneqq \begin{bmatrix}
            \frac{1}{2}\boldsymbol{\epsilon}^{-1} & -\frac{1}{2}\boldsymbol{\epsilon}^{-1}
        \end{bmatrix} \in \mathbb{R}^{1\times h}
    \end{align*}
    The final returned encoder is $\enc(\vx)\coloneqq \mathbf{E}\sigma(\mathbf{A}\mathbf{x} + \mathbf{b}_A) + b_E$, which by construction approximates the target outputs $\mathbf{o}$.
\end{enumerate}

Intuitively, the final non-gated gadget implements a finite-difference approximation of the ``derivative'' gadget. Plugging in the definitions of $\mathbf{A},\mathbf{b}_A,\mathbf{E}$, we obtain for any $\vx$:
\[
\enc(\vx)
= \sum_{r=1}^{h/2} \frac{1}{2\epsilon_r}
\Bigl[\sigma\bigl(g_r(\vx) + \epsilon_r a_r(\vx)\bigr)
      - \sigma\bigl(g_r(\vx) - \epsilon_r a_r(\vx)\bigr)\Bigr] + b_E,
\]
where $g_r(\vx)$ and $a_r(\vx)$ are the $r$-th coordinates of $\mathbf{G}_{\text{deriv}}\vx + \vb_{G,\text{deriv}}$ and $\mathbf{A}_{\text{deriv}}\vx + \vb_{A,\text{deriv}}$, respectively. By Taylor expansion (or the mean value theorem), each bracket implements
\[
\frac{\sigma(g_r+\epsilon_r a_r)-\sigma(g_r-\epsilon_r a_r)}{2\epsilon_r}
\approx \sigma'(g_r)\,a_r,
\]
so $\enc(\vx)$ approximates
\[
\enc_{\text{deriv}}(\vx)
= \sum_{r=1}^{h/2} \sigma'(g_r(\vx))\,a_r(\vx) + b_E.
\]
By construction of $\epsilon\in\bigcap_i S_i$, this approximation error is at most $\delta$ on all keys $\vk_i$, so the returned non-gated encoder matches the desired targets up to tolerance $\delta$.

\paragraph{Special Case: ReLU Activation}
Here, we show the generality of our framework by showing that \citep{bubeck2020networksizeweightssize} is a special case.
In the special case where the activation function is the ReLU function, the derivative $\sigma'(\mathbf{x}) = \mathbf{1}_{\{\mathbf{x} > 0\}}$ is used to construct the intermediate gadget. The final encoder returned by \Cref{alg:encoder_gadget_app} (Line 12) implements the central difference approximation:
\begin{align*}
    \enc(\mathbf{x}) = \begin{bmatrix} \frac{1}{2}\boldsymbol{\epsilon}^{-1} & -\frac{1}{2}\boldsymbol{\epsilon}^{-1} \end{bmatrix} \operatorname{ReLU}\left( \begin{bmatrix}
        \mathbf{G}_{\text{deriv}} + \operatorname{diag}(\boldsymbol{\epsilon}) \mathbf{A}_{\text{deriv}} \\
        \mathbf{G}_{\text{deriv}} - \operatorname{diag}(\boldsymbol{\epsilon}) \mathbf{A}_{\text{deriv}}
    \end{bmatrix} \mathbf{x} + \begin{bmatrix}
        \mathbf{b}_{G, \text{deriv}} + \boldsymbol{\epsilon} \odot \mathbf{b}_{A, \text{deriv}} \\
        \mathbf{b}_{G, \text{deriv}} - \boldsymbol{\epsilon} \odot \mathbf{b}_{A, \text{deriv}}
    \end{bmatrix} \right) + b_E.
\end{align*}
If a forward difference approximation were used instead (as in \citet{bubeck2020networksizeweightssize}), the form would be:
\begin{equation*}
    \text{MLP}(\mathbf{x}) = \mathbf{1}_{h/2}^\top \left(\operatorname{diag}(\boldsymbol{\epsilon})^{-1}\Big(\operatorname{ReLU}(\mathbf{G}_{\text{deriv}} \mathbf{x} + \mathbf{b}_{G, \text{deriv}} + \operatorname{diag}(\boldsymbol{\epsilon}) (\mathbf{A}_{\text{deriv}} \mathbf{x} + \mathbf{b}_{A, \text{deriv}})) - \operatorname{ReLU}(\mathbf{G}_{\text{deriv}} \mathbf{x} + \mathbf{b}_{G, \text{deriv}})\Big)\right) + b_E.
\end{equation*}
The portion inside the outer brackets is the derivative neuron from \citet{bubeck2020networksizeweightssize}.

Note that one can also pull the $\operatorname{diag}(\boldsymbol{\epsilon})^{-1}$ term inside the brackets and define $\boldsymbol{\lambda}$ such that $\boldsymbol{\epsilon} \odot \boldsymbol{\lambda} = \mathbf{1}$ (element-wise) to get a ``Lagrangian formulation":
\begin{equation*}
    \text{MLP}(\mathbf{x}) = \mathbf{1}_{h/2}^\top \Big(\operatorname{ReLU}(\operatorname{diag}(\boldsymbol{\lambda})(\mathbf{G}_{\text{deriv}} \mathbf{x} + \mathbf{b}_{G, \text{deriv}}) + (\mathbf{A}_{\text{deriv}} \mathbf{x} + \mathbf{b}_{A, \text{deriv}})) - \operatorname{diag}(\boldsymbol{\lambda}) \operatorname{ReLU}(\mathbf{G}_{\text{deriv}} \mathbf{x} + \mathbf{b}_{G, \text{deriv}})\Big) + b_E.
\end{equation*}
The ReLU case possesses the property that this forward difference approximation is exactly equal to the corresponding gated MLP on a set of points $\mathbf{x}_i$ as long as $\boldsymbol{\lambda} \ge -\min_{i} \frac{\mathbf{A}_{\text{deriv}} \mathbf{x}_i + \mathbf{b}_{A, \text{deriv}}}{\mathbf{G}_{\text{deriv}} \mathbf{x}_i + \mathbf{b}_{G, \text{deriv}}}$ (element-wise). In particular, if $\min_{i} \frac{\mathbf{A}_{\text{deriv}} \mathbf{x}_i + \mathbf{b}_{A, \text{deriv}}}{\mathbf{G}_{\text{deriv}} \mathbf{x}_i + \mathbf{b}_{G, \text{deriv}}} \ge 0$, then $\boldsymbol{\lambda} = \mathbf{0}$ can be set to achieve the exact result, which avoids extra neurons. In contrast, the \citet{bubeck2020networksizeweightssize} derivative neuron formulation would diverge in this case.

\subsection{Additional Details for \Cref{subsec:construction_decoder}}\label{sec:new_decoding_results}

We prove lower bounds on $\rho$ and detail our decoding construction. We use a slightly more practical definition of $\rho$ as follows when doing computations. However, since $\rho \geq \rho_{\min}$ by definition, similar statements hold for $\rho$. 

\begin{definition}\label{def:rho_R_tau}
    For vectors $\compressedOutEmbedding_1,\dots,\compressedOutEmbedding_{\numVectors}\in\R^{d}$ 
    and $\compressedAuxEmbedding_1, ..., \compressedAuxEmbedding_{\numVectors} \in \R^{d}$, 
    we define $\outCompEmbedMat = [\compressedOutEmbedding_1,\ldots, \compressedOutEmbedding_{\numVectors}]^\top \in \R^{{\numVectors}\times d}$ and $\targetDirectionMatrix = [\compressedAuxEmbedding_1, ..., \compressedAuxEmbedding_{\numVectors}]^\top \in \R^{{\numVectors} \times d}$. Let
    \begin{align*}
        \rho_{\min} (\outCompEmbedMat, \targetDirectionMatrix) &= \min_i \min_{j\neq i} \frac{\langle \compressedOutEmbedding_i-\compressedOutEmbedding_j,\compressedAuxEmbedding_i\rangle}{\norm{\compressedOutEmbedding_i - \compressedOutEmbedding_j} \norm{\compressedAuxEmbedding_i}}
    \end{align*}
    For ease of notation, we often write $\rho_{\min}:=\rho_{\min}(\outCompEmbedMat,\targetDirectionMatrix)$. Occasionally, we refer to the set $\{\compressedOutEmbedding_i\}_{i = 1}^{\numVectors}$ as our set of output embeddings, and the set $\{\compressedAuxEmbedding_i\}_{i = 1}^{\numVectors}$ as our set of auxiliary directions. 
\end{definition}

We now prove our full construction. In this case, we have that $\rho(\outCompEmbedMat)$ as defined in \Cref{subsec:construction_decoder} satisfies $\rho(\outCompEmbedMat) \geq \rho_{\min}(\outCompEmbedMat, \targetDirectionMatrix)$.

\begin{theorem}\label{thm:uniform_decoding_rho}
Assume $\compressedOutEmbedding_1,\ldots,\compressedOutEmbedding_{\numVectors} \stackrel{\text{i.i.d.}}{\sim}\mathrm{Unif}(\Sph)$ with $d\ge 2$ and for simplicity set\footnote{One may wonder why we can set $\compressedAuxEmbedding_i = \compressedOutEmbedding_i$ in this step. The reason we is that it simplifies the proof and shows existence of a lower bound on $\rho_{\min}$. However, there may be a better choice of $\compressedAuxEmbedding_i$ which yields a tighter bound.} $\compressedAuxEmbedding_i = \compressedOutEmbedding_i$ for all $i$. Then, with probability at least $1-\delta$,
\[
\qquad
\rho_{\min} \ \ge\ \sqrt{\frac{1-\sqrt{\tfrac{2}{d}\,\ln\!\tfrac{\binom{\numVectors}{2}}{\delta}}}{2}}\ .
\qquad
\]
\end{theorem}

\begin{proof}
    See \Cref{pf:NEW_main_decoding}
\end{proof}

\begin{theorem}\label{thm:NEW_main_decoding}
\label{thm:main_ex212}
Let $\gaussianMatrix \in \mathbb{R}^{m \times d}$ have i.i.d $\mathcal{N}(0,1)$ entries. Set $\rmM := \frac{1}{m} \gaussianMatrix^\top$ and, for each $i \in [\numVectors]$, define $\mathbf \codeMatrix[i]:=\gaussianMatrix\, \compressedAuxEmbedding_i\in\R^m$. Let $\rho_{\min}=\rho_{\min}(\mathbf V,\mathbf U)$ be as in \Cref{def:rho_R_tau}, and fix a failure probability $\delta \in (0,1)$. If 
\[
m\ \ge\ \frac{32}{\rho_{\min}^2}\ \ln\!\frac{4\numVectors(\numVectors-1)}{\delta},
\]
and $\rho_{\min} > 0$, then with probability at least $1 - \delta$ the following holds simultaneously for all $i \neq j$: 
\begin{align*}
    \langle \mathbf \compressedOutEmbedding_i,\rmM\mathbf \codeMatrix[i]\rangle-\langle \mathbf \compressedOutEmbedding_j,\rmM\mathbf \codeMatrix[i]\rangle
&\;\ge\; \frac{\rho_{\min}}{2} \norm{\compressedOutEmbedding_i - \compressedOutEmbedding_j}\norm{\compressedAuxEmbedding_i} > 0
\end{align*}
\end{theorem}

\begin{proof}
    See \Cref{pf:real_NEW_main_decoding}
\end{proof}

\begin{corollary}
    For $\delta = \frac{1}{\operatorname{poly} d}$, $\numVectors = \operatorname{poly}(d)$, large enough $d$, and for output embeddings $\{\compressedOutEmbedding_i\}_{i = 1}^{\numVectors}$ as in \Cref{thm:uniform_decoding_rho}, the set of output embeddings are softmax decodable with probability $1 - \delta$ as long as the conditions in \Cref{thm:NEW_main_decoding} on $m$ hold. 
\end{corollary}
\begin{proof}
    By \Cref{thm:uniform_decoding_rho}, $\rho_{\min} \geq \gamma$ for some $\gamma$ with $\gamma \rightarrow \frac{1}{\sqrt{2}}$ as $d \rightarrow \infty$. Hence, for all large enough $d$, there exists an absolute positive constant $\gamma^\star $ such that $\rho_{\min} \geq \gamma^\star$ with probability $1 - \delta$. Thus, we apply \Cref{lem:softmax_reformulation} and \Cref{thm:NEW_main_decoding} to decode the embeddings. 
\end{proof}

In the following theorem, we will need the sub-gaussian norm $\norm{\cdot}_{\psi_2}$: 
\begin{align*}
    \norm{\rmX}_{\psi_2} := \inf \{t > 0: \E[\exp(\rmX^2 / t^2)] \leq 2 \}
\end{align*}

\begin{theorem}\label{thm:subgaussian_decoding_bd}
Let $\compressedOutEmbedding_i=(\xi_{i1},\dots,\xi_{id})\in\R^{d}$ for $i=1,\dots,\numVectors$, where the coordinates are i.i.d. sub-gaussian with
\[
\mathbb{E}\left[\xi_{ik}\right]=0,\qquad \mathbb{E}\left[\xi_{ik}^2\right]=\frac{1}{d},\qquad \|\xi_{ik}\|_{\psitwo}\le \frac{K}{\sqrt d}.
\]
Set $\compressedAuxEmbedding_i:=\compressedOutEmbedding_i/\|\compressedOutEmbedding_i\|$ and let $c_B=\frac{1}{2(2e-1)}$. Then for every $\delta\in(0,1)$, with probability at least $1-\delta$,
\[
\quad
\rho_{\min}\ \ge\ \frac{1-\varepsilon_{\numVectors}-t_{\numVectors}}{2(1+\varepsilon_{\numVectors})}\ ,
\quad
\]
where

\begin{align*}
    \varepsilon_{\numVectors} &:=(K^2+\frac{1}{\ln 2})\, \max \left(\sqrt{\frac{1}{c_B\,d}\,\ln\!\frac{4\numVectors}{\delta}}\, ,  \frac{1}{c_B\,d}\,\ln\!\frac{4\numVectors}{\delta}\right)
\qquad \\ \qquad
t_{\numVectors}\ &:=\ K\,\sqrt{\frac{2\ln 2}{d}\,\ln\!\frac{4\numVectors(\numVectors-1)}{\delta}}.
\end{align*}

\end{theorem}
\begin{proof}
    See \Cref{pf:subgaussian_decoding_bd}
\end{proof}

\begin{corollary}
    For $\delta = \frac{1}{\operatorname{poly} d}$, $\numVectors = \operatorname{poly}(d)$, large enough $d$, and for output embeddings $\{\compressedOutEmbedding_i\}_{i = 1}^{\numVectors}$ as in \Cref{thm:subgaussian_decoding_bd}, the set of output embeddings are softmax decodable with probability $1 - \delta$ as long as the conditions in \Cref{thm:NEW_main_decoding} on $m$ hold. 
\end{corollary}
\begin{proof}
    By \Cref{thm:subgaussian_decoding_bd}, $\rho_{\min} \geq \gamma$ for some $\gamma$ with $\gamma \rightarrow 1/2$ as $d \rightarrow \infty$. Hence, for all large enough $d$, there exists an absolute positive constant $\gamma^\star $ such that $\rho_{\min} \geq \gamma^\star$ with probability $1 - \delta$. Thus, we apply \Cref{pf:NEW_main_decoding} to decode the embeddings. 
\end{proof}

\subsubsection{Relation of $\rho$ to Coherence}\label{sec:rho_relation_coherence}

Throughout this section, we define coherence in the traditional sense as follows.

\begin{definition}[Coherence]\label{def: coherence}
For unit–norm row vectors $\vSet=[\compressedOutEmbedding_1,\ldots, \compressedOutEmbedding_{\numVectors}]^\top \in \R^{{\numVectors}\times d}$,
\[
\mu(\vSet)\ :=\ \max_{i\neq j}\, \left|\langle \vv_i,\ \vv_j\rangle \right|.
\]
\end{definition}

Given the definition of $\rho(\mathbf V, \mathbf U)$, which doesn't have similar absolute values around the inner product term, we could have defined the coherence as $\mu(\vSet) = \max_{i \neq j} \langle \vv_i, \vv_j \rangle$. The results of this section hold using either definition of $\mu(\vSet)$.

\begin{lemma}[Lower bound via absolute coherence]
Let $\vSet=[\compressedOutEmbedding_1,\ldots, \compressedOutEmbedding_{\numVectors}]^\top \in \R^{{\numVectors}\times d}$ with $\|\vv_i\|_2=1$ for all $i$.
By $\cref{def: coherence}$, then
\[
\rho(\vSet)\ \ge\ \frac{1}{\sqrt{2}}\;\sqrt{\,1-\mu(\vSet)\,}.
\]
\end{lemma}

\begin{proof}
Fix $i$ and set $\vu_i:=\vv_i$.
For any $j\neq i$,
\[
\frac{\langle \vv_i-\vv_j,\;\vu_i\rangle}{\|\vv_i-\vv_j\|_2}
=
\frac{\langle \vv_i-\vv_j,\;\vv_i\rangle}{\|\vv_i-\vv_j\|_2}
=
\frac{1-\langle \vv_i,\vv_j\rangle}{\sqrt{\|\vv_i\|_2^2+\|\vv_j\|_2^2-2\langle \vv_i,\vv_j\rangle}}
=
\frac{1-\langle \vv_i,\vv_j\rangle}{\sqrt{2-2\langle \vv_i,\vv_j\rangle}}
=
\frac{1}{\sqrt{2}}\;\sqrt{\,1-\langle \vv_i,\vv_j\rangle\,}.
\]
Taking the minimum over $j\neq i$ and then over $i$ yields
\[
\rho(\vSet)
\;\ge\;
\frac{1}{\sqrt{2}}\;\min_{i\neq j}\sqrt{\,1-\langle \vv_i,\vv_j\rangle\,}.
\]
Since for every $i\neq j$ we have $\langle \vv_i,\vv_j\rangle\le |\langle \vv_i,\vv_j\rangle|\le \mu(\vSet)$ and
$a\mapsto\sqrt{1-a}$ is decreasing on $(-\infty,1]$, it follows that
\[
\min_{i\neq j}\sqrt{\,1-\langle \vv_i,\vv_j\rangle\,}
\ \ge\
\sqrt{\,1-\mu(\vSet)\,}.
\]
Therefore $\rho(\vSet)\ \ge\ \frac{1}{\sqrt{2}}\sqrt{\,1-\mu(\vSet)\,}$, as claimed.
\end{proof}

Given this lower bound on $\rho(\mathbf{U}, \mathbf{V})$ in terms of $1 - \mu(\vSet)$, one might wonder if there exists a similar upper bound. Specifically, does there exist some constant $\beta>0$ such that 
\begin{align*}
    \rho(\vSet) \leq O( (1 - \mu(\vSet))^\beta)
\end{align*}
In the following proposition, we provide a counter example which shows that this is false. Hence, $\rho(\vSet)$ and $1-\mu(\vSet)$ are fundamentally different quantities. 

\begin{lemma}
Fix a constant integer $p\ge 2$. Then, for large enough $d$, there exist
unit–norm row vectors $\vSet=[\compressedOutEmbedding_1,\ldots, \compressedOutEmbedding_{\numVectors}]^\top \in \R^{{\numVectors}\times d}$ such that
\[
\mu(\vSet)\ =\ 1-o(1)\qquad\text{but}\qquad \rho(\vSet)\ \ge\ \sqrt{\frac{1/p}{2}}\ >\ 0.
\]
\end{lemma}

\begin{proof}
Choose a dimension $d_0=o(d)$ and construct $\vSet_0=[\compressedOutEmbedding_1^{(0)},\ldots, \compressedOutEmbedding_{\numVectors}^{(0)}]^\top \in \R^{{\numVectors}\times d_0}$ as follows.
Choose each row $\vv_i^{(0)}$ to be the the $p-\text{hot}$ encoding of the row index. Thus each row has exactly $p$ non-zero entries, each equal to $1/\sqrt p$ and pairwise the non-zero entries overlap in at most $p-1$ coordinates. Then for $i\neq j$,
\[
\bigl|\langle \vv_i^{(0)},\vv_j^{(0)}\rangle\bigr|\ \le\ 1-\frac{1}{p}
\qquad\Longrightarrow\qquad
\mu(\vSet_0)\ \le\ 1-\frac{1}{p}\ <\ 1.
\]
Let $\vu_i^{(0)}:=\vv_i^{(0)}$. Then
\[
\Big\langle \frac{\vv_i^{(0)}-\vv_j^{(0)}}{\|\vv_i^{(0)}-\vv_j^{(0)}\|_2},\ \vu_i^{(0)}\Big\rangle
=\frac{1-\langle \vv_i^{(0)},\vv_j^{(0)}\rangle}{\sqrt{2-2\langle \vv_i^{(0)},\vv_j^{(0)}\rangle}}
=\sqrt{\frac{1-\langle \vv_i^{(0)},\vv_j^{(0)}\rangle}{2}}
\ \ge\ \sqrt{\frac{1-1/p}{2}}.
\]
Minimizing over all $i\neq j$ shows
\[
\rho(\vSet_0)\ \ge\ \gamma_0\ :=\ \sqrt{\frac{1/p}{2}}\ >\ 0.
\]
We now pad each vector with ones. 
Let $t:=d-d_0$ and define
\[
\widehat \vv_i\ :=\ (\vv_i^{(0)}, \mathbf 1_t)\ \in\mathbb R^{d},
\qquad
\vv_i\ :=\ \frac{\widehat \vv_i}{\|\widehat \vv_i\|_2}
=\frac{(\vv_i^{(0)}, \mathbf 1_t)}{\sqrt{1+t}}.
\]
where here $(\vv_i^{(0)}, \mathbf 1_t)$ denotes the lengthwise concatenation of $\vv_i^{(0)}$ and $\mathbf 1_t$ where $\mathbf 1_t$ is a vector of length $t$ of ones. Then for $i\neq j$,
\[
\langle \vv_i,\vv_j\rangle
=\frac{\langle \vv_i^{(0)},\vv_j^{(0)}\rangle + t}{1+t}
=1-\frac{1-\langle \vv_i^{(0)},\vv_j^{(0)}\rangle}{1+t}
\ \ge\ 1-\frac{1}{1+t}\ \ge\ 0,
\]
hence
\[
\mu(\vSet)\ =\ \max_{i\neq j}\bigl|\langle \vv_i,\vv_j\rangle\bigr|
\ \ge\ 1-\frac{1}{1+t}\ =\ 1-o(1),
\]
where the final equality holds since $t\rightarrow \infty$ increases $\frac{1}{1+t}\rightarrow 0$. 

On the other hand, if we set $\vu_i = (\vu_i^{(0)}, \mathbf 0_t)$, where $\vu_i^{(0)}$ are picked such that $\rho(\rmV^{(0)}, \rmU^{(0)})=\rho(\rmV^{(0)})$ and $\mathbf 0_t$ is a vector of length $t$ of all zeros, for any $i\neq j$,
\[
\rho(\mathbf{V}, \mathbf{U}) = \Big\langle \tfrac{\vv_i-\vv_j}{\|\vv_i-\vv_j\|_2},\,( \vu_i^{(0)}, \mathbf 0_t )\Big\rangle
=\Big\langle 
\tfrac{(\vv_i^{(0)}, \mathbf 1_t) - (\vv_j^{(0)}, \mathbf 1_t)}
{\norm{(\vv_i^{(0)}, \mathbf 1_t) - (\vv_j^{(0)}, \mathbf 1_t)}_2},\,
(\vu_i^{(0)},\mathbf 0_t)
\Big\rangle
=\Big\langle \tfrac{\vv_i^{(0)}-\vv_j^{(0)}}{\|\vv_i^{(0)}-\vv_j^{(0)}\|_2},\, \vu_i^{(0)}\Big\rangle
 = \rho(\rmV_0).
\]

Combining the bounds yields $\mu(\vSet)=1-o(1)$ while
$\rho(\vSet)\ge \sqrt{1/2p}>0$, completing the proof.
\end{proof}

\subsection{Additional Details for \Cref{subsec:full_construction}}\label{theory_sec_3.3}
\begin{theorem}[Full Construction]\label{thm:capacity-final}
    For any fact set $f$, generic key embeddings $\mathbf{K}$, and value embeddings
    $\mathbf{V}$ with $\rho(\mathbf{V})>0$, construct $\enc$ as in
    \Cref{subsec:encoder_twohot} and construct $\dec$ as in
    \Cref{subsec:construction_decoder}. Then the fact MLP
    \[
        \mathbf{g}(\mathbf{x})
        \;=\;
        \dec(\enc(\mathbf{x}))
        \;=\;
        \mathbf{D}\,\mathbf{E}\big(\sigma(\mathbf{G}\mathbf{x})\odot(\mathbf{A}\mathbf{x})\big)
    \]
    stores $f$ given $\mathbf{K}$ and $\mathbf{V}$, and has fact-storage cost
    \[
        \Theta\!\left([\rho(\mathbf{V})]^{-2}\,|\mathbf{K}|\,\log|\mathbf{V}|\right).
    \]
\end{theorem}

\begin{proof}
By \Cref{thm:NEW_main_decoding}, for any $\rho(\mathbf{V})>0$ there exist a
compressed dimension
\[
    m \;=\; \Theta\!\left([\rho(\mathbf{V})]^{-2}\log|\mathbf{V}|\right)
\]
and a linear decoder $\dec(\mathbf{x})=\mathbf{D}\mathbf{x}$ together with
compressed codes $\mathbf{C}=\{\mathbf{c}_i\}_{i=1}^{|\mathbf{V}|}$ such that
the dot-product decoding condition
\[
    \langle \mathbf{v}_i,\,\dec(\mathbf{c}_i)\rangle
    >
    \langle \mathbf{v}_j,\,\dec(\mathbf{c}_i)\rangle
    \qquad\forall i\neq j
\]
holds.  Fix such a $(\mathbf{C},\mathbf{D})$.

Given these compressed codes, apply \Cref{thm:gated_encoder_gadget_works}
coordinate-wise: for each $j\in[m]$, with $|\mathbf{K}|$ generic inputs and
targets $\{ \mathbf{c}_{f(i),j} \}_{i=1}^{|\mathbf{K}|}$, the theorem
guarantees a scalar-output gated encoder gadget that fits these values
exactly. Stacking the $m$ gadgets as in the encoder construction yields
$\enc$ with
\[
    \enc(\mathbf{k}_i) = \mathbf{c}_{f(i)} \quad\forall i,
\]
and total encoder parameter count $\Theta(m|\mathbf{K}|)$.

The composed MLP $\mathbf{g}=\dec\circ\enc$ thus satisfies
\[
    \mathbf{g}(\mathbf{k}_i) = \dec(\enc(\mathbf{k}_i)) = \dec(\mathbf{c}_{f(i)}),
\]
which decodes (under dot products with $\mathbf{V}$) to $\mathbf{v}_{f(i)}$
by the property of $\dec$ and $\mathbf{C}$. Hence $\mathbf{g}$ stores $f$.
Its parameter count is
\[
    \Theta(m|\mathbf{K}|) = \Theta\!\left([\rho(\mathbf{V})]^{-2}|\mathbf{K}|\log|\mathbf{V}|\right),
\]
as claimed.
\end{proof}

As it turns out, we may also prove a similar theorem using the result from \citet{bubeck2020networksizeweightssize} as follows: 

\begin{theorem}[Full construction]\label{thm:FULL_construction}
Let $\mathbf K=\{\vk_i\}_{i=1}^{|\mathbf K|}\subset\mathbb R^{d}$ be generic.
Let $\mathbf V=\{\vv_j\}_{j=1}^{|\mathbf V|}\subset\mathbb R^{d}$ with $\rho(\mathbf V)>0$,
and fix $f:[|\mathbf K|]\to[|\mathbf V|]$ and $\delta\in(0,1)$. Let $\rmU=\{\vu_j\}_{j=1}^{|\mathbf V|}\subset \R^{d}$. Additionally, set
\[
m\;\ge\;\frac{32}{\rho_{\min}(\mathbf V,\mathbf U)^2}\,
\ln\!\frac{4|\mathbf V|(|\mathbf V|-1)}{\delta},\qquad
\rmG\sim\mathcal{N}(0,1)^{m\times d},\qquad
\rmM:=\frac{1}{m}\,\rmG^{\!\top}\in\mathbb R^{d\times m}.
\]
where each coordinate $\mathbf G_{\ell, k}$ is sampled i.i.d from $\mathcal N(0,1)$. Then, with probability at least $1-\delta$ over $\rmG$, there exist
$\rmA\in\mathbb R^{\tilde m\times d}$ and $\vb\in\mathbb R^{\tilde m}$ with
$\tilde{m}=4\codeDim\,\lceil |\mathbf K|/d\rceil$ such that the one-hidden-layer ReLU network
\[
\mathbf V^{\!\top}\rmM\,\ReLU(\rmA\vx+\vb)\ \in\ \mathbb R^{|\mathbf V|}
\]
achieves for all $i,j$ such that $j \neq f(i)$: 
\[
\Big\langle \vv_{f(i)},\,\rmM\,\ReLU(\rmA \vk_i+\vb)\Big\rangle
-\Big\langle \vv_j,\,\rmM\,\ReLU(\rmA \vk_i+\vb)\Big\rangle
\ \ge\ \frac{\rho_{\min}(\mathbf V,\mathbf U)}{2}\,\big\|\vv_{f(i)}-\vv_j\big\|\,\|\vu_{f(i)}\|
\]
The number of trainable parameters that scale with $|\mathbf K|$ (the \emph{fact-storage cost})
is $\Theta\!\big(m\,|\mathbf K|\big)=\Theta\!\big(\rho_{\min}(\mathbf{V}, \mathbf{U})^{-2}\,|\mathbf K|\,\log|\mathbf{V}|\big)$.
\end{theorem}

\begin{proof}[Proof]
Define the $m$–dimensional codes $\vc_j:=\rmG\,\vu_j\in\mathbb R^{m}$ for $j\in[|\mathbf V|]$.
By \Cref{thm:NEW_main_decoding}, the stated lower bound on $m$ ensures that, with probability at least $1-\delta$,
for all $i$ and all $j\neq i$,
\begin{align}\label{eq: dp-dec-pf}
\big\langle \vv_i,\ \rmM\,\vc_i\big\rangle-\big\langle \vv_j,\ \rmM\,\vc_i\big\rangle
\ \ge\ \frac{\rho_{\min}}{2}\,\|\vv_i-\vv_j\|\,\|\vu_i\|\ >\ 0.
\end{align}

Note that in the above, $\vc_i$ are defined exactly as $\mH[i]$ in \Cref{thm:NEW_main_decoding}.

Apply \Cref{thm:bubeck_result} \emph{coordinatewise} to the dataset $\{(\vk_i,\,(\vc_{f(i)})_t)\}_{i}$ for each $t\in[m]$:
stacking the $m$ constructions produced by \Cref{thm:bubeck_result} yields a ReLU map with width
$\tilde m=4m\lceil |\mathbf K|/d\rceil$ and parameters $\rmA \in \mathbb R^{\tilde{m} \times d},\vb \in \mathbb R^{\tilde{m}}$, together with a fixed matrix
$\rmE\in\mathbb R^{m\times \tilde m}$, such that
\[
\rmE\,\ReLU(\rmA \vk_i+\vb)\ =\ \vc_{f(i)}\qquad\text{for all }i.
\]
Now set
\[
g(\vx)\ :=\ \rmM\,\rmE\,\ReLU(\rmA \vx+\vb).
\]
For each $\vk_i$ we have $\rmM\,\rmE\,\ReLU(\rmA \vk_i+\vb)=\rmM\,\vc_{f(i)}$, so the margin at $\vk_i$ equals the left–hand side of \eqref{eq: dp-dec-pf} with $i\mapsto f(i)$ (i.e., $g$ stores $f$).
Finally, only $(\rmA,\vb)$ scale with $|\mathbf K|$, giving the claimed $\Theta(m|\mathbf K|)$ fact–storage cost; substituting the bound on $m$ finishes the proof.
\end{proof}

\subsection{Additional Details for \Cref{sec:preconditioning}}
\label{app:subsec_embeddings_theory}

We provide theoretical results on embeddings and decodability.

\begin{theorem}[Affine invariance for 1-hidden-layer MLP with keys/values]
\label{thm:affine_invariance_app}
Consider a fact set $f : [\numKV] \to [\numKV]$, key embeddings $\mathbf{K} = \{\keyEmbeddings_i\}_{i=1}^\numKV \subset \R^{\kvDim}$, and value embeddings $\mathbf{V} = \{\valueEmbeddings_i\}_{i=1}^\numKV \subset \R^{\kvDim}$.
Assume there exist $\AOne\in\R^{m\times \kvDim}$, $\vb\in\R^m$, $\BOne\in\R^{\kvDim\times m}$ such that
\begin{equation}
\label{eq:affine_invariance_app_condition}
\langle \valueEmbeddings_{f(i)}-\valueEmbeddings_j,\; \BOne\,\ReLU(\AOne\keyEmbeddings_i+\vb)\rangle > 0
\quad\text{for all } i\in[\numKV],\; j\neq f(i).
\end{equation}
Then for any affine transformation\footnote{$\invertible{\kvDim}$ is the group of invertible $\kvDim \times \kvDim$ (real) matrices.} of the key and value embeddings:
\[
\tilde{\keyEmbeddings}_i = \tSubK\, \keyEmbeddings_i + \embeddingC_k,\quad \tSubK\in \invertible{\kvDim},\; \embeddingC_k\in\R^{\kvDim},\qquad
\tilde{\valueEmbeddings}_i = \mathbf{T_v}\, \valueEmbeddings_i + \embeddingC_v,\quad \mathbf{T_v}\in \invertible{\kvDim},\; \embeddingC_v\in\R^{\kvDim},
\]
there exist $\ATwo\in\R^{m\times \kvDim}$, $\vb'\in\R^m$, $\BTwo\in\R^{d\times m}$ such that
\[
\langle \tilde{\valueEmbeddings}_{f(i)}-\tilde{\valueEmbeddings}_j,\; \BTwo\,\ReLU(\ATwo\tilde{\keyEmbeddings}_i+\vb')\rangle > 0
\quad\text{for all } i\in[\numKV],\; j\neq f(i).
\]
\end{theorem}

\begin{proof}
Define
\[
\ATwo \coloneqq \AOne\,\tSubK^{-1},\quad
\vb' \coloneqq \vb - \AOne\,\tSubK^{-1} \embeddingC_k,\quad
\BTwo \coloneqq (\mathbf{T_v}^\top)^{-1} \BOne.
\]
Then for each $i$,
\[
\ReLU(\ATwo \tilde{\keyEmbeddings}_i + \vb') = \ReLU\big(\AOne\,\tSubK^{-1}(\tSubK \keyEmbeddings_i + \embeddingC_k) + \vb - \AOne\,\tSubK^{-1}\embeddingC_k\big) = \ReLU(\AOne\keyEmbeddings_i+\vb).
\]
Thus for any $i$ and $j\neq f(i)$,
\begin{align*}
\langle \tilde{\valueEmbeddings}_{f(i)}-\tilde{\valueEmbeddings}_j,\; \BTwo\ReLU(\ATwo\tilde{\keyEmbeddings}_i+\vb')\rangle
&= \langle \mathbf{T_v}(\valueEmbeddings_{f(i)}-\valueEmbeddings_j),\; (\mathbf{T_v}^\top)^{-1}\BOne\,\ReLU(\AOne\keyEmbeddings_i+\vb)\rangle \\
&= \langle \valueEmbeddings_{f(i)}-\valueEmbeddings_j,\; \BOne\,\ReLU(\AOne\keyEmbeddings_i+\vb)\rangle \;>\; 0,
\end{align*}
using \Cref{eq:affine_invariance_app_condition}.
\end{proof}

\subsubsection{Decodability and affine transformations on embeddings}
We study how the \emph{decodability} of embeddings changes after affine transformations. Starting from the definition from~\Cref{def:rho_R_tau}, we take the maximum over all decoder inputs:
\[
\rho(\mathbf{V}) \;\coloneqq\; \max_{\mathbf{U}} \;\min_{i\neq j}\;
\frac{\langle \valueEmbeddings_i-\valueEmbeddings_j,\,\vu_i\rangle}{\norm{\valueEmbeddings_i-\valueEmbeddings_j}\,\norm{\auxTheoryEmbed_i}},
\quad
\mathbf{V}=\{\valueEmbeddings_i\}_{i=1}^\numKV\subset\R^\kvDim,\ \mathbf{U}=\{\auxTheoryEmbed_i\}_{i=1}^\numKV\subset\R^d\setminus\{0\}.
\]
Given $\mathbf{V}$, consider new embeddings $\widetilde{\mathbf{V}}$ via the affine map
$\tilde \valueEmbeddings_i=\tMatrix \vv_i+\embeddingC$ with $\tMatrix\in\invertible{\kvDim}$, $\embeddingC\in\R^\kvDim$.

\begin{lemma}[Translation, scaling, and orthogonal invariance]
    For any $\embeddingC\in\R^\kvDim$, $\alpha>0$, and any orthogonal $\mathbf{R}\in \invertible{\kvDim}$,
    \[
    \rho(\mathbf{V}+\{\embeddingC\})=\rho(\mathbf{V}),\qquad
    \rho(\alpha\mathbf{V})=\rho(\mathbf{V}),\qquad
    \rho(\mathbf{R}\mathbf{V})=\rho(\mathbf{V}).
    \]
\end{lemma}
\begin{proof}
Each claim follows by the invariance of the objective:
(i) translation leaves all differences $\valueEmbeddings_i-\valueEmbeddings_j$ unchanged; (ii) positive scaling multiplies both the numerator and the $\norm{\valueEmbeddings_i-\valueEmbeddings_j}$ factor by $\alpha$; (iii) taking $\tilde{\auxTheoryEmbed}_i=\mathbf{R}\auxTheoryEmbed_i$, orthogonality preserves inner products and norms, hence each cosine is unchanged. Taking $\min$ and then $\max$ preserves equality.
\end{proof}

\begin{lemma}[Linear conditioning bound]
    Let $\tMatrix\in\invertible{\kvDim}$ with condition number
    $\kappa(\tMatrix)=\norm{\tMatrix}_2\norm{\tMatrix^{-1}}_2=\sigma_{\max}(\tMatrix)/\sigma_{\min}(\tMatrix)$.
    Then
    \[
    \;\;\frac{1}{\kappa(\tMatrix)}\,\rho(\mathbf{V})
    \;\le\; \rho(\tMatrix\mathbf{V})
    \;\le\; \kappa(\tMatrix)\,\rho(\mathbf{V}). \;\;
    \]
\end{lemma}
\begin{proof}
\emph{Lower bound.} Let $\mathbf{U}^\star=\{\vu_i^{\star}\}$ attain $\rho(\mathbf{V})$. We compute the cosine similarity term for $\tilde \vu_i:=\tMatrix^{-\top}\vu_i^\star$ given transformed embeddings $\mathbf{TV}$:
\[
\frac{\langle \tMatrix(\valueEmbeddings_i-\valueEmbeddings_j),\,\tilde \auxTheoryEmbed_i\rangle}{\norm{\tMatrix(\valueEmbeddings_i-\valueEmbeddings_j)}\,\norm{\tilde \auxTheoryEmbed_i}}
=\frac{\langle \valueEmbeddings_i-\valueEmbeddings_j,\,\auxTheoryEmbed_i^\star\rangle}
{\norm{\tMatrix(\valueEmbeddings_i-\valueEmbeddings_j)}\,\norm{\tMatrix^{-\top}\auxTheoryEmbed_i^\star}}
\ge \frac{1}{\kappa(\tMatrix)}\,
\frac{\langle \valueEmbeddings_i-\valueEmbeddings_j,\,\auxTheoryEmbed_i^\star\rangle}{\norm{\valueEmbeddings_i-\valueEmbeddings_j}\,\norm{\auxTheoryEmbed_i^\star}}.
\]
Taking $\min_{j\neq i}$ and then $\max$ over $\tilde{\mathbf{U}}$ gives the left inequality.

\emph{Upper bound.} Apply the lower bound from above to $\mathbf{V} = \tMatrix^{-1}(\mathbf{TV})$:
\[
\rho(\mathbf{V}) \ge \frac{1}{\kappa(\tMatrix^{-1})}\rho(\tMatrix\mathbf{V})
= \frac{1}{\kappa(\tMatrix)}\rho(\tMatrix\mathbf{V}),
\]
so $\rho(\tMatrix\mathbf{V}) \le \kappa(\tMatrix)\,\rho(\mathbf{V})$.
\end{proof}

\begin{remark}[Embedding-aware bound]
    Let $\mathbf{C}=\tMatrix^\top\tMatrix\succ0$ and define
    \[
    \kappa_{\mathrm{eff}}(\tMatrix;\mathbf{V},\mathbf{U})
    \;\coloneqq\;
    \max_{i}\max_{j\neq i}\sqrt{
    \frac{(\valueEmbeddings_i-\valueEmbeddings_j)^\top \mathbf{C}\,(\valueEmbeddings_i-\valueEmbeddings_j)}{\norm{\valueEmbeddings_i-\valueEmbeddings_j}^2}\cdot
    \frac{\auxTheoryEmbed_i^\top \mathbf{C}^{-1}\auxTheoryEmbed_i}{\norm{\auxTheoryEmbed_i}^2}}.
    \]
    Intuitively, $\kappa_{\mathrm{eff}}(\tMatrix;\mathbf{V},\mathbf{U})$ captures the worst-case conditioning of $\tMatrix$, when its action is restricted to the subspaces $\mathrm{span}(\{\valueEmbeddings_i - \valueEmbeddings_j, \auxTheoryEmbed_i\})$ for all $i \neq j$.
    Then computing the cosine similarity term for $\tilde \auxTheoryEmbed_i=\tMatrix^{-\top}\auxTheoryEmbed_i$ yields
    \[
    \rho(\tMatrix\mathbf{V})
    \;\ge\;
    \frac{1}{\kappa_{\mathrm{eff}}(\tMatrix;\mathbf{V},\mathbf{U})}\;
    \min_{i\neq j}\frac{\langle \valueEmbeddings_i-\valueEmbeddings_j,\,\auxTheoryEmbed_i\rangle}{\norm{\valueEmbeddings_i-\valueEmbeddings_j}\,\norm{\auxTheoryEmbed_i}}.
    \]
    In particular, with $\mathbf{U}=\mathbf{U}^\star$ that attains $\rho(\mathbf{V})$,
    \[
    \rho(\tMatrix\mathbf{V}) \;\ge\; \frac{\rho(\mathbf{V})}{\kappa_{\mathrm{eff}}(\tMatrix;\mathbf{V},\mathbf{U}^\star)},
    \qquad
    \kappa_{\mathrm{eff}}(\tMatrix;\mathbf{V},\mathbf{U}^\star)\le \kappa(\tMatrix).
    \]
\end{remark}

\begin{remark}[Tightness]

The $1/\kappa(\tMatrix)$ lower bound is tight in general.

As a concrete example for $\kvDim=2$, consider \( \valueEmbeddings_1=(0,0),\ \valueEmbeddings_2=(1,0),\ \valueEmbeddings_3=(1,-\varepsilon). \) For $i=1$, the tightest cosine margin is between $\ve_1$ and $\ve_1-\varepsilon \ve_2$. The optimal $\auxTheoryEmbed_1^\star$ then lies in the direction of their angle bisector, giving \( \rho(\mathbf{V})=\Theta(\varepsilon) \) as $\varepsilon\to0$. Then, consider \( \tMatrix=\mathrm{diag}(\sigma_{\max},\sigma_{\min}) \), for which $\kappa(\tMatrix)=\sigma_{\max}/\sigma_{\min}$. A direct calculation with $\tilde \auxTheoryEmbed_1=\tMatrix^{-\top}\auxTheoryEmbed_1^\star$ shows \( \rho(\tMatrix\mathbf{V}) \approx \rho(\mathbf{V})/\kappa(\tMatrix) \) as $\varepsilon\to0$.
 showing the lower bound factor $1/\kappa_2(\mathbf{T})$ is tight.
\end{remark}

\subsection{Bit Complexity}
\label{app:subsec_bit_complexity}
\begin{theorem} \label{thm: encoder-bits-UB}
    Let $\numKeys = |K|$. Suppose that $\hiddenDim,\tokenDim,\encDim = O(\poly \, \numKeys)$, that $\sigma$ is an $L^2$ continuously differentiable function, that $\gaussMat$ is such that all its rows are i.i.d. $\gaussRow{i} \sim Normal(0,\rmI_\tokenDim)$, that for all $\key_i \in K$, $\key_i$ is sampled from a rotationally invariant distribution with $\|\key_i\| \leq O(\poly \, \numKeys)$, that the targets $\|\target_i\| \leq O(\poly \, \numKeys)$, that $\numKeys \geq C_0\tokenDim \hiddenDim$ for some sufficiently large universal constant $C_0$, that $\E[\sigma(\gaussRow{1}^\top \key_i)\mid \key_i]=0$ for all $i$, and that $\rho \geq O(\frac{1}{\poly \, \numKeys})$. Then with high probability (depending on $\numKeys$), the encoder / decoder construction described in~\Cref{thm:capacity-final} requires $O(\log \numKeys)$ bits per parameter to store, of which there are $O(\poly\numKeys)$.
\end{theorem}
\begin{proof}
    See \Cref{pf: encoder-bits-UB}.
\end{proof}

\subsubsection{Noisy Decoding}

\begin{theorem}[Noisy decoding via JL, Rademacher case]\label{rad_noisy_decoding}
Let $\decMat \in \{-1,+1\}^{\encDim\times \tokenDim}$ have i.i.d.\ Rademacher entries
($\Pr(\decMat_{kl}=1)=\Pr(\decMat_{kl}=-1)=\tfrac{1}{2}$) and set $\mMat := \frac{1}{\encDim}\decMat^\top$.
For each $i \in [N]$, let $\vvec_i,\uvec_i \in \R^\tokenDim$ and define
\[
\rho
:= \min_{i\neq j}
   \frac{\langle \vvec_i - \vvec_j,\;\uvec_i\rangle}
        {\|\vvec_i - \vvec_j\|\,\|\uvec_i\|} \;>\; 0.
\]
Let the noisy codes be
\[
\rmH[i] := (\decMat \uvec_i)\odot(1+\nu_i),\qquad
\nu_i \in [-\varepsilon,\varepsilon]^\encDim,\quad \varepsilon\in[0,1),
\]
and define scores $s_{ij} := \langle \vvec_j,\;\mMat \rmH[i]\rangle$.
Then there is a universal constant $C>0$ such that if
\[
m \;\ge\; \frac{C}{\rho^2}\,
          \ln\!\frac{4N(N-1)}{\delta},
\]
then with probability at least $1-\delta$ over $\decMat$, we have, simultaneously for all $i\neq j$,
\[
s_{ii} - s_{ij}
\;\ge\;
\Big(\frac{\rho}{2} - 4\varepsilon\Big)\,
\|\vvec_i - \vvec_j\|\,\|\uvec_i\|.
\]
\end{theorem}
\begin{proof}
    See \Cref{pf:rad_noisy_decoding}.
\end{proof}

\subsubsection{Bounding The Magnitudes}
 \begin{lemma} \label{lm: mag-prob-LB}
Let $\key_1,\dots,\key_{\numKeys} \in \R^\tokenDim$ be i.i.d.\ random vectors with
$\key_i \sim \mathcal N(0,\rmI_\tokenDim)$. Then for every $c>0$ there exists a constant
$C=C(c)>0$ such that
\[
  \Pr\!\left[
    \max_{1\le i\le \numKeys} \|\key_i\|_2
    \le C\big(\sqrt \tokenDim + \sqrt{\log \numKeys}\big)
  \right]
  \;\ge\; 1 - \numKeys^{-c}.
\]
\end{lemma}\label{bound_gates_and_keys}
\begin{proof}
    See \Cref{pf: mag-prob-LB}.
\end{proof}

\begin{lemma}[Row covariance is well-conditioned under rotationally invariant model]
\label{lem:row_cov_rot_inv}
Fix $\tokenDim,\hiddenDim \in \mathbb N$ and let
\[
    \key \in \R^\tokenDim
    \quad\text{and}\quad
    \gaussRow{1},\dots,\gaussRow{\hiddenDim} \in \R^{\tokenDim}
\]
be random vectors such that:
\begin{enumerate}[label=(\roman*)]
    \item $\key$ has a rotationally invariant distribution 
    \item $\gaussRow{1},\dots,\gaussRow{\hiddenDim}$ are i.i.d.\ rotationally invariant.
    \item $\sigma:\R\to\R$ is a non-constant measurable function with
    \(\E[\sigma(\gaussRow{1}^\top \key)^2] < \infty\) and $\E[\sigma(\gaussRow{1}^\top \key)\mid \key]=0$ a.s.
\end{enumerate}
Define the random row vector $\randVec^\top \in \R^{\tokenDim \hiddenDim}$ by
\[
    \randVec^\top (\key, \gaussRow{1}, ..., \gaussRow{\hiddenDim})
    :=
    \big(
        \sigma(\gaussRow{1}^\top \key) \key^\top,\;
        \dots,\;
        \sigma(\gaussRow{\hiddenDim}^\top \key) \key^\top
    \big),
\]
and let
\[
    \Sigma_{\mathrm{row}} := \E[ \randVec \randVec^\top ] \in \R^{\tokenDim \hiddenDim\times \tokenDim \hiddenDim}.
\]
Then there exists a constant $c>0$, depending only on the distributions of $\key$, $\gaussRow{\ell}$, and $\sigma$ (but \emph{independent} of $\numKeys$), such that
\[
    \lambda_{\min}(\Sigma_{\text{row}}) = \lambda_{\max}(\Sigma_{\mathrm{row}}) = c.
\]
In particular,$\lambda_{\min}(\Sigma_{\mathrm{row}}) \;\ge\; \numKeys^{-C_1}$ and $\lambda_{\max}(\Sigma_{\mathrm{row}}) \leq \numKeys^{C_2}$ 
for some fixed exponents $C_1,C_2$ and all $\numKeys$ (i.e., the lower bound is $\frac{1}{\poly(\numKeys)}$).
\end{lemma}
\begin{proof}
    See \Cref{pf: row_cov_rot_inv}.
\end{proof}

Equipped with \Cref{lem:row_cov_rot_inv} (which gives us assumption ii) in the theorem below) we may now finish the prove that the parameter magnitudes are bounded.

\begin{theorem}[Encoder weight norm bound]\label{bounding_theorem}
Fix an output coordinate $j$ and consider the linear system
\[
    \mMat\,\va = \target,
\]
where $\mMat \in \R^{\numKeys\times \tokenDim \hiddenDim}$ and $\va = \mathrm{vec}(\gateMat)\in\R^{\tokenDim \hiddenDim}$. Assume:
\begin{enumerate}[label=(\roman*)]
    \item The $i$-th row of $\mMat$ is
    \[
        \randVec_i^\top
        = \big(\sigma(\gaussRow{1}^\top \key_i)\key_i^\top,\dots,\sigma(\gaussRow{\hiddenDim}^\top \key_i)\key_i^\top\big),
    \]
    where $\{\key_i\}_{i=1}^\numKeys$ and $\{\gaussRow{\ell}\}_{\ell=1}^\hiddenDim$ are independent, rotationally invariant subgaussian random vectors in $\R^\tokenDim$, and
    $\sigma$ is continuously differentiable and non-constant.
    \item The covariance $\Sigma_{\mathrm{row}} := \E[\randVec_i \randVec_i^\top]$ satisfies
    $\lambda_{\min}(\Sigma_{\mathrm{row}})\ge\lambda_0>0$ and
    $\lambda_{\max}(\Sigma_{\mathrm{row}})\le \Lambda_0<\infty$, with $\lambda_0,\Lambda_0$ independent of $\numKeys$.
    \item The targets $\target\in\R^\numKeys$ obey $|\vo_i|\le B(\numKeys)$ for all $i$, where $B(\numKeys)\le \poly(\numKeys)$.
    \item $\numKeys \ge C_0\,\tokenDim \hiddenDim$ for a sufficiently large absolute constant $C_0$.
\end{enumerate}
Let $\va_\star$ be the minimum–$\ell_2$–norm solution of $\mMat \va=\target$ (i.e.\ $a_\star = \mMat^\dagger \target$). Then with probability at least $1 - e^{-c\numKeys}$, $c > 0$ we have
\[
    \|\va_\star\|_2 \;\le\; \poly(\numKeys).
\]
\end{theorem}

\begin{proof}
    See section \Cref{bounding_theorem}.
\end{proof}

\subsubsection{Precision Bound}

\begin{lemma}[Encoder is Lipschitz in the parameters] \label{lm: enc_lipschitz}
Fix a number of facts $\numKeys$ and keys $\{\key_i\}_{i=1}^\numKeys \subset \R^\tokenDim$.
Consider the scalar-output gated encoder
\[
\enc_\theta(\inVec)
= \mathbf{1}_\hiddenDim^\top\big[\sigma(\gaussMat \inVec)\odot(\gateMat\inVec)\big]
= \sum_{r=1}^\hiddenDim \sigma(\langle \mathbf{g}_r,\inVec\rangle)\,\langle \mathbf{a}_r,\inVec\rangle,
\]
where $\gateMat,\gaussMat\in\R^{\hiddenDim\times \tokenDim}$ have rows $\mathbf{a}_r^\top,\gaussRow{r}^\top$, and $\theta\in\R^P$ is the vector of
all entries of $\gateMat,\gaussMat$.

Assume:
\begin{enumerate}
    \item $\|\key_i\|_2 \le R_\inVec(\numKeys)$ for all $i$, with $R_\inVec(\numKeys)\le \poly(\numKeys)$.
    \item $\|\theta\|_2 \le R_\theta(\numKeys)$, with $R_\theta(\numKeys)\le \poly(\numKeys)$.
    \item The width and input dimension satisfy $\hiddenDim,\tokenDim \le \poly(\numKeys)$, so that $P = 2\hiddenDim\tokenDim \le \poly(\numKeys)$.
    \item The activation $\sigma:\R\to\R$ is continuously differentiable and on the interval
    $[-B(\numKeys),B(\numKeys)]$ with $B(\numKeys):=R_\theta(\numKeys)R_\inVec(\numKeys)$ we have
    \[
      |\sigma(t)| \le C_\sigma,\qquad |\sigma'(t)| \le C'_\sigma
      \quad\forall t\in[-B(\numKeys),B(\numKeys)],
    \]
    for some constants $C_\sigma,C'_\sigma$ independent of $\numKeys$. \footnote{
Since $\sigma \in C^{1}$ and all preactivations satisfy 
$|\langle \mathbf g_r, \inVec\rangle| \le R_\theta(\numKeys)R_\inVec(\numKeys)$,
they lie in the compact interval $[-B(\numKeys),B(\numKeys)]$. 
By continuity, $\sigma$ and $\sigma'$ are bounded on this interval, yielding constants 
$C_\sigma, C'_\sigma < \infty$. 
This ensures $\enc_\theta$ is Lipschitz in $\theta$, with constants growing at most polynomially in $\numKeys$.
}
\end{enumerate}
Then for each key $\key_i$ there exists a constant $L(\numKeys)\le \poly(\numKeys)$ such that for all
parameter vectors $\theta,\theta'$ with $\|\theta\|_2,\|\theta'\|_2\le R_\theta(\numKeys)$,
\[
  |\enc_\theta(\key_i) - \enc_{\theta'}(\key_i)|
  \;\le\;
  L(\numKeys)\,\|\theta-\theta'\|_2.
\]
In particular, $\enc_\theta(\key_i)$ is Lipschitz in $\theta$ with Lipschitz constant at most
polynomial in $\numKeys$.
\end{lemma}

\begin{proof}
    See \Cref{pf: enc_lipshitz}.
\end{proof}

\begin{theorem}[Polynomial precision for encoder parameters]\label{bounding_precision}
Let $\numKeys$ be the number of facts, and assume the noisy decoding theorem
above holds for some choice of $\encDim$ (so that, for any codes whose noise
is at most a fixed constant multiple of $\rho$, decoding is
still correct).

Assume the following polynomial bounds:
\begin{enumerate}[label=(\roman*)]
    \item (Margin) $\rho \;\ge\; 1/\poly(\numKeys)$.
    \item (Lipschitz in parameters) For each key $\key_i$ and all
    encoder parameter vectors $\theta,\theta'$,
    \[
        \|\enc_\theta(\key_i) - \enc_{\theta'}(\key_i)\|
        \;\le\;
        L(\numKeys)\,\|\theta-\theta'\|
        \quad\text{with } L(\numKeys)\le\poly(\numKeys).
    \]
    \item (Parameter count) The number of encoder parameters satisfies
    $P \le \poly(\numKeys)$.
    \item\label{prec-assm-mag} (Magnitude) There is an encoder $\theta_\star$ such that
    $H_\star[i] := \enc_{\theta_\star}(\key_i) = \decMat \uvec_i$ and
    $\|\theta_\star\|_\infty \le \poly(\numKeys)$. 
\end{enumerate}
Then there exists a constant $c>0$ such that if we quantize each
coordinate of $\theta_\star$ to the grid $\numKeys^{-c}\Z$, obtaining
$\tilde\theta$, the corresponding codes
$\tilde H[i] := \enc_{\tilde\theta}(\key_i)$ still satisfy the conditions
of the noisy decoding theorem and hence decode all $\numKeys$ facts
correctly. In particular, each encoder parameter requires only
$O(\log \numKeys)$ bits of precision.
\end{theorem}

\begin{proof}
    See \Cref{pf: bounding_precision}.
\end{proof}

\subsection{Spherical Chebyshev Bounds with a Fixed Anchor}
\label{app:subsec_spherical_shevyshev}
We derive explicit lower and upper bounds on the spherical Chebyshev value $\rho^{*}$ of the star $\{\inEmbedding_{aj}\}_{j\neq a}$. We show (i) general bounds with no assumptions, (ii) simplifications under unit-norm embeddings, and (iii) coarse coherence-based corollaries. 

Let $\vvec_1,\dots,\vvec_n\in\R^{\modelDim}$ and define, for any ordered pair $(i,j)$ with $i\neq j$,
\[
\inEmbedding_{ij} \;:=\; \frac{\vvec_i - \vvec_j}{\|\vvec_i - \vvec_j\|}.
\]
We \emph{always} assume a \textbf{fixed anchor} index $a$ and consider only the star
\[
\{\,\inEmbedding_{aj} : \ j\neq a\,\}.
\]

We are then interested in the following quantity: 
\begin{definition}
    Define the Spherical Chebyshev value as 
    \begin{align*}
        \rho^* \;:=\; \max_{\|\canCen\|=1}\ \min_{j\neq a}\ \canCen^\top \inEmbedding_{aj}
    \end{align*}
\end{definition}

the cosine of the smallest spherical cap covering the star induced by anchor $a$.

\subsubsection{General bounds (no norm assumptions on $\vvec_i$)}

For notational simplicity, define
\begin{align*}
    m_{\mathrm{edge}}
    := \min_{\substack{j\neq k\\ j\neq a,\ k\neq a}}
    \inEmbedding_{aj}^\top \inEmbedding_{ak},
\end{align*}

Then we have the following result. 

\begin{lemma}[Spherical Chebyshev sandwich for a star]\label{lem:star_cheb_general}
For the spherical Chebyshev value $\rho^*$ as defined above we have 
\[
m_{\mathrm{edge}}
\;\le\;
\rho^*
\;\le\;
\sqrt{\frac{1 + m_{\mathrm{edge}}}{2}}.
\]
\end{lemma}

\begin{proof}
For the lower bound, fix $j_0\neq a$ and take $\vc = \vx_{aj_0}$. Then $\|\vc\|=1$ and
\[
\min_{j\neq a} \vc^\top \vx_{aj}
= \min_{j\neq a} \vx_{aj_0}^\top \vx_{aj}
= \min\Big(1,\ \min_{\substack{j\neq a\\ j\neq j_0}} \vx_{aj_0}^\top \vx_{aj}\Big)
\;\ge\;
m_{\mathrm{edge}},
\]
so $\rho^* \ge m_{\mathrm{edge}}$.

For the upper bound, pick $j,k$ with $j\neq k$, $j\neq a$, $k\neq a$ such that
$\vx_{aj}^\top \vx_{ak} = m_{\mathrm{edge}}$. For any unit $\vc$,
\[
\min_{i\neq a} \vc^\top \vx_{ai}
\;\le\;
\min(\vc^\top \vx_{aj},\ \vc^\top \vx_{ak}),
\]
hence
\[
\rho^*
\;\le\;
\sup_{\|\vc\|=1} \min(\vc^\top \vx_{aj},\ \vc^\top \vx_{ak}).
\]

Let $P := \operatorname{span}\{\vx_{aj},\vx_{ak}\}$. Orthogonal projection onto $P$ cannot
decrease both inner products simultaneously, so the supremum is attained by some unit
$\vc\in P$. In an orthonormal basis of $P$, write
\[
\vx_{aj} = (1,0),\quad
\vx_{ak} = (\cos\theta,\sin\theta),\quad
\vc = (\cos\varphi,\sin\varphi),
\]
where $\theta := \arccos(\vx_{aj}^\top \vx_{ak})$ so $\cos\theta = m_{\mathrm{edge}}$.
Then
\[
\vc^\top \vx_{aj} = \cos\varphi,
\qquad
\vc^\top \vx_{ak} = \cos(\theta-\varphi),
\]
and we must maximize
\[
f(\varphi) := \min\big(\cos\varphi,\ \cos(\theta-\varphi)\big).
\]
On $[0,\pi]$, $\cos$ is strictly decreasing, so $f$ is maximized when
$\cos\varphi = \cos(\theta-\varphi)$, i.e.\ $\varphi = \theta/2$, giving
\[
\sup_{\|\vc\|=1}\min(\vc^\top \vx_{aj},\vc^\top \vx_{ak})
= \cos(\theta/2).
\]
Therefore $\rho^* \le \cos(\theta/2)$. Using
$\cos^2(\theta/2) = \tfrac{1+\cos\theta}{2}$ and $\cos\theta = m_{\mathrm{edge}}$, we obtain
\[
\rho^*
\;\le\;
\sqrt{\frac{1+\cos\theta}{2}}
=
\sqrt{\frac{1 + m_{\mathrm{edge}}}{2}}.
\]
Combining both bounds yields the claim.
\end{proof}

\subsubsection{Unit-norm specialization.}

For notational simplicity, define 
\[
    s_a := \max_{j\neq a} \vvec_a^\top \vvec_j
\]

\begin{lemma}[Spherical Chebyshev bounds for a star: unit-norm case]\label{lem:star_cheb_unit}
In the setting of \Cref{lem:star_cheb_general}, assume in addition that
$\|\vvec_i\|=1$ for all $i\in[n]$. Then 
\[
    \sqrt{\frac{1 - s_a}{2}}
    \;\le\;
    \rho^*
    \;\le\;
    \sqrt{\frac{1 + m_{\mathrm{edge}}}{2}}.
\]
\end{lemma}
\begin{proof}
    The upper bound follows directly from \Cref{lem:star_cheb_general}. When $\|\vvec_i\|=1$ for all $i$,
\[
\|\vvec_a - \vvec_j\|=\sqrt{2 - 2\,\vvec_a^\top \vvec_j}.
\]

By direct calculation,
\[
\vvec_a^\top \inEmbedding_{aj}
\;=\;
\frac{1 - \vvec_a^\top \vvec_j}{\sqrt{2 - 2\,\vvec_a^\top \vvec_j}}
\;=\;
\sqrt{\frac{1 - \vvec_a^\top \vvec_j}{2}},
\]
so
\begin{equation}
\rho^* \;\ge\; \min_{j\neq a}\;\sqrt{\frac{1 - \vvec_a^\top \vvec_j}{2}}.
\label{eq:unit-lb-ij}
\end{equation}
Writing $s_a:=\max_{j\neq a} \vvec_a^\top \vvec_j$ (the anchor's nearest neighbor in cosine),
\begin{equation}
\rho^* \;\ge\; \sqrt{\frac{1 - s_a}{2}}.
\label{eq:unit-lb-sa-ij}
\end{equation}
\end{proof}

To obtain bounds that depend only on a single global parameter, we now suppose the vectors satisfy a standard coherence condition.
\subsubsection{Coherence-style corollaries (unit-norm)}

\begin{lemma}[Coherence-style bounds for a fixed-anchor star]\label{lem:star_coherence}
In the setting of \Cref{lem:star_cheb_general}, assume in addition that
$|\vvec_i^\top \vvec_j| \le \mu$ for all $i\neq j$, with $\mu \in [0,1)$ and $\|\vvec_i\|=1$ for all $i\in[n]$.
Then the spherical Chebyshev value $\rho^*$ satisfies
\[
    \sqrt{\frac{1 - \mu}{2}}
    \;\le\;
    \rho^*
    \leq
    \sqrt{\frac{1}{2}\!\left(1 + \frac{1 + 3\mu}{2 - 2\mu}\right)}.
\]
\end{lemma}

\begin{proof}
The coherence bound implies, for the anchor $a$,
\[
s_a := \max_{j\neq a} \vvec_a^\top \vvec_j \;\le\; \mu.
\]
By \eqref{eq:unit-lb-sa-ij} from the unit-norm specialization,
\[
\rho^* \;\ge\; \sqrt{\frac{1 - s_a}{2}}
\;\ge\;
\sqrt{\frac{1 - \mu}{2}}.
\]

For any $j\neq k$ by direct computation,
\[
\inEmbedding_{aj}^\top \inEmbedding_{ak}
=
\frac{1 - \vvec_a^\top \vvec_j - \vvec_a^\top \vvec_k + \vvec_j^\top \vvec_k}
{\sqrt{(2 - 2\,\vvec_a^\top \vvec_j)(2 - 2\,\vvec_a^\top \vvec_k)}}.
\]
Write $a_j := \vvec_a^\top \vvec_j$, $a_k := \vvec_a^\top \vvec_k$,
$b_{jk} := \vvec_j^\top \vvec_k$. Then
$|a_j|,|a_k|,|b_{jk}|\le\mu$, so
\[
1 - a_j - a_k + b_{jk}
\;\le\;
1 + |a_j| + |a_k| + |b_{jk}|
\;\le\;
1 + 3\mu,
\]
and since $a_j,a_k \le \mu$,
\[
2 - 2a_j \;\ge\; 2 - 2\mu,
\qquad
2 - 2a_k \;\ge\; 2 - 2\mu,
\]
hence
\[
\sqrt{(2 - 2a_j)(2 - 2a_k)} \;\ge\; 2 - 2\mu.
\]
Therefore, for all $j\neq k$,
\[
\inEmbedding_{aj}^\top \inEmbedding_{ak}
\;\le\;
\frac{1 + 3\mu}{2 - 2\mu},
\]
and taking the minimum over $j\neq k$ yields
\[
m_{\mathrm{edge}}
:= \min_{j\neq k} \inEmbedding_{aj}^\top \inEmbedding_{ak}
\;\le\;
\frac{1 + 3\mu}{2 - 2\mu}.
\]

By \Cref{lem:star_cheb_general},
\[
\rho^*
\;\le\;
\sqrt{\frac{1 + m_{\mathrm{edge}}}{2}}
\;\le\;
\sqrt{\frac{1}{2}\!\left(1 + \frac{1 + 3\mu}{2 - 2\mu}\right)}.
\]
Combining with the lower bound completes the proof.
\end{proof}

\subsection{Deferred proofs}\label{app:extended_theory_proofs}
\subsubsection{Proof of \Cref{lem:katri_rao_matroid_union}}\label{pf:katri_rao_matroid_union}
\begin{proof}
We proceed in three steps:
\begin{enumerate}
    \item Proof a Matroid Union Theorem sublemma which we use in Part 4.
    \item Establish the rank upper bound from linear algebra principles.
    \item Show that the set of $\mathbf{K}$ achieving this bound is Zariski open.
    \item Show that this set is non-empty by constructing a $\mathbf{K}$ that achieves the bound.
\end{enumerate}

\textbf{Part 1: Matroid Union Theorem Sublemma}

\begin{lemma}
    The rank $R(\mathbf{\Sigma})$ is also given by:
    \[
        R(\mathbf{\Sigma}) = \max_{\substack{\rmI_1,\dots,\rmI_d \subseteq [|\mathbf{K}|]\\
        \rmI_i \cap \rmI_j = \varnothing\ \forall i\neq j\\
        \bigcup_{i=1}^d \rmI_i = [|\mathbf{K}|]}}
        \; \left[\sum_{i=1}^d \operatorname{rank}\!\big(\boldsymbol{\Sigma}[:, \rmI_i]\big)\right].
    \]
\end{lemma}
\begin{proof}
    Define $R_k(\mathbf{\Sigma}, S) = \min_{S'\subseteq S}\Big[|S| - |S'| + k\cdot \textrm{\upshape rank}(\mathbf{\Sigma}[:, S'])\Big]$.
    
    We first prove by induction on $d$ that $R_k(\mathbf{\Sigma}, S)$ is the rank of $S$ in the matroid union of $d$ copies of the matroid of $\mathbf{\Sigma}$. 

    The base case is $d=1.$ In this case $R_1(\mathbf{\Sigma}, S) = \min_{S'\subseteq S}\Big[|S| - |S'| + \textrm{\upshape rank}(\mathbf{\Sigma}[:, S'])\Big]$ is minimized for $S = S',$ so $R_1(\mathbf{\Sigma}, S) = \textrm{\upshape rank}(\mathbf{\Sigma}[:, S])$, which is exactly the rank of $S$ in the matroid union of $1$ copy of the matroid of $\mathbf{\Sigma}$ (just the matroid of $\mathbf{\Sigma}$).
    
    Now, for the inductive step, suppose that the inductive hypothesis is true for $d-1$. By the Matroid Union Theorem\footnote{See Theorem 11.3.1 of \citet{oxley2011matroid}.} between the matroid of $\mathbf{\Sigma}$ and the matroid union of $d-1$ copies of $\mathbf{\Sigma}$, the rank of $S$ under the matroid union of $d$ copies of the matroid of $\mathbf{\Sigma}$ is given by
    \begin{align*}
        &\min_{S'\subseteq S}\Big[|S - S'| + \textrm{\upshape rank}(\mathbf{\Sigma}[:, S']) + R_{d-1}(\mathbf{\Sigma},S')\Big]\\
        =&\min_{S'\subseteq S}\Big[|S| - |S'| + \textrm{\upshape rank}(\mathbf{\Sigma}[:, S']) + \min_{S''\subseteq S'}\Big[|S'| - |S''| + (d-1)\textrm{\upshape rank}(\mathbf{\Sigma}[:, S''])\Big]\Big]\\
        =&\min_{S''\subseteq S'\subseteq S}\Big[|S| - |S''| + \textrm{\upshape rank}(\mathbf{\Sigma}[:, S']) + (d-1)\textrm{\upshape rank}(\mathbf{\Sigma}[:, S''])\Big]\\
        =&\min_{S''\subseteq S}\Big[|S| - |S''| + d\cdot \textrm{\upshape rank}(\mathbf{\Sigma}[:, S''])\Big]\\
        =& R_d(\mathbf{\Sigma},S),
    \end{align*}
    as desired.

    Now, we prove that \[
        R(\mathbf{\Sigma}) = R_d(\mathbf{\Sigma}, [|\mathbf{K}|]) = \max_{\substack{\rmI_1,\dots,\rmI_d \subseteq [|\mathbf{K}|]\\
        \rmI_i \cap \rmI_j = \varnothing\ \forall i\neq j}}
        \; \left[\sum_{i=1}^d \operatorname{rank}\!\big(\boldsymbol{\Sigma}[:, \rmI_i]\big)\right].
    \]
    First, note that by the definition of the matroid union,
    \begin{align*}
        R(\mathbf{\Sigma}) &= \max\left\{ \left|\bigcup_{i=1}^d \rmI_i \right| \, \Big| \, \forall i\in [d],\quad \text{rank}(\mathbf{\Sigma}[:, \rmI_i]) = |\rmI_i|\right\}\\
        &= \max\left\{ \left|\bigcup_{i=1}^d \rmI_i \right| \, \Big| \, \forall i\in [d],\quad \text{rank}(\mathbf{\Sigma}[:, \rmI_i]) = |\rmI_i|, \quad\quad \forall i\neq j \in [d],\quad \rmI_i\cap \rmI_j = \varnothing\right\}\\
        &= \max\left\{ \sum_{i=1}^d \left|\rmI_i \right| \, \Big| \, \forall i\in [d],\quad \text{rank}(\mathbf{\Sigma}[:, \rmI_i]) = |\rmI_i|, \quad\quad \forall i\neq j \in [d],\quad \rmI_i\cap \rmI_j = \varnothing\right\}\\
        &= \max\left\{ \sum_{i=1}^d \text{rank}(\mathbf{\Sigma}[:, \rmI_i]) \, \Big| \, \forall i\in [d],\quad \text{rank}(\mathbf{\Sigma}[:, \rmI_i]) = |\rmI_i|, \quad\quad \forall i\neq j \in [d],\quad \rmI_i\cap \rmI_j = \varnothing\right\}\\
        &= \max\left\{ \sum_{i=1}^d \text{rank}(\mathbf{\Sigma}[:, \rmI_i]) \, \Big| \, \quad\quad \forall i\neq j \in [d],\quad \rmI_i\cap \rmI_j = \varnothing\right\}\\
        &= \max_{\substack{\rmI_1,\dots,\rmI_d \subseteq [|\mathbf{K}|]\\
        \rmI_i \cap \rmI_j = \varnothing\ \forall i\neq j}}
        \; \left[\sum_{i=1}^d \operatorname{rank}\!\big(\boldsymbol{\Sigma}[:, \rmI_i]\big)\right].
    \end{align*}

    This completes our proof.
\end{proof}

\textbf{Part 2: Rank Upper Bound}

We first derive the upper bound for $\normDec(\mathbf{\Sigma}, \mathbf{K})$.
The matrix $\normDec \equiv \normDec(\mathbf{\Sigma}, \mathbf{K})$ is a $|\mathbf{K}| \times (dh)$ matrix.
The definition $\normDec = [\textrm{\upshape diag}(\mathbf{\Sigma}_1)\mathbf{K}, \ldots, \textrm{\upshape diag}(\mathbf{\Sigma}_{h})\mathbf{K}]$ concatenates by $h$ blocks of size $|\mathbf{K}| \times d$.

The columns of $\normDec$ can be re-grouped to form $d$ blocks of size $|\mathbf{K}| \times h$. Let $\normDec_j$ be the $j$-th new block, $j \in [d]$. This block contains all columns from $\normDec$ that were constructed using the $j$-th column of $\mathbf{K}$, $\mathbf{K}[:, j]$. This block can be written as:
\[
\normDec_j = \textrm{\upshape diag}(\mathbf{K}[:, j])\mathbf{\Sigma}^\top
\]
Here, $\textrm{\upshape diag}(\mathbf{K}[:, j])$ is $|\mathbf{K}| \times |\mathbf{K}|$ and $\mathbf{\Sigma}^\top$ is $|\mathbf{K}| \times h$, so $\normDec_j$ is $|\mathbf{K}| \times h$.
The full matrix $\normDec$ is a column-permutation of the concatenation $[\normDec_1, \dots, \normDec_d]$.
The column space of $\normDec$ is the sum of the column spaces of these submatrices:
\[
\textrm{\upshape col}(\normDec) = \sum_{j=1}^d \textrm{\upshape col}(\normDec_j).
\]

By the subadditivity of rank over sums of subspaces, the rank is bounded by:
\begin{align*}
    \textrm{\upshape rank}(\normDec) &\le \min_{S \subseteq [|\mathbf{K}|]} \left(\text{rank}(\normDec[\neg S, :]) + \text{rank}(\normDec[S, :]) \right)\\
    &\le \min_{S \subseteq [|\mathbf{K}|]} \left(\text{rank}(\normDec[\neg S, :]) + \sum_{j=1}^d \textrm{\upshape rank}(\normDec_j[S, :]) \right)\\
    &\le \min_{S \subseteq [|\mathbf{K}|]} \left( |\neg S| + \sum_{j=1}^d \textrm{\upshape rank}(\normDec_j[S, :]) \right)
\end{align*}
where $S$ is a set of \textit{row} indices, $\neg S$ is its complement ($|\neg S| = |\mathbf{K}| - |S|$), and $\normDec_j[S, :]$ is the submatrix of $\normDec_j$ with rows from $S$.

We now analyze $\textrm{\upshape rank}(\normDec_j[S, :])$:
\[
\normDec_j[S, :] = (\textrm{\upshape diag}(\mathbf{K}[:, j])\mathbf{\Sigma}^\top)[S, :] = \textrm{\upshape diag}(\mathbf{K}[S, j]) \cdot (\mathbf{\Sigma}^\top[S, :]).
\]
Note that $\mathbf{\Sigma}^\top[S, :] = (\mathbf{\Sigma}[:, S])^\top$.
For any rectangular matrices $\mathbf A$ and $\mathbf B$ we have\footnote{This follows by basic properties of linear maps. Let $
B : \mathbb{R}^p \to \mathbb{R}^n$, $
A : \mathbb{R}^n \to \mathbb{R}^m$, $
AB : \mathbb{R}^p \to \mathbb{R}^m.
$
Then \(\operatorname{Im}(AB) = A(\operatorname{Im}(B))\), so
$
\operatorname{rank}(AB)
= \dim\big(\operatorname{Im}(AB)\big)
= \dim\big(A(\operatorname{Im}(B))\big)
\le \dim\big(\operatorname{Im}(B)\big)
= \operatorname{rank}(B).
$} $\textrm{\upshape rank}(\mathbf{A}\mathbf{B}) \le \textrm{\upshape rank}(\mathbf{B})$. Thus:
\[
\textrm{\upshape rank}(\normDec_j[S, :]) \le \textrm{\upshape rank}((\mathbf{\Sigma}[:, S])^\top) = \textrm{\upshape rank}(\mathbf{\Sigma}[:, S]).
\]
Substituting this back into our rank bound for $\normDec$:
\[
\textrm{\upshape rank}(\normDec) \le \min_{S \subseteq [|\mathbf{K}|]} \left( (|\mathbf{K}| - |S|) + \sum_{j=1}^d \textrm{\upshape rank}(\mathbf{\Sigma}[:, S]) \right)
\]
\[
\textrm{\upshape rank}(\normDec(\mathbf{\Sigma}, \mathbf{K})) \le \min_{S\in [|\mathbf{K}|]}\Big[|\mathbf{K}| - |S| + d\cdot \textrm{\upshape rank}(\mathbf{\Sigma}[:, S])\Big] \equiv R(\mathbf{\Sigma}).
\]
This establishes $R(\mathbf{\Sigma})$ as the maximum possible rank.

\textbf{Part 3: a Zariski open set}

Let $R = R(\mathbf{\Sigma})$. From Part 2, the rank cannot exceed $R$. The set of $\mathbf{K}$ for which the rank is \emph{sub-maximal} is $\mathcal{K}^c = \{\mathbf{K} \, | \, \textrm{\upshape rank}(\normDec(\mathbf{\Sigma}, \mathbf{K})) < R\}$.
    
This condition $\textrm{\upshape rank}(\normDec(\mathbf{\Sigma}, \mathbf{K})) < R$ holds if and only if every $R \times R$ submatrix of $\normDec(\mathbf{\Sigma}, \mathbf{K})$ has a determinant equal to 0.

The entries of $\normDec(\mathbf{\Sigma}, \mathbf{K})$ are polynomial functions of the entries of $\mathbf{\Sigma}$ and $\mathbf{K}$. Since $\mathbf{\Sigma}$ is fixed, the determinant of any $R \times R$ submatrix is a polynomial in the entries (components) of $\mathbf{K}$. Let this finite set of polynomials be $\mathcal{P} = \{p_j(\mathbf{K})\}_j$.

The set $\mathcal{K}^c$ is the set of $\mathbf{K}$ that are common zeros of all polynomials in $\mathcal{P}$. By definition, this set $\mathcal{K}^c$ is an algebraic variety (a Zariski closed set). The set $\mathcal{K} = \{\mathbf{K} \, | \, \textrm{\upshape rank}(\normDec(\mathbf{\Sigma}, \mathbf{K})) = R\}$ is the complement of $\mathcal{K}^c$. As the complement of a Zariski closed set, $\mathcal{K}$ is, by definition, a Zariski open set.

An algebraic variety over $\mathbb{R}$ or $\mathbb{C}$ is either the entire space or a set of measure zero. To show $\mathcal{K}$ has full measure, it suffices to show it is non-empty (proving $\mathcal{K}^c$ is not the entire space). We construct an explicit $\mathbf{K}$ that achieves the maximum rank $R(\mathbf{\Sigma})$.

\textbf{Part 4: An explicit example}

By the Matroid Union Theorem\footnote{See Theorem 11.3.1 of \citet{oxley2011matroid}.}, the rank $R(\mathbf{\Sigma})$ is also given by:
\[
R(\mathbf{\Sigma}) = \max_{\substack{\rmI_1,\dots,\rmI_d \subseteq [|\mathbf{K}|]\\
\rmI_i \cap \rmI_j = \varnothing\ \forall i\neq j}}
\; \left[\sum_{i=1}^d \operatorname{rank}\!\big(\boldsymbol{\Sigma}[:, \rmI_i]\big)\right].
\]
Let $\rmI_1^*, \dots, \rmI_d^*$ be an optimal partition, defined as:
\[
(\rmI_1^*, \dots, \rmI_d^*) = \operatorname*{argmax}_{\substack{\rmI_1,\dots,\rmI_d \subseteq [|\mathbf{K}|]\\
\rmI_i \cap \rmI_j = \varnothing\ \forall i\neq j}}
\; \left[\sum_{i=1}^d \operatorname{rank}\!\big(\boldsymbol{\Sigma}[:, \rmI_i]\big)\right].
\]
Construct $\mathbf{K}(\rmI_1^*, \dots, \rmI_d^*) \in \mathbb{R}^{|\mathbf{K}| \times d}$ as in \Cref{def:katri_k}. Then, by \Cref{lem:katri_rao_rank_equality},
\[
    \textrm{\upshape rank}(\normDec(\mathbf{\Sigma}, \mathbf{K}(\rmI_1^*,\ldots,\rmI_d^*))) = \sum_{j=1}^d \operatorname{rank}\!\big(\boldsymbol{\Sigma}[:, \rmI_j^*]\big).
\]
This is exactly the maximal value $R(\mathbf{\Sigma})$. Since one $\mathbf{K}$ has been found for which the rank $R(\mathbf{\Sigma})$ is achieved, the set $\mathcal{K}$ is non-empty.

\end{proof}

\subsubsection{Proof of \Cref{lem:indep_func_to_full_rank_matrix}}\label{pf:indep_func_to_full_rank_matrix}
\begin{proof}
Define a map
\[
F : S \to \mathbb{R}^r, \qquad
F(\avec) := \bigl[f_1(\avec),\dots,f_r(\avec)\bigr].
\]
Then, for any choice of points $\avec^{(1)},\dots,\avec^{(r)}\in S$, the $j$-th column of the matrix
$M = \bigl(f_i(\avec^{(j)})\bigr)_{1\le i,j\le r}$ is exactly the vector $F(\avec^{(j)}) \in \mathbb{R}^r$.
Thus, it suffices to show that there exist points $\avec^{(1)},\dots,\avec^{(r)}\in S$ such that
the vectors $F(\avec^{(1)}),\dots,F(\avec^{(r)})$ are linearly independent in $\mathbb{R}^r$.

We construct such points inductively.

\medskip
\noindent\textit{Base step.}
Since $f_1,\dots,f_r$ are linearly independent as functions on $S$, not all of them are identically zero.
Hence, there exists some $\avec^{(1)}\in S$ such that
\[
F(\avec^{(1)}) = \bigl[f_1(\avec^{(1)}),\dots,f_r(\avec^{(1)})\bigr] \neq 0.
\]
Thus the single vector $F(\avec^{(1)})$ is linearly independent (as a set of size one).

\medskip
\noindent\textit{Inductive step.}
Assume that for some $k$ with $1 \le k < r$ we have already chosen points
$\avec^{(1)},\dots,\avec^{(k)}\in S$ such that
\[
F(\avec^{(1)}),\dots,F(\avec^{(k)})
\]
are linearly independent in $\mathbb{R}^r$. Let
\[
W := \operatorname{span}\{F(\avec^{(1)}),\dots,F(\avec^{(k)})\} \subset \mathbb{R}^r.
\]
Then $\dim W = k < r$, so $W$ is a proper subspace of $\mathbb{R}^r$.

We claim there exists $\avec^{(k+1)}\in S$ such that $F(\avec^{(k+1)}) \notin W$. Suppose, for contradiction,
that $F(\avec)\in W$ for all $\avec \in S$. Since $W$ is a proper subspace of $\mathbb{R}^r$, there exists
a nonzero linear functional $\ell : \mathbb{R}^r \to \mathbb{R}$ such that $\ell(v) = 0$ for all $v\in W$.
Equivalently, there exists a nonzero vector $\lambda = (\lambda_1,\dots,\lambda_r)\in \mathbb{R}^r$ such that
\[
\lambda \cdot v = 0 \quad \text{for all } v\in W.
\]

In particular, for every $\avec\in S$ we have $F(\avec)\in W$, hence
\[
0 = \lambda \cdot F(\avec)
  = \sum_{i=1}^r \lambda_i f_i(\avec).
\]
Therefore the function
\[
g := \sum_{i=1}^r \lambda_i f_i
\]
is identically zero on $S$, i.e.,
\[
g(\avec) = 0 \quad \text{for all } \avec\in S.
\]
Since $\lambda \neq 0$, this is a nontrivial linear relation among the functions $f_1,\dots,f_r$, contradicting the
assumption that they are linearly independent.

Hence our supposition was false, and there exists some $\avec^{(k+1)}\in S$ with $F(\avec^{(k+1)}) \notin W$.
Then
\[
F(\avec^{(1)}),\dots,F(\avec^{(k)}),F(\avec^{(k+1)})
\]
are linearly independent in $\mathbb{R}^r$, completing the inductive step.

By induction, we can choose points $\avec^{(1)},\dots,\avec^{(r)}\in S$ so that
the vectors $F(\avec^{(1)}),\dots,F(\avec^{(r)})$ are linearly independent in $\mathbb{R}^r$.
Equivalently, the $r\times r$ matrix
\[
M = \bigl(f_i(\avec^{(j)})\bigr)_{1\le i,j\le r}
\]
has $r$ linearly independent columns, so $\operatorname{rank}(M) = r$, and $M$ is invertible.
\end{proof}

\subsubsection{Proof of \Cref{lem:rank_determines_independence_sigma}}\label{pf:rank_determines_independence_sigma}
\begin{proof}

Since $\sigma$ is real-analytic and not a polynomial, its Taylor series at any point has infinitely many nonzero coefficients.

\medskip

\noindent\textbf{(1) The family $\{\sigma(\lambda t)\}$.}
Expand $\sigma$ at $0$:
\[
\sigma(t) = \sum_{k=0}^\infty c_k t^k
\]
with infinitely many $c_k \neq 0$. For $n \in \mathbb{N}$, define
\[
f_n(t) := \sigma(n t).
\]
We show that $\{f_n\}_{n\ge 1}$ is linearly independent.

Suppose, for some $N \ge 1$, there exist real numbers $\beta_1,\dots,\beta_N$ such that
\[
\sum_{n=1}^N \beta_n f_n(t) \equiv 0
\quad\text{as a function of $t$.}
\]
Expand using the Taylor series:
\[
0 = \sum_{n=1}^N \beta_n \sigma(n t)
= \sum_{n=1}^N \beta_n \sum_{k=0}^\infty c_k (n t)^k
= \sum_{k=0}^\infty c_k \left(\sum_{n=1}^N \beta_n n^k\right) t^k.
\]
Since two power series are equal if and only if all their coefficients are equal, we obtain
\[
c_k \left(\sum_{n=1}^N \beta_n n^k\right) = 0
\quad\text{for all }k\ge 0.
\]
For each $k$ with $c_k \neq 0$, this implies
\[
\sum_{n=1}^N \beta_n n^k = 0. \tag{$\ast$}
\]

Because there are infinitely many $k$ with $c_k \neq 0$, we have infinitely many equations $(\ast)$. Let $n_{\max}$ be the largest index with $\beta_{n_{\max}} \neq 0$. Define
\[
S(k) := \sum_{n=1}^N \beta_n n^k.
\]
Then for each such $k$,
\[
S(k) = 0.
\]

Now divide by $n_{\max}^k$:
\[
\frac{S(k)}{n_{\max}^k}
= \beta_{n_{\max}} + \sum_{n=1}^{N-1} \beta_n \left(\frac{n}{n_{\max}}\right)^k.
\]
Since $n < n_{\max}$, we have $\left|\frac{n}{n_{\max}}\right| < 1$, and so
\[
\sum_{n=1}^{N-1} \beta_n \left(\frac{n}{n_{\max}}\right)^k \xrightarrow[k\to\infty]{} 0.
\]
Thus
\[
\frac{S(k)}{n_{\max}^k} \xrightarrow[k\to\infty]{} \beta_{n_{\max}}.
\]

On the other hand, $S(k)=0$ for infinitely many $k$ (all those with $c_k \neq 0$), and these $k$ tend to infinity. Along that subsequence $k_j$, we have
\[
0 = \frac{S(k_j)}{n_{\max}^{k_j}} \xrightarrow[j\to\infty]{} \beta_{n_{\max}},
\]
so $\beta_{n_{\max}} = 0$, contradicting the definition of $n_{\max}$. Therefore all $\beta_n$ must be zero, and $\{f_n\}_{n\ge 1}$ is linearly independent. Hence the span of $\{\sigma(\lambda t)\}$ is infinite-dimensional.

\end{proof}

\subsubsection{Proof of \Cref{lem:d-rank_gives_rank}}\label{pf:d-rank_gives_rank}

\begin{proof}

Note that if $|S_1|=0$ or $|S_2|=0$, then the submatrix 
$\sigma(\mathbf{x}\mathbf{y}^\top)[S_1,S_2]$ has rank $0$, which agrees
with $\min\{|S_1|,|S_2|\}$. Thus such subsets impose no
nontrivial constraints, and we may freely ignore them in the argument below.

Define the row-restricted vectors
\[
\mathbf{x}_{S_1} := \mathbf{x}[S_1]\in\mathbb{R}^{|S_1|},\qquad
\mathbf{y}_{S_2} := \mathbf{y}[S_2]\in\mathbb{R}^{|S_2|}.
\]
Then $\sigma(\mathbf{x}\mathbf{y}^\top)[S_1,S_2] = \sigma((\mathbf{x}\mathbf{y}^\top)[S_1,S_2]) = \sigma((\mathbf{x}_{S_1}\mathbf{y}_{S_2}^\top)).$

Now, for arbitrary nonempty subsets $S_1\in[d_1]$ and  $S_2\in[d_2]$, define the map
\[
    \pi_{S_1,S_2} : \mathbb{R}^{d_1}\times\mathbb{R}^{d_2}
    \to \mathbb{R}^{|S_1|}\times\mathbb{R}^{|S_2|},
    \quad
    \pi_{S_1,S_2}(\mathbf{x},\mathbf{y}) = (\mathbf{x}_{S_1}, \mathbf{y}_{S_2}).
\]
This map is analytic and surjective.

By Lemma~\ref{lem:d-rank_gives_rank_partial}, the set
\[
\mathcal{S}^{\text{part}}_{S_1,S_2} 
:= \Bigl\{(\mathbf{x},\mathbf{y}) |\mathbf{x}\in\mathbb{R}^{|S_1|},\quad
\mathbf{y}\in\mathbb{R}^{|S_2|},\quad\operatorname{rank}\bigl(\sigma(\mathbf{x}\mathbf{y}^\top)\bigr)
= \min\{|S_1|,|S_2|\}\Bigr\}
\]
is the complement of a proper analytic subvariety of 
$\mathbb{R}^{|S_1|}\times\mathbb{R}^{|S_2|}$.

Define the corresponding full-parameter set
\[
    \mathcal{S}^{(S_1,S_2)} 
    := \pi_{S_1,S_2}^{-1}\bigl(\mathcal{S}^{\text{part}}_{S_1,S_2}\bigr)
    \subseteq \mathbb{R}^{d_1}\times\mathbb{R}^{d_2}.
\]

Let
\[
\mathcal{V}^{\text{part}}_{S_1,S_2}
:= \bigl(\mathcal{S}^{\text{part}}_{S_1,S_2}\bigr)^c
\]
denote the ``bad'' set in the smaller space (a proper analytic subvariety by
Lemma~\ref{lem:d-rank_gives_rank_partial}) and define
\[
    \mathcal{V}_{S_1,S_2}
    := \bigl(\mathcal{S}^{(S_1,S_2)}\bigr)^c
    = \pi_{S_1,S_2}^{-1}\bigl(\mathcal{V}^{\text{part}}_{S_1,S_2}\bigr).
\]
Since $\pi_{S_1,S_2}$ is analytic, the preimage of an analytic subvariety is again an analytic
subvariety, so $\mathcal{V}_{S_1,S_2}$ is an analytic subvariety of 
$\mathbb{R}^{n\times d}\times\mathbb{R}^{h\times d}$. It is \emph{proper} because
$\mathcal{V}^{\text{part}}_{S_1,S_2}$ is a proper subset and $\pi_{S_1,S_2}$ is surjective: there are
points $(\mathbf{x},\mathbf{y})$ in $\mathcal{S}^{\text{part}}_{S_1,S_2}$, and any lift of such a
point is not in $\mathcal{V}_{S_1,S_2}$.

Now define the global no-bias set
\[
\mathcal{S} := \bigcap_{\substack{S_1\subseteq [h]\\ S_2\subseteq [n]}}
\mathcal{S}^{(S_1,S_2)}.
\]

The complement of $\mathcal{S}$ is
\[
\mathcal{S}^c 
= \bigcup_{\substack{S_1\subseteq [h]\\ S_2\subseteq [n]}} \mathcal{V}_{S_1,S_2}.
\]
This is a finite union of analytic subvarieties, hence itself an analytic subvariety (see e.g., 1.2 of \citet{chirka1997}).

Finally, to see that $\mathcal{S}^c$ is \emph{proper}, it suffices to note that each 
$\mathcal{V}_{S_1,S_2}$ is a proper analytic subvariety, hence has empty interior (a nontrivial real
analytic function cannot vanish on a nonempty open set). Because the union is finite, the union also
has empty interior, and so its complement $\mathcal{S}$ is nonempty and dense. Thus $\mathcal{S}$ is
the complement of a proper analytic subvariety of 
$\mathbb{R}^{d_1}\times\mathbb{R}^{d_2}$, and it is full measure, completing the proof.
\end{proof}

\subsubsection{Proof of \Cref{lem:full_M_rank_result_analytic}}\label{pf:full_M_rank_result_analytic}

\begin{proof}
Throughout, $N:=|\mathbf{K}|$ and we assume $d\ge h$.

Define
\[
F: (\mathbf{K},\mathbf{G}) \;\longmapsto\;
\mathbf{M}\big(\sigma(\mathbf{G}\mathbf{K}^\top),\mathbf{K}\big)\in\mathbb{R}^{N\times(dh)}.
\]
Each entry of $\mathbf{G}\mathbf{K}^\top$ is a polynomial in the entries of $(\mathbf{K},\mathbf{G})$.
Since $\sigma$ is analytic, each entry of $\sigma(\mathbf{G}\mathbf{K}^\top)$ is an analytic function of
$(\mathbf{K},\mathbf{G})$. Multiplying by $\mathbf{K}$ and taking diagonals are polynomial operations,
hence every entry of $\mathbf{M}(\sigma(\mathbf{G}\mathbf{K}^\top),\mathbf{K})$ is analytic in
$(\mathbf{K},\mathbf{G})$.

Therefore, every $N\times N$ minor of $\mathbf{M}(\sigma(\mathbf{G}\mathbf{K}^\top),\mathbf{K})$ is an
analytic function of $(\mathbf{K},\mathbf{G})$. The set
\[
\mathcal{B}
:= \Bigl\{(\mathbf{K},\mathbf{G}) :
\operatorname{rank}\mathbf{M}\big(\sigma(\mathbf{G}\mathbf{K}^\top),\mathbf{K}\big) < N\Bigr\}
\]
is exactly the common zero set of all these minors, hence an \emph{analytic subvariety} of
$\mathbb{R}^{N\times d}\times\mathbb{R}^{h\times d}$.

If we can find \emph{one} parameter choice for which the corresponding matrix has full row
rank $N$, then not all $N\times N$ minors vanish identically, and the “bad” set is a \emph{proper} analytic
subvariety. Its complement is then a nonempty Zariski open set, proving the desired generic statement.

Thus, the rest of the proof is devoted to constructing such a full-row-rank example.

Define $\rmI_i = \{j\,|\, j\in[|\mathbf{K}|],\quad (i-1)h < j \le ih  \}$ for all $i\in[d]$.
Fix pairwise distinct nonzero scalars $\{\alpha_t\}_{t=1}^N$. Also, define $\vec{\alpha} = [\alpha_1,\ldots,\alpha_N].$

Finally, define $\mathbf{K}\in \mathbb{R}^{|\mathbf{K}|\times d}$ such that $\mathbf{K}[i,j] = \alpha_i \mathbbm{1}\{i\in\rmI_j\}.$ Note that each $\alpha_i$ occurs exactly once in $\mathbf{K}.$

We keep this $\mathbf{K}$ fixed from now on. We will choose $\mathbf{G}$ and $\vec{\alpha}$ to make the resulting $\mathbf{M}$ full row rank.

By \Cref{lem:katri_rao_rank_equality}, we have
\[
\operatorname{rank}\bigl(\mathbf{M}(\mathbf{\Sigma},\mathbf{K})\bigr)
= \sum_{j=1}^d \operatorname{rank}\bigl(\mathbf{\Sigma}[:,\rmI_j]\bigr),
\]
so we must simply choose $\mathbf{G}$ and $\alpha$ such that $\operatorname{rank}\bigl(\mathbf{\Sigma}[:,\rmI_j]\bigr) = |\rmI_j|$ for all $j\in[d]$.

Now,
\begin{align*}
    \mathbf{\Sigma}[:,\rmI_j] &= \sigma(\mathbf{G}\mathbf{K}^\top)[:,\rmI_j]\\
    &= \sigma(\mathbf{G}\mathbf{K}^\top[:,\rmI_j])\\
    &= \sigma(\mathbf{G}(\mathbf{K}[\rmI_j,:])^\top)\\
    &= \sigma(\mathbf{G}[:,j](\vec{\alpha}[\rmI_j])^\top)\in\mathbb{R}^{h\times |\rmI_j|}.
\end{align*}

Now, $\text{rank}[\sigma]\ge h$, by \Cref{lem:d-rank_gives_rank_partial}, $\sigma(\mathbf{G}[:,j](\vec{\alpha}[\rmI_j])^\top)$ has rank $|\rmI_j|$ for generic $\mathbf{G}[:,j]$ and $\vec{\alpha}[\rmI_j].$

Thus there exists $\mathbf{G}$ and $\vec{\alpha}$ such that $\operatorname{rank}\bigl(\mathbf{\Sigma}[:,\rmI_j]\bigr) = |\rmI_j|$ for all $j\in[d]$.

This completes the proof.
\end{proof}

\subsubsection{Proof of \Cref{lem:softmax_reformulation}}\label{pf:softmax_reformulation}
\begin{proof}
We first assume our code to be \textit{softmax-decodable} as defined in \Cref{def:EHSM-decode-exact} to prove the forward direction. For the sake of contradiction, assume there exists some $\noisyCodeMatrix[i]$, $i$, $j \neq i$ such that 
\begin{align}\label{eq:h-bound}
    \langle \deferredDecoder\noisyCodeMatrix[i],\defCompOutEmbedding_j\rangle \geq \langle \deferredDecoder\noisyCodeMatrix[i],\defCompOutEmbedding_i\rangle
\end{align}
For ease of notation, define
\begin{align*}
w &= \exp(\langle \deferredDecoder\noisyCodeMatrix[i],\defCompOutEmbedding_i\rangle), \\
z &= \exp(\langle \deferredDecoder\noisyCodeMatrix[i],\defCompOutEmbedding_j\rangle), \\
S &= \sum_{k = 1}^n \exp (\langle \deferredDecoder\noisyCodeMatrix[i],\defCompOutEmbedding_k\rangle).
\end{align*}
\Cref{eq:softmax-hypothesis} gives
\begin{align}
\left|\frac{w}{S}-1\right|<\alpha,
\qquad
\frac{z}{S}<\alpha. 
\end{align}
Since \cref{eq:softmax-hypothesis} holds for all $\frac{1}{2}>\alpha>0$, fix some $\alpha < 1/2$. From the first inequality,
\begin{align}
\frac{w}{S}>1-\alpha
\Longrightarrow
S<\frac{w}{1-\alpha}. 
\end{align}
Substituting this into the second part of (2) yields
\begin{align}
z<\alpha S<\frac{\alpha w}{1-\alpha}. 
\end{align}
Inequality (4) and our assumption \cref{eq:h-bound} implies that
\begin{align*}
w<\frac{\alpha w}{1-\alpha}
\Longrightarrow
1<\frac{\alpha}{1-\alpha}
\Longrightarrow
\alpha>\tfrac12,
\end{align*}
contradicting $\alpha<\tfrac12$.  
Therefore 
\begin{align*}
\langle \deferredDecoder \noisyCodeMatrix[i],\defCompOutEmbedding_i\rangle>\langle \deferredDecoder\noisyCodeMatrix[i],\defCompOutEmbedding_j\rangle
\end{align*}
for every $j\neq i$.
We now prove the backwards direction. 

Assume that for every index $i$
\begin{align}
\langle \deferredDecoder\noisyCodeMatrix[i],\textbf{y}_i\rangle > \langle \deferredDecoder\noisyCodeMatrix[i],\textbf{y}_j\rangle
\quad\text{for all } j\neq i. 
\end{align}
Then we show that we can handle any tolerance by scaling  $\deferredDecoder$.
For any $\noisyCodeMatrix$ and $i$ and for ease of notation define
\begin{align*}
\fixedCompressed_k &= \langle \deferredDecoder\noisyCodeMatrix[i],\textbf{y}_k\rangle,\\
g   &= \min_{j\neq i}\bigl(\fixedCompressed_i - \fixedCompressed_j\bigr).
\end{align*}

Choose $\lambda>0$ and set $\deferredDecoder_\lambda=\lambda\,\deferredDecoder$.  Define
\begin{align*}
\tilde{\fixedCompressed}_k(\lambda) &= \lambda \fixedCompressed_k,\\
p_k(\lambda) &= \frac{\exp\bigl(\tilde{\fixedCompressed}_k(\lambda)\bigr)}
                    {\sum_\ell \exp\bigl(\tilde{\fixedCompressed}_\ell(\lambda)\bigr)}.
\end{align*}

Because $\fixedCompressed_i - \fixedCompressed_j \ge g$ for every $j\neq i$,
\begin{align} \label{eq:p-lambda}
p_i(\lambda)
  &= \frac{1}{1+\sum_{j\neq i}\exp\bigl(\lambda(\fixedCompressed_j-\fixedCompressed_i)\bigr)}
  \ge \frac{1}{1+(n-1)\exp(-\lambda g)}, \\
p_j(\lambda)
  &= \frac{\exp\bigl(\lambda \fixedCompressed_j\bigr)}
         {\exp\bigl(\lambda \fixedCompressed_i\bigr)+\sum_{\ell\neq i}\exp\bigl(\lambda \fixedCompressed_\ell\bigr)}
   = \frac{\exp\bigl(-\lambda(\fixedCompressed_i-\fixedCompressed_j)\bigr)}
          {1+\sum_{\ell\neq i}\exp\bigl(-\lambda(\fixedCompressed_i-\fixedCompressed_\ell)\bigr)}
   \le \exp(-\lambda g) .
\end{align}

Given any $\alpha\in(0,1/2)$ pick
\begin{align} \label{eq:lamba-bd}
\lambda > \frac{1}{g}\ln\!\bigl((n-1)/\alpha\bigr). 
\end{align}
Then $(n-1)\exp(-\lambda g)<\alpha$ and $\exp(-\lambda g)<\alpha$, so \cref{eq:p-lambda}–\cref{eq:lamba-bd} give
\begin{align*}
p_i(\lambda) > 1-\alpha, \qquad
p_j(\lambda) < \alpha \text{ for } j\neq i.
\end{align*}
Also note that since $\exp$ has positive range and addition is monotonic over $\Z^+$, for all $i,j, \lambda$: 
\begin{align*}
    p_i(\lambda) \le 1, \quad p_j(\lambda)\ge0.
\end{align*}
Hence
\begin{align*}
\Bigl\|\operatorname{softmax}_k\bigl(\langle \deferredDecoder_\lambda\noisyCodeMatrix[i],\textbf{y}_k\rangle\bigr) - e_i\Bigr\|_\infty < \alpha .
\end{align*}

Since $\alpha$ was arbitrary, the softmax condition holds for every tolerance after scaling $\deferredDecoder$ by a suitable $\lambda$.
\end{proof}

\subsubsection{Proof of \Cref{thm:NEW_main_decoding}}\label{pf:real_NEW_main_decoding}

Fix a finite $\mathcal{P}\subset \mathbb S^{d-1}\times \mathbb S^{d-1}$ and define
\[
\mathcal{S}_\pm\;:=\;\{\decodingSphereX\pm \decodingSphereY:\ (\decodingSphereX,\decodingSphereY)\in\mathcal{P}\}.
\]
Going forward, for convenience we use the notation \[\decodeDiff_{ij}:=\compressedOutEmbedding_i-\compressedOutEmbedding_j, \quad \auxHold_i:=\compressedAuxEmbedding_i,\] define \[\widehat{\decodeDiff}_{ij}=\decodeDiff_{ij}/\|\decodeDiff_{ij}\|, \quad \widehat{\auxHold}_i=\auxHold_i/\|\auxHold_i\|.\]

We first show the following intermediate result.

\begin{lemma}\label{lem:angleJL}
Let $\Phi=\tfrac{1}{\sqrt m}\,\gaussianMatrix$ with $\gaussianMatrix\in\mathbb{R}^{m\times d}$ having i.i.d.\ $\mathcal{N}(0,1)$ entries.

Then for any $\varepsilon\in(0,1)$,
\[
\Pr\Big[\ \forall (\decodingSphereX,\decodingSphereY)\in\mathcal{P}:\ \big|\,\langle \Phi \decodingSphereX,\Phi \decodingSphereY\rangle-\langle \decodingSphereX,\decodingSphereY\rangle\,\big|\le \varepsilon\ \Big]
\ \ge\ 1\ -\ 2\,|\mathcal{S}_\pm|\,\exp\!\Big(-\tfrac{\varepsilon^2}{8}\,m\Big).
\]
Equivalently, it suffices that
\begin{align} \label{eq:m-lower-bound}
m\ \ge\ \frac{8}{\varepsilon^2}\,\ln\!\Big(\frac{2\,|\mathcal{S}_\pm|}{\delta}\Big)
\end{align}
to ensure the event above holds with probability at least $1-\delta$.
\end{lemma}
\begin{proof}
    See \Cref{pf:angleJL}
\end{proof}

\begin{corollary}
\label{cor:bilinear_ex212}
Let $\isoErr:=\Phi^\top\Phi-\rmI$ with $\Phi=\tfrac{1}{\sqrt m}\gaussianMatrix$ and $\gaussianMatrix$ i.i.d.\ standard Gaussian.
If \Cref{eq:m-lower-bound} holds, then for
\[
\mathcal{P}=\{(\widehat{\decodeDiff}_{ij},\widehat{\auxHold}_i):\ i\in[{\numVectors}],\ j\neq i\}
\]
it follows that
\[
\mathcal{S}_\pm=\{\widehat{\decodeDiff}_{ij}\pm \widehat{\auxHold}_i\},\ \ |\mathcal{S}_\pm|\le 2\numVectors(\numVectors-1),
\]
we have, simultaneously for all $i\neq j$,
\[
\big|\,\decodeDiff_{ij}^\top \isoErr\,\auxHold_i\,\big|
=\|\decodeDiff_{ij}\|\,\|\auxHold_i\|\cdot
\big|\,\langle \Phi\widehat{\decodeDiff}_{ij},\Phi\widehat{\auxHold}_i\rangle-\langle \widehat{\decodeDiff}_{ij},\widehat{\auxHold}_i\rangle\,\big|
\ \le\ \varepsilon\,\|\decodeDiff_{ij}\|\,\|\auxHold_i\|.
\]
\end{corollary}

\begin{proof}
This follows directly from \Cref{lem:angleJL}.

Equipped with these results, the proof of the theorem is relatively concise. 

Define $\vs_{ij} = \langle \compressedOutEmbedding_j, \rmM \codeMatrix[i] \rangle = \langle \compressedOutEmbedding_j, \frac{1}{m} \gaussianMatrix^\top \gaussianMatrix \compressedAuxEmbedding_i \rangle$. Apply \Cref{cor:bilinear_ex212} with $\varepsilon=\rho_{\min}/2$ to the family
$\mathcal{P}=\{(\widehat{\decodeDiff}_{ij},\widehat{\auxHold_i})\}$.
By \Cref{cor:bilinear_ex212},
\(
|\decodeDiff_{ij}^\top \isoErr \auxHold_i|\le (\rho_{\min}/2)\,\|\decodeDiff_{ij}\|\,\|\auxHold_i\|,
\)
where $\isoErr$ is the same as in \Cref{cor:bilinear_ex212}. We then have $\vs_{ii} - \vs_{ij} = \langle \compressedOutEmbedding_i - \compressedOutEmbedding_j, \frac{1}{m} \gaussianMatrix^\top \gaussianMatrix \compressedAuxEmbedding_i \rangle  = \langle \decodeDiff_{ij}, (\rmI + \isoErr) \auxHold_i \rangle = \langle \decodeDiff_{ij}, \auxHold_i \rangle + \decodeDiff_{i,j}^\top \isoErr \auxHold_i$. By definition of $\rho_{\min}$,
\(
\langle \decodeDiff_{ij},\auxHold_i\rangle\ge \rho_{\min}\,\|\decodeDiff_{ij}\|\,\|\auxHold_i\|.
\)
Therefore each gap satisfies
\[
\vs_{ii}-\vs_{ij}
=\langle \decodeDiff_{ij},\auxHold_i\rangle+\decodeDiff_{ij}^\top \isoErr \auxHold_i
\ \ge\ \rho_{\min}\,\|\decodeDiff_{ij}\|\,\|\auxHold_i\| - (\rho_{\min}/2)\,\|\decodeDiff_{ij}\|\,\|\auxHold_i\|
\ =\ (\rho_{\min}/2)\,\|\decodeDiff_{ij}\|\,\|\auxHold_i\|\ >0,
\]
simultaneously for all $i\neq j$ on the high-probability event.
To make this event have probability at least $1-\delta$,
Lemma~\ref{lem:angleJL} requires
\(
m\ge \tfrac{8}{(\rho_{\min}/2)^2}\ln\!\big(2|\mathcal{S}_\pm|/\delta\big)
\)
Substituting in $|\mathcal{S}_\pm|\le 2\numVectors(\numVectors-1)$, which follows from the number of elements in $\mathcal{P}$, provides the stated condition.

\end{proof}

\subsubsection{Proof of \Cref{lem:angleJL}}\label{pf:angleJL}
\begin{proof}
    For any fixed $\fixedCompressed\in\mathbb R^{d}$ we have
\[
\|\Phi \fixedCompressed\|_2^2 \;=\; \frac{1}{m}\,\|\gaussianMatrix \fixedCompressed\|_2^2 \ \sim \ \|\fixedCompressed\|_2^2\cdot\frac{\chi_m^2}{m}.
\]
 This fact and the following $\chi^2$ tail bound are well known results. For instance, see Example 2.12 of \citep{wainwright2019high}. Remember that $\chi^2_m \sim \text{Gamma}(\alpha=\frac{m}{2},\theta=2)$. We then have from a classic $\chi^2$ tail bound for any $0<\varepsilon<1$ and any fixed $\fixedCompressed\neq \bm{0}$,
\[
\Pr\!\left[\left|\frac{\|\Phi z\|_2^2}{\|\fixedCompressed\|_2^2}-1\right|\ge \varepsilon\right]\ \le\ 2\,\exp\!\Big(-\tfrac{\varepsilon^2}{8}\,m\Big).
\]
Equivalently,
\[
\Pr\Big[\big|\|\Phi z\|_2^2-\|\fixedCompressed\|_2^2\big|>\varepsilon\,\|\fixedCompressed\|_2^2\Big]\ \le\ 2\,\exp\!\Big(-\tfrac{\varepsilon^2}{8}\,m\Big).
\]

Then for any $(\sphereX,\sphereY)\in\mathbb S^{d-1}\times\mathbb S^{d-1}$,
\[
\langle \Phi\sphereX, \Phi\sphereY\rangle-\langle \sphereX,\sphereY\rangle
=\frac14\Big(\|\Phi(\sphereX+\sphereY)\|_2^2-\|\sphereX+\sphereY\|_2^2\Big)
-\frac14\Big(\|\Phi(\sphereX-\sphereY)\|_2^2-\|\sphereX-\sphereY\|_2^2\Big).
\]
If simultaneously
\[
\big|\|\Phi(\sphereX+\sphereY)\|_2^2-\|\sphereX+\sphereY\|_2^2\big|\le \varepsilon\,\|\sphereX+\sphereY\|_2^2,
\qquad
\big|\|\Phi(\sphereX-\sphereY)\|_2^2-\|\sphereX-\sphereY\|_2^2\big|\le \varepsilon\,\|\sphereX-\sphereY\|_2^2,
\]
then, using $\|\sphereX\|=\|\sphereY\|=1$,
\[
\big|\,\langle \Phi \sphereX,\Phi \sphereY\rangle-\langle \sphereX,\sphereY\rangle\,\big|
\le \frac{\varepsilon}{4}\big(\|\sphereX+\sphereY\|_2^2+\|\sphereX-\sphereY\|_2^2\big)
= \frac{\varepsilon}{4}\big(2\|\sphereX\|_2^2+2\|\sphereY\|_2^2\big)=\varepsilon.
\]

Let $A_z$ denote the event that $\big|\|\Phi \inSPM\|_2^2-\|\inSPM\|_2^2\big|>\varepsilon\,\|\inSPM\|_2^2$ for a fixed $\inSPM\in\mathcal S_\pm$. Then
$\Pr[A_z]\le 2e^{-(\varepsilon^2/8)m}$.
If none of the events $\{A_z\}_{z\in\mathcal S_\pm}$ occur, the bound in the previous step holds for all $(\sphereX,\sphereY)\in\mathcal P$.
Therefore,
\[
\Pr\Big[\exists (\sphereX,\sphereY)\in\mathcal P:\ \big|\,\langle \Phi \sphereX,\Phi \sphereY\rangle-\langle \sphereX,\sphereY\rangle\,\big|>\varepsilon\Big]
\ \le\ \sum_{z\in\mathcal S_\pm}\Pr[A_z]
\ \le\ 2\,|\mathcal S_\pm|\,\exp\!\Big(-\tfrac{\varepsilon^2}{8}\,m\Big),
\]
upon union bounding over all $(\sphereX,\sphereY)\in\mathcal S_{\pm}$ which proves the claim.
\end{proof}

\subsubsection{Proof of \Cref{thm:uniform_decoding_rho}}\label{pf:NEW_main_decoding}
\begin{proof}
From our definition of $\rho_{\min}$ (recall that $\va_{i,j} = \mathbf{\tilde v}_i - \mathbf{\tilde v}_j$ and $\bvec_i = \mathbf{\tilde u}_i$)
\[
\rho_{\min} \;=\; \min_{i\neq j}\frac{\ip{\va_{ij}}{\bvec_i}}{\norm{\va_{ij}}\norm{\bvec_i}}
= \min_{i\neq j}\frac{\ip{\defCompOutEmbedding_i-\defCompOutEmbedding_j}{\defCompOutEmbedding_i}}{\norm{\defCompOutEmbedding_i-\defCompOutEmbedding_j}}
= \min_{i\neq j}\sqrt{\frac{1-\ip{\defCompOutEmbedding_i}{\defCompOutEmbedding_j}}{2}}.
\]
Note that $||\tilde{\vu}_i||=1.$\\
Let $\maxDifference:=\max_{i<j}\ip{\defCompOutEmbedding_i}{\defCompOutEmbedding_j}$; since the map $\sphereX\mapsto\sqrt{(1-\sphereX)/2}$ is decreasing on $(-1,1)$,
\begin{align} \label{eq:rho-bd}
\rho_{\min} \;\ge\; \sqrt{\frac{1-\maxDifference}{2}}.
\end{align}

To control $\maxDifference$, fix $\va\in\Sph$ and let $\rmX\sim\mathrm{Unif}(\Sph)$.  
The function $f(\sphereX)=\ip{\sphereX}{\va}$ is $1$-Lipschitz on $\Sph$ (geodesic metric) and $\E[f]=0$ by symmetry.  
By Theorem~3 of \citep{aubrun2024optimalconstantsconcentrationinequalities}, for all $\tBound>0$,
\begin{align} \label{eq: prob-t-bd}
\Pr\{\ip{\rmX}{\va}\ge \tBound\}\ \le\ e^{-d\,\tBound^2/2}.
\end{align}
Conditioning on $\defCompOutEmbedding_j$ and applying \cref{eq: prob-t-bd} with $\rmX=\defCompOutEmbedding_i$, $\va=\defCompOutEmbedding_j$ yields, for each unordered pair $\{i,j\}$,
\(
\Pr\{\ip{\defCompOutEmbedding_i}{\defCompOutEmbedding_j}\ge \tBound\}\le e^{-d\,\tBound^2/2}.
\)
Union-bounding over the $\binom{\decodingNumVectors}{2}$ pairs gives
\[
\Pr\{\maxDifference\ge \tBound\}\ \le\ \binom{\decodingNumVectors}{2}\,e^{-d\,\tBound^{2}/2}.
\]
Hence with probability at least $1-\probDelta$,
\begin{align}\label{eq:mu-bd}
\maxDifference\ \le\ \sqrt{\frac{2}{d}\,\ln\!\frac{\binom{\decodingNumVectors}{2}}{\probDelta}}
\end{align}

Combining \ref{eq:rho-bd}–\ref{eq:mu-bd} yields the stated bound.
\end{proof}

\subsubsection{Proof of \Cref{thm:subgaussian_decoding_bd}}\label{pf:subgaussian_decoding_bd}
\begin{proof}
Let $\normXi_{ik}:=\sqrt d\,\xi_{ik}$. Then $\|\normXi_{ik}\|_{\psitwo}\le K$ and $\mathbb{E}[\normXi_{ik}^2]=1$. Note that we also have\footnote{This is well known. For instance, it follows directly from Lemma 2.8.6 of \citep{vershynin2018high}} $\norm{\normXi^2}_{\psi_1} \leq \norm{\normXi}_{\psi_2}^2 \leq K^2$ From the definition of the sub-exponential norm\footnote{Here we use the usual definition $\norm{\rmX}_{\psi_1} := \inf\{\tBound > 0: \E[\exp(|\rmX| / \tBound)] \leq 2\}$} we have that $\|1\|_{\psione}=1/\ln 2$, so
\[
\|\normXi_{ik}^2-1\|_{\psione}\ \le\ \|\normXi_{ik}^2\|_{\psione}+\|1\|_{\psione}\ \le\ K^2+\frac{1}{\ln 2}
\]
Since
\(
\|\defCompOutEmbedding_i\|^2-1=\frac1d\sum_{k=1}^d(\normXi_{ik}^2-1),
\)
we apply the Bernstein bound for sub-exponentials \footnote{See Theorem 1.2.7 of \citep{chafai2012interactions}. This text uses the slightly different Orcliz norm $\norm{\rmX}_{\psi_1}^{(e)} = \inf \{c > 0: \psi(|\rmX| / c) \leq \psi(1) \}$
where $\psi$ is some Orcliz function. Recall that our definition has been $\norm{X}_{\psi_1} = \inf \{c > 0: \exp(|X| / c) \leq 2\}$. Fortunately, if we set $\psi_1 (\xvec) = \exp(|\xvec|^\alpha) - 1$ it follows that $\{c > 0: \exp(|\rmX| / c) \leq 2\} \subseteq \{c > 0: \exp(|\rmX| / c) \leq e \}$, and after taking infimums we have $\norm{\rmX}_{\psi_1}^{(e)} \leq \norm{\rmX}_{\psi_1}$. So we may use the bound as if it were our familiar norm. 
} to find, for all $\eta>0$,
\[
\Pr\!\big(|\|\defCompOutEmbedding_i\|^2-1|\ge \eta\big)\le
2\exp\!\Big(-c_B d\,\min\!\Big\{\frac{\eta^2}{(K^2+\frac{1}{\ln 2})^2},\ \frac{\eta}{K^2+\frac{1}{\ln 2}}\Big\}\Big).
\]
Union bound over $i\in[\decodingNumVectors]$\footnote{To find $\varepsilon_{\decodingNumVectors}$, take the right hand side of the above equation and set it less than or equal to $\probDelta / 2$. Solving for $\eta$ yields $\varepsilon_{\decodingNumVectors}$.}. Using $|\sqrt{1+u}-1|\le |u|$ ($u>-1$), with probability $\ge 1-\probDelta/2$,
\[
\big|\|\defCompOutEmbedding_i\|-1\big|\ \le\ \varepsilon_{\decodingNumVectors}\qquad\text{for all }i,
\quad
\varepsilon_{\decodingNumVectors}:=(K^2+\frac{1}{\ln 2})\, \max \left(\sqrt{\frac{1}{c_B\,d}\,\ln\!\frac{4\decodingNumVectors}{\probDelta}}\, ,  \frac{1}{c_B\,d}\,\ln\!\frac{4\decodingNumVectors}{\probDelta}\right)
\]

We now find a bound for $\langle \defCompOutEmbedding_j, \vu_i \rangle $. Condition on $\vu_i$. Then for $j\ne i$,
\[
\ip{\defCompOutEmbedding_j}{\vu_i}=\sum_{k=1}^d \vu_{ik}\,\xi_{jk}
\]
is a sum of independent centered subgaussians with $\|\;\vu_{ik}\xi_{jk}\;\|_{\psitwo}\le |\vu_{ik}|\,K/\sqrt d$. By Theorem~1.1 of \citep{leskelä2025sharpconstantsrelatingsubgaussian}, the corresponding variance proxies are
$\sigma_k^2=(\sqrt{\ln 2}\,K\,|\vu_{ik}|/\sqrt d)^2$. The Hoeffding bound for sub-gaussians\footnote{See Proposition 2.5 of \citep{wainwright2019high}} gives for any $\tBound\ge 0$,
\[
\Pr\!\Big(|\ip{\defCompOutEmbedding_j}{\vu_i}|\ge t\ \Big|\ \vu_i\Big)
\ \le\ 2\exp\!\left(-\frac{t^2}{2\sum_k\sigma_k^2}\right)
\ =\ 2\exp\!\left(-\frac{t^2}{2(\ln 2)\,K^2/d}\right),
\]
since $\sum_k \vu_{ik}^2=1$. Removing the conditioning and union-bounding over ordered pairs $(i,j)$ shows that, with probability $\ge 1-\probDelta/2$,
\[
|\ip{\defCompOutEmbedding_j}{\vu_i}|\ \le\ t_{\decodingNumVectors}\qquad\text{for all }i\ne j,
\quad
t_{\decodingNumVectors}:=K\,\sqrt{\frac{2\ln 2}{d}\,\ln\!\frac{4\decodingNumVectors(\decodingNumVectors-1)}{\probDelta}}\ .
\]

On the intersection of the two events (probability $\ge 1-\probDelta$), for every $i\ne j$,
\[
\ip{\defCompOutEmbedding_i-\defCompOutEmbedding_j}{\vu_i}=\|\defCompOutEmbedding_i\|-\ip{\defCompOutEmbedding_j}{\vu_i}\ \ge\ 1-\varepsilon_{\decodingNumVectors}-t_{\decodingNumVectors},
\qquad
\|\defCompOutEmbedding_i-\defCompOutEmbedding_j\|\ \le\ \|\defCompOutEmbedding_i\|+\|\defCompOutEmbedding_j\|\ \le\ 2(1+\varepsilon_{\decodingNumVectors}).
\]
Therefore $\displaystyle (\rho_{\min})_{ij}\ge \frac{1-\varepsilon_{\decodingNumVectors}-t_{\decodingNumVectors}}{2(1+\varepsilon_{\decodingNumVectors})}$, and taking the minimum over $i\ne j$ yields the claim.    
\end{proof}

\begin{theorem}[Noisy decoding via JL, Rademacher case]
Let $\decMat \in \{-1,+1\}^{m\times d}$ have i.i.d.\ Rademacher entries
($\Pr(\decMat_{kl}=1)=\Pr(\decMat_{kl}=-1)=\tfrac{1}{2}$) and set $\normDec := \frac{1}{m}\decMat^\top$.
For each $i \in [N]$, let $\vvec_i,\uvec_i \in \R^d$ and define
\[
\rho_{\min}
:= \min_{i\neq j}
   \frac{\langle \vvec_i - \vvec_j,\;\uvec_i\rangle}
        {\|\vvec_i - \vvec_j\|\,\|\uvec_i\|} \;>\; 0.
\]
Let the noisy codes be
\[
\rmH[i] := (\decMat \uvec_i)\odot(1+\nu_i),\qquad
\nu_i \in [-\varepsilon,\varepsilon]^m,\quad \varepsilon\in[0,1),
\]
and define scores $s_{ij} := \langle \vvec_j,\;\normDec \rmH[i]\rangle$.
Then there is a universal constant $C>0$ such that if
\[
m \;\ge\; \frac{C}{\rho_{\min}^2}\,
          \ln\!\frac{4N(N-1)}{\delta},
\]
then with probability at least $1-\delta$ over $\decMat$, we have, simultaneously for all $i\neq j$,
\[
s_{ii} - s_{ij}
\;\ge\;
\Big(\frac{\rho_{\min}}{2} - 4\varepsilon\Big)\,
\|\vvec_i - \vvec_j\|\,\|\uvec_i\|.
\]
\end{theorem}

\begin{proof}
Set $\Phi := \tfrac{1}{\sqrt m}\decMat$ and $\rmE := \Phi^\top\Phi - \rmI$.
For $i\neq j$, write
\[
\avec_{ij} := \vvec_i - \vvec_j,\qquad \bvec_i := \uvec_i.
\]
Let $\gvec_i := \decMat \uvec_i$ and $\Delta_i := \gvec_i\odot\nu_i$, so $\rmH[i] = \gvec_i + \Delta_i$.
Then
\[
\normDec \rmH[i]
= \tfrac{1}{m}\decMat^\top(\gvec_i + \Delta_i)
= \Phi^\top\Phi\,\bvec_i + \tfrac{1}{m}\decMat^\top\Delta_i
= (\rmI+\rmE)\bvec_i + \tfrac{1}{m}\decMat^\top\Delta_i,
\]
and the score gap is
\begin{equation}\label{eq:gap}
s_{ii} - s_{ij}
= \langle \avec_{ij}, \normDec \rmH[i]\rangle
= \avec_{ij}^\top \bvec_i
 + \avec_{ij}^\top \rmE \bvec_i
 + \frac{1}{m}(\decMat \avec_{ij})^\top\Delta_i.
\end{equation}

\paragraph{Margin term.}
By the definition of $\rho_{\min}$,
\begin{equation}\label{eq:margin}
\avec_{ij}^\top \bvec_i
= \langle \vvec_i - \vvec_j, \uvec_i\rangle
\;\ge\; \rho_{\min}\,\|\avec_{ij}\|\,\|\bvec_i\|
\quad\forall i\neq j.
\end{equation}

\paragraph{JL event (inner products and norms).}
Define
\[
\hat \avec_{ij} := \frac{\avec_{ij}}{\|\avec_{ij}\|},\qquad
\hat \bvec_i := \frac{\bvec_i}{\|\bvec_i\|},
\]
and consider the finite set of unit-vector pairs
\[
\mathcal P :=
\{(\hat \avec_{ij},\hat \bvec_i) : i\in[N],\, j\neq i\}
\;\cup\;
\{(\hat \xvec,\hat \xvec) : \xvec\in X\},
\]
where $X := \{\avec_{ij}: i\neq j\}\cup\{\bvec_i : i\in[N]\}$.
Since the rows of $\Phi$ are isotropic subgaussian (Rademacher),
the Johnson--Lindenstrauss lemma implies:

for $\eta := \rho_{\min}/2$, if
\[
m \;\ge\; \frac{C}{\rho_{\min}^2}\,
          \ln\!\frac{4N(N-1)}{\delta},
\]
then with probability at least $1-\delta$,
\[
\big|\langle\Phi \xvec,\Phi \yvec\rangle - \langle \xvec,\yvec\rangle\big|
\le \eta
\quad\forall (\xvec,\yvec)\in\mathcal P.
\]
Following  from \Cref{jl-ip-result}.

On this event, we get:

(i) For $(\xvec,\yvec)=(\hat \avec_{ij},\hat \bvec_i)$,
\[
|\hat \avec_{ij}^\top \rmE\hat \bvec_i|
= \big|\langle\Phi\hat \avec_{ij},\Phi\hat \bvec_i\rangle
     - \langle\hat \avec_{ij},\hat \bvec_i\rangle\big|
\le \frac{\rho_{\min}}{2},
\]
so
\begin{equation}\label{eq:E}
|\avec_{ij}^\top \rmE \bvec_i|
\le \frac{\rho_{\min}}{2}\,\|\avec_{ij}\|\,\|\bvec_i\|
\quad\forall i\neq j.
\end{equation}

(ii) For $(\xvec,\yvec)=(\hat \xvec,\hat \xvec)$,
\[
|\|\Phi\hat \xvec\|^2 - 1|
= \big|\langle\Phi\hat \xvec,\Phi\hat \xvec\rangle - 1\big|
\le \frac{\rho_{\min}}{2}
\le 1,
\]
so $\|\Phi\hat \xvec\|\le \sqrt{2} \le 2$ and hence
\begin{equation}\label{eq:norm}
\|\decMat \xvec\| = \sqrt m\,\|\Phi \xvec/\|\xvec\|\| \cdot \|\xvec\|
\le 2\sqrt m\,\|\xvec\|
\quad\forall \xvec\in X.
\end{equation}

\paragraph{Noise term.}
Since $|\nu_{i,k}|\le\varepsilon$, we have
\[
|\Delta_{i,k}| = |\gvec_{i,k}\nu_{i,k}|
\le \varepsilon\,|\gvec_{i,k}|,
\quad\Rightarrow\quad
\|\Delta_i\|\le\varepsilon\,\|\gvec_i\|
= \varepsilon\,\|\decMat \bvec_i\|.
\]
By Cauchy--Schwarz and \eqref{eq:norm},
\[
\big|(\decMat \avec_{ij})^\top\Delta_i\big|
\le \|\decMat \avec_{ij}\|\,\|\Delta_i\|
\le \varepsilon\,\|\decMat \avec_{ij}\|\,\|\decMat \bvec_i\|
\le \varepsilon\,(2\sqrt m\,\|\avec_{ij}\|)(2\sqrt m\,\|\bvec_i\|),
\]
so
\begin{equation}\label{eq:noise}
\Big|\frac{1}{m}(\decMat \avec_{ij})^\top\Delta_i\Big|
\le 4\varepsilon\,\|\avec_{ij}\|\,\|\bvec_i\|
\quad\forall i\neq j.
\end{equation}

\paragraph{Conclusion.}
Conditioning on the JL event, combining \eqref{eq:margin}, \eqref{eq:E}, and
\eqref{eq:noise} in \eqref{eq:gap} gives, for all $i\neq j$,
\[
\begin{aligned}
s_{ii} - s_{ij}
&\ge
\rho_{\min}\,\|\avec_{ij}\|\,\|\bvec_i\|
 - \frac{\rho_{\min}}{2}\,\|\avec_{ij}\|\,\|\bvec_i\|
 - 4\varepsilon\,\|\avec_{ij}\|\,\|\bvec_i\| \\
&=
\Big(\frac{\rho_{\min}}{2} - 4\varepsilon\Big)\,
\|\avec_{ij}\|\,\|\bvec_i\|.
\end{aligned}
\]
Since $\avec_{ij} = \vvec_i - \vvec_j$ and $\bvec_i = \uvec_i$, this is exactly
\[
s_{ii} - s_{ij}
\;\ge\;
\Big(\frac{\rho_{\min}}{2} - 4\varepsilon\Big)\,
\|\vvec_i - \vvec_j\|\,\|\uvec_i\|,
\]
as claimed.
\end{proof}

\newpage

\begin{theorem}[Polynomial precision for encoder parameters]
Let $F$ be the number of facts, and assume the noisy decoding theorem
above holds for some choice of $m$ (so that, for any codes whose noise
is at most a fixed constant multiple of $\rho_{\min}$, decoding is
still correct).

Assume the following polynomial bounds:
\begin{enumerate}[label=(\roman*)] 
    \item (Margin) $\rho_{\min} \;\ge\; 1/\poly(F)$.
    \item (Lipschitz in parameters) For each key $k_i$ and all
    encoder parameter vectors $\theta,\theta'$,
    \[
        \|\enc_\theta(k_i) - \enc_{\theta'}(k_i)\|
        \;\le\;
        L(F)\,\|\theta-\theta'\|
        \quad\text{with } L(F)\le\poly(F).
    \]
    \item (Parameter count) The number of encoder parameters satisfies
    $P \le \poly(F)$.
    \item (Magnitude) There is an encoder $\theta_\star$ such that
    $H_\star[i] := \enc_{\theta_\star}(k_i) = \decMat \uvec_i$ and
    $\|\theta_\star\|_\infty \le \poly(F)$.
\end{enumerate}
Then there exists a constant $c>0$ such that if we quantize each
coordinate of $\theta_\star$ to the grid $F^{-c}\Z$, obtaining
$\tilde\theta$, the corresponding codes
$\tilde H[i] := \enc_{\tilde\theta}(k_i)$ still satisfy the conditions
of the noisy decoding theorem and hence decode all $F$ facts
correctly. In particular, each encoder parameter requires only
$O(\log F)$ bits of precision.
\end{theorem}

\begin{proof}
\textbf{Step 1: Allowed code noise.}
From the noisy decoding theorem, there is a constant $c_0>0$ such that,
if the code for fact $i$ is perturbed by at most $c_0\rho_{\min}$ in
an appropriate sense (as in the theorem’s proof), then the score margin
remains positive:
\[
    s_{ii} - s_{ij} \;\ge\; \Omega(\rho_{\min})\,\|\vvec_i-\vvec_j\|\,\|\uvec_i\|.
\]
Thus the encoder codes are robust to perturbations of size
$\Theta(\rho_{\min})$. Using (i), we have
\[
    \rho_{\min} \;\ge\; \frac{1}{\poly(F)},
\]
so the allowed code noise is at least $1/\poly(F)$.

\medskip
\noindent
\textbf{Step 2: From parameter perturbation to code perturbation.}
Let $\theta_\star$ be the ideal encoder parameters and
$\tilde\theta$ any other parameter vector. For each key $k_i$,
define the code perturbation
\[
    \Delta_i := \enc_{\tilde\theta}(k_i) - \enc_{\theta_\star}(k_i).
\]
By the Lipschitz assumption (ii),
\[
    \|\Delta_i\|
    =
    \|\enc_{\tilde\theta}(k_i) - \enc_{\theta_\star}(k_i)\|
    \;\le\;
    L(F)\,\|\tilde\theta - \theta_\star\|
    \quad\forall i.
\]
To keep the codes within the robustness radius from Step~1, it suffices
to impose
\[
    \|\Delta_i\| \;\le\; c_0\rho_{\min}
    \quad\forall i.
\]
A sufficient condition is therefore
\[
    \|\tilde\theta - \theta_\star\|
    \;\le\;
    \delta(F)
    := \frac{c_0\rho_{\min}}{L(F)}.
\]
Using (i) and (ii), we obtain
\[
    \delta(F)
    \;\ge\;
    \frac{c_0}{\poly(F)\,\poly(F)}
    \;=\;
    \frac{1}{\poly(F)}.
\]
So there is a ball of radius at least $1/\poly(F)$ around $\theta_\star$
in parameter space such that any $\tilde\theta$ in this ball produces
codes that the noisy decoding theorem can tolerate.

\medskip
\noindent
\textbf{Step 3: Quantization and choice of grid size.}
Now quantize each coordinate of $\theta_\star$ to a grid of step size
$\Delta>0$, obtaining $\tilde\theta$. Each coordinate changes by at most
$\Delta/2$, so
\[
    \|\tilde\theta - \theta_\star\|_2
    \;\le\;
    \sqrt{P}\,\frac{\Delta}{2}.
\]
To guarantee $\|\tilde\theta - \theta_\star\|\le\delta(F)$, it is enough
to choose $\Delta$ so that
\[
    \sqrt{P}\,\frac{\Delta}{2}
    \;\le\;
    \delta(F)
    \quad\Longleftrightarrow\quad
    \Delta \;\le\; \frac{2\,\delta(F)}{\sqrt{P}}.
\]
Using $\delta(F)\ge 1/\poly(F)$ and $P\le\poly(F)$ from (iii), we get
\[
    \frac{2\,\delta(F)}{\sqrt{P}}
    \;\ge\;
    \frac{1}{\poly(F)}.
\]
Thus the admissible step size $\Delta$ can be as large as
$1/\poly(F)$. In particular, we may pick
\[
    \Delta := F^{-c}
\]
for some constant $c>0$ large enough so that $\Delta\le 2\delta(F)/\sqrt{P}$.
This ensures $\|\tilde\theta - \theta_\star\|\le\delta(F)$ and, by
Step~2, that the induced code perturbations are within the noise budget
of the noisy decoding theorem. Hence decoding remains correct.

\medskip
\noindent
\textbf{Step 4: Bit complexity.}
By (iv), each parameter lies in an interval of length at most
$\text{range} \le 2\,\poly(F)$. With grid spacing
$\Delta = F^{-c} = 1/\poly(F)$, the number of representable levels per
parameter is at most
\[
    \frac{\text{range}}{\Delta}
    \;\le\;
    \frac{\poly(F)}{1/\poly(F)}
    = \poly(F).
\]
Therefore the number of bits per parameter is
\[
    \log_2\Big(\frac{\text{range}}{\Delta}\Big)
    = O(\log \poly(F))
    = O(\log F).
\]
This proves that encoder parameters require only $O(\log F)$ bits of
precision.
\end{proof}

Note that the last part (assumption 4) is true because $\sigma$ is analytic, which implies that it is continuously differentiable. 

\begin{theorem}[Encoder is Lipschitz in the parameters]
Fix a number of facts $F$ and keys $\{k_i\}_{i=1}^F \subset \R^d$.
Consider the scalar-output gated encoder
\[
\enc_\theta(\xvec)
= \mathbf{1}_h^\top\big[\sigma(\rmG\xvec)\odot(\rmA\xvec)\big]
= \sum_{r=1}^h \sigma(\langle \gvec_r,\xvec\rangle)\,\langle \rva_r,\xvec\rangle,
\]
where $\rmA,\rmG\in\R^{h\times d}$ have rows $\rva_r^\top,\gvec_r^\top$, and $\theta\in\R^P$ is the vector of
all entries of $\rmA,\rmG$.

Assume:
\begin{enumerate}[label=(\roman*)]
    \item $\|k_i\|_2 \le R_\xvec(F)$ for all $i$, with $R_\xvec(F)\le \poly(F)$.
    \item $\|\theta\|_2 \le R_\theta(F)$, with $R_\theta(F)\le \poly(F)$.
    \item The width and input dimension satisfy $h,d \le \poly(F)$, so that $P = 2hd \le \poly(F)$.
    \item The activation $\sigma:\R\to\R$ is continuously differentiable and on the interval
    $[-B(F),B(F)]$ with $B(F):=R_\theta(F)R_\xvec(F)$ we have
    \[
      |\sigma(t)| \le C_\sigma,\qquad |\sigma'(t)| \le C'_\sigma
      \quad\forall t\in[-B(F),B(F)],
    \]
    for some constants $C_\sigma,C'_\sigma$ independent of $F$.
\end{enumerate}
Then for each key $k_i$ there exists a constant $L(F)\le \poly(F)$ such that for all
parameter vectors $\theta,\theta'$ with $\|\theta\|_2,\|\theta'\|_2\le R_\theta(F)$,
\[
  |\enc_\theta(k_i) - \enc_{\theta'}(k_i)|
  \;\le\;
  L(F)\,\|\theta-\theta'\|_2.
\]
In particular, $\enc_\theta(k_i)$ is Lipschitz in $\theta$ with Lipschitz constant at most
polynomial in $F$.
\end{theorem}

\begin{proof}
Fix $i$ and write $\xvec := k_i$. For fixed $\xvec$, view $\enc_\theta(\xvec)$ as a function
$\R^P\to\R$ of the parameter vector $\theta$. Its partial derivatives are, for each
$r\in[h]$ and $\ell\in[d]$,
\[
  \frac{\partial \enc_\theta(\xvec)}{\partial \rmA_{r\ell}}
  = \sigma(\langle \gvec_r,\xvec\rangle)\,\xvec_\ell,
  \qquad
  \frac{\partial \enc_\theta(\xvec)}{\partial \rmG_{r\ell}}
  = \sigma'(\langle \gvec_r,\xvec\rangle)\,\langle \rva_r,\xvec\rangle\,\xvec_\ell.
\]
On the parameter ball $\|\theta\|_2\le R_\theta(F)$ and with $\|\xvec\|\le R_\xvec(F)$ we have
$|\langle \gvec_r,\xvec\rangle| \le \|\gvec_r\|\|\xvec\|\le R_\theta(F)R_\xvec(F) = B(F)$, so by assumption
$|\sigma(\langle \gvec_r,\xvec\rangle)| \le C_\sigma$ and
$|\sigma'(\langle \gvec_r,\xvec\rangle)| \le C'_\sigma$.
Moreover $|\xvec_\ell|\le R_\xvec(F)$ and
\[
  |\langle \rva_r,\xvec\rangle|
  \le \|\rva_r\|\,\|\xvec\|
  \le R_\theta(F)R_\xvec(F).
\]
Hence
\[
  \Big|\frac{\partial \enc_\theta(\xvec)}{\partial A_{r\ell}}\Big|
  \le C_\sigma\,R_\xvec(F),
  \qquad
  \Big|\frac{\partial \enc_\theta(\xvec)}{\partial G_{r\ell}}\Big|
  \le C'_\sigma\,R_\theta(F)R_\xvec(F)^2.
\]

The gradient $\nabla_\theta \enc_\theta(\xvec)\in\R^P$ collects all these partial derivatives,
so its Euclidean norm satisfies
\[
  \|\nabla_\theta \enc_\theta(\xvec)\|_2^2
  \;\le\;
  P\cdot\big(\max\{C_\sigma R_\xvec(F),\,C'_\sigma R_\theta(F)R_\xvec(F)^2\}\big)^2
  \;\le\;
  C\,\poly(F)^2
\]
for some constant $C>0$, using $P\le\poly(F)$ and $R_\xvec(F),R_\theta(F)\le\poly(F)$.
Thus there exists $L(F)\le C^{1/2}\poly(F)$ such that
\[
  \|\nabla_\theta \enc_\theta(\xvec)\|_2 \;\le\; L(F)
  \quad\text{for all }\|\theta\|_2\le R_\theta(F).
\]

For any $\theta,\theta'$ with $\|\theta\|_2,\|\theta'\|_2\le R_\theta(F)$, the mean value
inequality in $\R^P$ yields
\[
  |\enc_\theta(\xvec) - \enc_{\theta'}(\xvec)|
  \;\le\;
  \sup_{\tilde\theta\text{ on the segment }[\theta,\theta']}
  \|\nabla_\theta \enc_{\tilde\theta}(\xvec)\|_2
  \cdot \|\theta-\theta'\|_2
  \;\le\;
  L(F)\,\|\theta-\theta'\|_2.
\]
Since $L(F)\le\poly(F)$ by construction, this proves the claim.
\end{proof}

\begin{lemma}[Encoder weight norm bound]
Fix an output coordinate $j$ and consider the linear system
\[
    \normDec\,\rva = \rvo,
\]
where $\normDec \in \R^{F\times dh}$ and $a = \mathrm{vec}(\rmA)\in\R^{dh}$. Assume:
\begin{enumerate}[label=(\roman*)]
    \item The $i$-th row of $\normDec$ is
    \[
        \rvr_i^\top
        = \big(\sigma(\gvec_1^\top \rvk_i)\rvk_i^\top,\dots,\sigma(\gvec_h^\top \rvk_i)\rvk_i^\top\big),
    \]
    where $\{\rvk_i\}_{i=1}^F$ and $\{\gvec_\ell\}_{\ell=1}^h$ are independent subgaussian random vectors in $\R^d$, and
    $\sigma$ is analytic and non-constant.
    \item The covariance $\Sigma_{\mathrm{row}} := \E[\rvr_i \rvr_i^\top]$ satisfies
    $\lambda_{\min}(\Sigma_{\mathrm{row}})\ge\lambda_0>0$ and
    $\lambda_{\max}(\Sigma_{\mathrm{row}})\le \Lambda_0<\infty$, with $\lambda_0,\Lambda_0$ independent of $F$.
    \item The targets $o\in\R^F$ obey $|o_i|\le B(F)$ for all $i$, where $B(F)\le \poly(F)$.
    \item $F \ge C_0\,dh$ for a sufficiently large absolute constant $C_0$.
\end{enumerate}
Let $\avec_\star$ be the minimum–$\ell_2$–norm solution of $\normDec a=o$ (i.e.\ $\avec_\star = \normDec^\dagger o$). Then
\[
    \|\avec_\star\|_2 \;\le\; \poly(F).
\]
\end{lemma}

\begin{proof}
Let $\tilde \rvr_i := \Sigma_{\mathrm{row}}^{-1/2} \rvr_i$ and let $\tilde \normDec\in\R^{F\times dh}$ have rows
$\tilde \rvr_i^\top$. By construction, the rows of $\tilde \normDec$ are independent, isotropic, subgaussian
random vectors in $\R^{dh}$, and $\|\tilde \rvr_i\|_{\psi_2}$ is bounded uniformly in $F$.

Apply \Cref{sub-gaussian-rows} to $\tilde \normDec$ with
$N=F$ and $n=dh$. There exist constants $c,C>0$ depending only on the subgaussian norm such that,
with probability at least $1-2\exp(-c t^2)$,
\[
    \sqrt{F} - C\sqrt{dh} - t
    \;\le\; s_{\min}(\tilde \normDec)
    \;\le\; s_{\max}(\tilde \normDec)
    \;\le\; \sqrt{F} + C\sqrt{dh} + t
    \quad\forall t\ge 0.
\]
Choose $t=\sqrt{F}/4$ and use the assumption $F\ge C_0 dh$ with $C_0$ large enough to obtain
\[
    s_{\min}(\tilde \normDec) \;\ge\; c_1 \sqrt{F}
\]
for some constant $c_1>0$, with probability at least $1-\exp(-c_2 F)$.

Since $\normDec = \tilde \normDec\,\Sigma_{\mathrm{row}}^{1/2}$, we have
\[
    s_{\min}(\normDec)
    \;\ge\; \sqrt{\lambda_{\min}(\Sigma_{\mathrm{row}})}\, s_{\min}(\tilde \normDec)
    \;\ge\; \sqrt{\lambda_0}\, c_1 \sqrt{F}
    \;=\; c_3 \sqrt{F}.
\]
Furthermore,
\[
    \|\rvo\|_2^2 \;=\; \sum_{i=1}^F \rvo_i^2 \;\le\; F\,B(F)^2,
    \qquad\Rightarrow\qquad
    \|\rvo\|_2 \;\le\; \sqrt{F}\,B(F) \;\le\; \poly(F).
\]

Let $\avec_\star$ be the minimum–norm solution $\normDec \rva=\rvo$, so $\avec_\star = \normDec^\dagger o$ and
$\|\normDec^\dagger\|_{\mathrm{op}} = 1/s_{\min}(\normDec)$. Then
\[
    \|\avec_\star\|_2
    \;=\; \|\normDec^\dagger \rvo\|_2
    \;\le\; \|\normDec^\dagger\|_{\mathrm{op}} \,\|\rvo\|_2
    \;=\; \frac{\|\rvo\|_2}{s_{\min}(\normDec)}
    \;\le\; \frac{\sqrt{F}\,B(F)}{c_3 \sqrt{F}}
    \;=\; \frac{B(F)}{c_3}
    \;\le\; \poly(F).
\]
This holds for each output coordinate $j$, and stacking the corresponding vectors $\avec_\star^{(j)}$
over $m=\poly(F)$ coordinates preserves a $\poly(F)$ bound on the encoder parameter norm.
\end{proof}

\begin{lemma}[Row covariance is well-conditioned under rotationally invariant model]
\label{lem:row_cov_polyF}
Fix $d,h \in \N$ and let
\[
    \rvk \in \R^d
    \quad\text{and}\quad
    \gvec_1,\dots,\gvec_h \in \R^d
\]
be random vectors such that:
\begin{enumerate}[label=(\roman*)]
    \item $\rvk$ has a rotationally invariant distribution with
    \(\E[k]=0\) and \(\E[\rvk\rvk^\top] = \frac{1}{d} \rmI_d\);
    \item $\gvec_1,\dots,\gvec_h$ are i.i.d.\ $\mathcal N(0,\rmI_d/d)$, independent of $\rvk$;
    \item $\sigma:\R\to\R$ is a non-constant measurable function with
    \(\E[\sigma(\gvec_1^\top \rvk)^2] < \infty\).
\end{enumerate}
Define the random row vector $\rvr^\top \in \R^{dh}$ by
\[
    \rvr^\top
    :=
    \big(
        \sigma(\gvec_1^\top \rvk) \rvk^\top,\;
        \dots,\;
        \sigma(\gvec_h^\top \rvk) \rvk^\top
    \big),
\]
and let
\[
    \Sigma_{\mathrm{row}} := \E[ \rvr \rvr^\top ] \in \R^{dh\times dh}.
\]
Then there exists a constant $c>0$, depending only on the distributions of $\rvk$, $\gvec_\ell$, and $\sigma$ (but \emph{independent} of $F$), such that
\[
    \lambda_{\min}(\Sigma_{\mathrm{row}}) = c.
\]
In particular,
\[
    \lambda_{\min}(\Sigma_{\mathrm{row}}) \;\ge\; F^{-C}
\]
for some fixed exponent $C$ and all $F$ (i.e., the lower bound is $\poly(F)$).
\end{lemma}

\begin{proof}
For any orthogonal $\rmU \in O(d)$, define a block-rotation $
\rmT_\rmU:\R^{dh}\to\R^{dh}$ by
\[
\rmT_\rmU(\xvec_1,\dots,\xvec_h) := (\rmU\xvec_1,\dots, \rmU\xvec_h), \qquad \xvec_\ell \in \R^d.
\]
By rotational invariance of $\rvk$ and Gaussianity of $\gvec_\ell$, we have
\[
    (\rvk,\gvec_1,\dots,\gvec_h) \;\sim \; (\rmU \rvk,\rmU\gvec_1,\dots,\rmU\gvec_h),
\]
and a direct calculation shows
\[
    \rvr(\rmU k,\rmU\gvec_1,\dots,\rmU\gvec_h) = \rmT_\rmU \, \rvr(\rvk,\gvec_1,\dots,\gvec_h).
\]
Hence $\rvr \sim \rmT_\rmU \rvr$ for all $\rmU \in O(d)$.
Taking expectations,
\[
    \rmT_\rmU \Sigma_{\mathrm{row}} \rmT_\rmU^\top
    = \E[\rmT_\rmU \rvr \rvr^\top \rmT_\rmU^\top]
    = \E[\rvr \rvr^\top]
    = \Sigma_{\mathrm{row}},
    \qquad \forall \rmU \in O(d).
\]
Thus $\Sigma_{\mathrm{row}}$ commutes with every block-rotation $\rmT_\rmU$.
By Schur's lemma / symmetry, the only matrices with this property are scalar multiples of the identity, so
\[
    \Sigma_{\mathrm{row}} = c \rmI_{dh}
\]
for some $c\ge 0$.
Since $\sigma$ is non-constant and $k,\gvec_\ell$ are non-degenerate, we have
\(\Var(\langle \rvr,\rvu\rangle) = \rvu^\top \Sigma_{\mathrm{row}} u > 0\) for some unit $u$, forcing $c>0$.
Therefore
\[
    \lambda_{\min}(\Sigma_{\mathrm{row}}) = c > 0,
\]
which is a positive constant independent of $F$, and hence trivially satisfies
$\lambda_{\min}(\Sigma_{\mathrm{row}}) \ge F^{-C}$ for some fixed $C$.
\end{proof}

\newpage

\subsubsection{Proof of \Cref{thm: encoder-bits-UB}}\label{pf: encoder-bits-UB}
\begin{proof}
    The full construction can be described as $g(\inVec) = \decMat\encMat (\sigma(\gaussMat \inVec) \odot (\gateMat \inVec) )$, where $\decMat \in \mathbb R^{\tokenDim \times \encDim}$, $\gateMat, \gaussMat \in \mathbb R^{\hiddenDim \times \tokenDim}$, $\encMat \in \mathbb R^{\encDim \times \hiddenDim}$ and $\inVec \in \mathbb R^\tokenDim$. A few of these we can bound easily. 
    \begin{enumerate}
        \item $\encMat$ is a matrix which contains just 1s, and thus contributes $\encDim \hiddenDim$ bits. 
        \item We will show in \Cref{rad_noisy_decoding} that $\decMat$ is a matrix which can be stored with values in $\{-1, 1\}$, which means that it can be stored using $\tokenDim \encDim$ bits.  
        \item The matrices $\gaussMat$ and $\gateMat$ are not as easy to determine how many bits they take to store since these matrices can take on continuous values. We need to prove two things. First, we need to show that the parameters of $\gaussMat$ and $\gateMat$ are bounded. Since $\gaussMat$ has rows that are normal, the magnitude of the parameters of $\gaussMat$ are bounded with high probability by \Cref{bound_gates_and_keys}. It remains to be shown that the parameters of $\gateMat$ are bounded by $O(\text{poly} \numKeys)$. If this is true, then the integer part of the parameter can be represented by $O(\log \operatorname{poly} \numKeys) = O(\log \numKeys)$ bits. This is proved in \Cref{bounding_theorem}.
        \item Second, we will prove that the parameters of these two matrices can be stored with finite precision. That is, if we truncate the decimal expansion of the parameter values of each of the matrices after a certain number of places, the construction still works when each parameter only has $O(\log \numKeys)$ bits of precision. This is proved in \Cref{bounding_precision}.
    \end{enumerate}
    Combining all of these steps completes the proof. 
\end{proof}

\subsubsection{Proof of \Cref{rad_noisy_decoding}}\label{pf:rad_noisy_decoding}
\begin{proof}
Set $\Phi := \tfrac{1}{\sqrt m}\decMat$ and $\encMat := \Phi^\top\Phi - \rmI$.
For $i\neq j$, write
\[
\avec_{ij} := \vvec_i - \vvec_j,\qquad \bvec_i := \uvec_i.
\]
Let $\gaussRow{i} := \decMat \uvec_i$ and $\Delta_i := \gaussRow{i}\odot\nu_i$, so $\rmH[i] = \gaussRow{i} + \Delta_i$.
Then
\[
\rmM \rmH[i]
= \tfrac{1}{\encDim}\decMat^\top(\gaussRow{i} + \Delta_i)
= \Phi^\top\Phi\,\vb_i + \tfrac{1}{\encDim}\decMat^\top\Delta_i
= (\rmI+\encMat)\vb_i + \tfrac{1}{\encDim}\decMat^\top\Delta_i,
\]
and the score gap is
\begin{equation}\label{eq:gap}
s_{ii} - s_{ij}
= \langle \va_{ij}, \mMat \rmH[i]\rangle
= \va_{ij}^\top \vb_i
 + \va_{ij}^\top \encMat \vb_i
 + \frac{1}{\encDim}(\decMat \va_{ij})^\top\Delta_i.
\end{equation}

\paragraph{Margin term.}
By the definition of $\rho$,
\begin{equation}\label{eq:margin}
\va_{ij}^\top \vb_i
= \langle \vvec_i - \vvec_j, \uvec_i\rangle
\;\ge\; \rho\,\|\va_{ij}\|\,\|\vb_i\|
\quad\forall i\neq j.
\end{equation}

\paragraph{JL event (inner products and norms).}
Define
\[
\hat \avec_{ij} := \frac{\avec_{ij}}{\|\avec_{ij}\|},\qquad
\hat \bvec_i := \frac{\bvec_i}{\|\bvec_i\|},
\]
and consider the finite set of unit-vector pairs
\[
\mathcal P :=
\{(\hat \avec_{ij},\hat \bvec_i) : i\in[N],\, j\neq i\}
\;\cup\;
\{(\hat \inVec,\hat \inVec) : \inVec\in X\},
\]
where \[X := \{\avec_{ij}: i\neq j\}\cup\{\bvec_i : i\in[N]\}.\]
Since the rows of $\Phi$ are isotropic subgaussian (Rademacher),
the Johnson--Lindenstrauss lemma implies:

for $\eta := \rho/2$, if
\[
\encDim \;\ge\; \frac{C}{\rho^2}\,
          \ln\!\frac{4N(N-1)}{\delta},
\]
then with probability at least $1-\delta$,
\[
\big|\langle\Phi \inVec,\Phi y\rangle - \langle \inVec,\rvy\rangle\big|
\le \eta
\quad\forall (\inVec,\yvec)\in\mathcal P.
\]
Following from \Cref{jl-ip-result}.

On this event, we get:

(i) For $(\inVec,\yvec)=(\hat \avec_{ij},\hat \bvec_i)$,
\[
|\hat \avec_{ij}^\top \encMat\hat \bvec_i|
= \big|\langle\Phi\hat \avec_{ij},\Phi\hat \bvec_i\rangle
     - \langle\hat \avec_{ij},\hat \bvec_i\rangle\big|
\le \frac{\rho}{2},
\]
so
\begin{equation}\label{eq:E}
|\avec_{ij}^\top \encMat \bvec_i|
\le \frac{\rho}{2}\,\|\avec_{ij}\|\,\|\bvec_i\|
\quad\forall i\neq j.
\end{equation}

(ii) For $(\inVec,\yvec)=(\hat \inVec,\hat \inVec)$,
\[
|\|\Phi\hat \inVec\|^2 - 1|
= \big|\langle\Phi\hat \inVec,\Phi\hat \inVec\rangle - 1\big|
\le \frac{\rho}{2}
\le 1,
\]
so $\|\Phi\hat \inVec\|\le \sqrt{2} \le 2$ and hence
\begin{equation}\label{eq:norm}
\|\decMat \inVec\| = \sqrt \encDim\,\|\Phi \inVec/\|\inVec\|\|
\le 2\sqrt \encDim\,\|\inVec\|
\quad\forall \inVec\in X.
\end{equation}

\paragraph{Noise term.}
Since $|\nu_{i,k}|\le\varepsilon$, we have
\[
|\Delta_{i,k}| = |\gaussEle{i}{k}\nu_{i,k}|
\le \varepsilon\,|\gaussEle{i}{k}|,
\quad\Rightarrow\quad
\|\Delta_i\|\le\varepsilon\,\|\gaussRow{i}\|
= \varepsilon\,\|\decMat \bvec_i\|.
\]
By Cauchy--Schwarz and \eqref{eq:norm},
\[
\big|(\decMat \avec_{ij})^\top\Delta_i\big|
\le \|\decMat \avec_{ij}\|\,\|\Delta_i\|
\le \varepsilon\,\|\decMat \avec_{ij}\|\,\|\decMat \bvec_i\|
\le \varepsilon\,(2\sqrt \encDim\,\|\avec_{ij}\|)(2\sqrt \encDim\,\|\bvec_i\|),
\]
so
\begin{equation}\label{eq:noise}
\Big|\frac{1}{m}(\decMat \avec_{ij})^\top\Delta_i\Big|
\le 4\varepsilon\,\|\avec_{ij}\|\,\|\bvec_i\|
\quad\forall i\neq j.
\end{equation}

\paragraph{Conclusion.}
On the JL event, combining \eqref{eq:margin}, \eqref{eq:E}, and
\eqref{eq:noise} in \eqref{eq:gap} gives, for all $i\neq j$,
\[
\begin{aligned}
s_{ii} - s_{ij}
&\ge
\rho\,\|\avec_{ij}\|\,\|\bvec_i\|
 - \frac{\rho}{2}\,\|\avec_{ij}\|\,\|\bvec_i\|
 - 4\varepsilon\,\|\avec_{ij}\|\,\|\bvec_i\| \\
&=
\Big(\frac{\rho}{2} - 4\varepsilon\Big)\,
\|\avec_{ij}\|\,\|\bvec_i\|.
\end{aligned}
\]
Since $\avec_{ij} = \vvec_i - \vvec_j$ and $\bvec_i = \uvec_i$, this is exactly
\[
s_{ii} - s_{ij}
\;\ge\;
\Big(\frac{\rho}{2} - 4\varepsilon\Big)\,
\|\vvec_i - \vvec_j\|\,\|\uvec_i\|,
\]
as claimed.
\end{proof}

\subsubsection{Proof of \Cref{lm: mag-prob-LB}} \label{pf: mag-prob-LB}
\begin{proof}
    When the keys are Gaussian, $\key_i \sim \mathcal N(0,\rmI_\tokenDim)$, we have $\|\key_i\|_2^2 \sim \chi^2_\tokenDim$ and standard
concentration implies
\[
\Pr\big(\|\key_i\|_2 \ge \sqrt \tokenDim + t\big) \le \exp(-c t^2)
\quad\forall t\ge 0
\]
for some absolute constant $c>0$.(See \href{https://www.math.uci.edu/~rvershyn/papers/HDP-book/HDP-1.pdf?utm_source=chatgpt.com}{Theorem 3.1.1}) By a union bound,
\[
\Pr\Big(\max_{1\le i\le \numKeys} \|\key_i\|_2 \ge \sqrt d + t\Big)
\le \numKeys \exp(-c t^2).
\]
Taking $t = \sqrt{C\log \numKeys}$ with $C$ large enough, we obtain
\[
\max_{1\le i\le \numKeys} \|\key_i\|_2 \le \sqrt \tokenDim + \sqrt{C\log \numKeys}
\]
with probability at least $1 - \numKeys^{-\Omega(1)}$. Thus, defining
$R_\inVec(\numKeys) := \sqrt \tokenDim + \sqrt{C\log \numKeys}$ and assuming $\tokenDim \le \poly(\numKeys)$, we have $R_\inVec(\numKeys)\le \poly(\numKeys)$, so the
deterministic assumption $\|\key_i\|_2 \le R_\inVec(\numKeys)$ for all $i$ holds with high probability.
\end{proof}

\subsubsection{Proof of \Cref{lem:row_cov_rot_inv}} \label{pf: row_cov_rot_inv}
\begin{proof}
For any orthogonal $\orthoMat \in O(\tokenDim):=\{\rmV \in \mathbb R^{\tokenDim \times \tokenDim}: \rmV^\top \rmV = \rmI_\tokenDim\}$, define a block-rotation $T_{\orthoMat}:\R^{\tokenDim \hiddenDim}\to\R^{\tokenDim \hiddenDim}$ by
\[
    T_{\orthoMat}(\inVec_1,\dots,\inVec_\hiddenDim) := (\orthoMat\inVec_1,\dots, \orthoMat\inVec_\hiddenDim), \qquad \inVec_\ell \in \R^\tokenDim.
\]
By rotational invariance of $\key$ and $\gaussRow{\ell}$, we have
\[
    (\key,\gaussRow{1},\dots,\gaussRow{\hiddenDim}) \;\sim\; (\orthoMat\key,\orthoMat \gaussRow{1},\dots,\orthoMat \gaussRow{\hiddenDim}),
\]
and a direct calculation\footnote{\[
\randVec(\orthoMat\key, \orthoMat \gaussRow{1}, \ldots, \orthoMat \gaussRow{\hiddenDim})
  = \bigl(\sigma((\orthoMat \gaussRow{1})^\top \orthoMat\key)(\orthoMat\key)^\top,\ldots,
          \sigma((\orthoMat \gaussRow{\hiddenDim})^\top \orthoMat\key)(\orthoMat\key)^\top\bigr).
\]
Since $\orthoMat$ is orthogonal, $(\orthoMat \gaussRow{\ell})^\top \orthoMat\key = \gaussRow{\ell}^\top \orthoMat^\top \orthoMat\key = \gaussRow{\ell}^\top \key$, so this becomes
\[
\randVec(\orthoMat\key, \orthoMat \gaussRow{1}, \ldots, \orthoMat \gaussRow{\hiddenDim})
  = \bigl(\sigma(\gaussRow{1}^\top \key)(\orthoMat\key)^\top,\ldots,
          \sigma(\gaussRow{\hiddenDim}^\top \key)(\orthoMat\key)^\top\bigr).
\]

On the other hand, applying $T_{\orthoMat}$ to 
$ \randVec(\key, \gaussRow{1}, \ldots, \gaussRow{\hiddenDim})
  = \bigl(\sigma(\gaussRow{1}^\top \key)\key^\top,\ldots,
          \sigma(\gaussRow{\hiddenDim}^\top \key)\key^\top\bigr)$
clearly gives the same result, so the two expressions coincide.
} 
shows
\[
    \randVec(\orthoMat\key,\orthoMat \gaussRow{1},\dots,\orthoMat \gaussRow{\hiddenDim}) = T_{\orthoMat} \, \randVec(\key,\gaussRow{1},\dots,\gaussRow{\hiddenDim}).
\]

Taking expectations,
\[
    T_{\orthoMat} \Sigma_{\mathrm{row}} T_{\orthoMat}^\top
    = \E[T_{\orthoMat} \randVec \randVec^\top T_{\orthoMat}^\top]
    = \E[\randVec \randVec^\top]
    = \Sigma_{\mathrm{row}},
    \qquad \forall \orthoMat \in O(\tokenDim).
\]
Thus $\Sigma_{\mathrm{row}}$ commutes with every block-rotation $T_{\orthoMat}$.

Looking at the $(i,j)$ block of this identity $\aMat_{ij} \in \mathbb R^{d \times d}$ yields
\[
    \orthoMat \aMat_{ij} \orthoMat^\top = \aMat_{ij},\qquad \forall \orthoMat\in O(d).
\tag{1}
\]

\smallskip\noindent
\emph{Step 1: form of $\aMat_{ij}$.}
Let $M\in\R^{d\times d}$ be symmetric and satisfy $\orthoMat\mMat\orthoMat^\top=\mMat$ for all $\orthoMat\in O(d)$. Then, it is a well known result that $\mMat=\lambda \rmI_d$\footnote{Theorem A.4 in \citep{kotelenez2008depletioneffectcolloidscorrelated}}.

Applying this to each symmetric $\aMat_{ij}$ in (1) gives
\[
    \aMat_{ij}=\lambda_{ij} \rmI_d
    \quad\text{for some }\lambda_{ij}\in\R.
\tag{2}
\]

\smallskip\noindent
\emph{Step 2: diagonal blocks.}
Since the $\rvg_\ell$ are i.i.d., each $\rvr_i$ has the same distribution, so
$\aMat_{11}=\cdots=\aMat_{hh}=c\rmI_d$ for some $c\ge 0$.
Moreover,
\[
    c \rmI_d = \aMat_{11}
    = \E[\rvr_1 \rvr_1^\top]
    = \E\big[\sigma(\rvg_1^\top \rvk)^2\,\rvk\rvk^\top\big],
\]
and by non-degeneracy of $(\rvk,\rvg_1)$ and non-constancy of $\sigma$ we have
$ \E[\sigma(\rvg_1^\top \rvk)^2\|\rvk\|_2^2]>0$, so $c>0$.

\smallskip\noindent
\emph{Step 3: off-diagonal blocks vanish.}
For $i\neq j$,
\[
    \aMat_{ij}
    = \E[\sigma(\rvg_i^\top \rvk)\sigma(\rvg_j^\top \rvk)\,\rvk\rvk^\top].
\]
Conditioning on $k$ and using
\(
    \E(f(Z)Y\mid Z)=f(Z)\E(Y\mid Z)
\),
we obtain
\[
    \aMat_{ij}
    = \E\Big[
        \rvk\rvk^\top
        \E\big[\sigma(\rvg_i^\top \rvk)\sigma(\rvg_j^\top \rvk)\mid \rvk\big]
    \Big].
\]
Given $\rvk$, the vectors $\rvg_i,\rvg_j$ are independent and identically distributed, hence
\[
    \E\big[\sigma(\rvg_i^\top \rvk)\sigma(\rvg_j^\top \rvk)\mid \rvk\big]
    = \E[\sigma(\rvg_1^\top \rvk)\mid \rvk]^2.
\]
Let $\lambda(\rvk):=\E[\sigma(\rvg_1^\top \rvk)\mid \rvk]$. Assumption (iv) gives
$\lambda(\rvk)=0$ a.s., so $\lambda(\rvk)^2=0$ a.s. and therefore
\[
    \aMat_{ij}
    = \E[\rvk\rvk^\top\,\lambda(\rvk)^2] = 0,
    \qquad i\neq j.
\tag{3}
\]

\smallskip
Combining (2), (3), and the identification of the diagonal blocks,
\[
    \Sigma_{\mathrm{row}}
    = \mathrm{diag}(c\rmI_d,\dots,c\rmI_d)
    = c\,\rmI_{dh},
\]
so all eigenvalues of $\Sigma_{\mathrm{row}}$ equal $c>0$.
\end{proof}

\subsubsection{Proof of \Cref{bounding_theorem}}
\begin{proof} \label{pf: bounding_theorem}
Let $\tilde \randVec_i := \Sigma_{\mathrm{row}}^{-1/2} \randVec_i$ and let $\tilde \mMat\in\R^{\numKeys\times \tokenDim \hiddenDim}$ have rows
$\tilde \randVec_i^\top$. By construction, the rows of $\tilde \mMat$ are independent, isotropic, subgaussian
random vectors in $\R^{\tokenDim \hiddenDim}$, and $\|\tilde \randVec_i\|_{\psi_2}$ is bounded uniformly\footnote{Subgaussianity is preserved under linear
maps: for any $\uvec\in\R^{\tokenDim \hiddenDim}$,
$\langle \uvec,\tilde \randVec_i\rangle = \langle \Sigma_{\mathrm{row}}^{-1/2}{}^{\!\top}\uvec, \randVec_i\rangle$ is subgaussian with
$\|\langle \uvec,\tilde \randVec_i\rangle\|_{\psi_2} \le K \|\Sigma_{\mathrm{row}}^{-1/2}\|_{\text{op}}\|\uvec\|_2$, where
$K$ bounds $\|\randVec_i\|_{\psi_2}$. By assumption 2,
$\|\Sigma_{\mathrm{row}}^{-1/2}\|_{\text{op}} = 1/\sqrt{\lambda_{\min}(\Sigma_{\mathrm{row}})} \le 1/\sqrt{\lambda_0}$,
so $\|\tilde \randVec_i\|_{\psi_2} \lesssim K/\sqrt{\lambda_0}$, a constant independent of $\numKeys$.}
 in $\numKeys$.

Apply \Cref{sub-gaussian-rows} to $\tilde \mMat$ with
$N=\numKeys$ and $n=\tokenDim \hiddenDim$. There exist constants $c,C>0$ depending only on the subgaussian norm such that,
with probability at least $1-2\exp(-c t^2)$,
\[
    \sqrt{\numKeys} - C\sqrt{\tokenDim \hiddenDim} - t
    \;\le\; s_{\min}(\tilde \mMat)
    \;\le\; s_{\max}(\tilde \mMat)
    \;\le\; \sqrt{\numKeys} + C\sqrt{\tokenDim \hiddenDim} + t
    \quad\forall t\ge 0.
\]
Choose $t=\sqrt{\numKeys}/4$ and use the assumption $\numKeys\ge C_0 \tokenDim \hiddenDim$ with $C_0$ large enough to obtain
\[
    s_{\min}(\tilde \mMat) \;\ge\; c_1 \sqrt{\numKeys}
\]
for some constant $c_1>0$, with probability at least $1-\exp(-c_2 \numKeys)$.

Since $\mMat = \tilde \mMat\,\Sigma_{\mathrm{row}}^{1/2}$, we have
\[
    s_{\min}(\mMat)
    \;\ge\; \sqrt{\lambda_{\min}(\Sigma_{\mathrm{row}})}\, s_{\min}(\tilde \mMat)
    \;\ge\; \sqrt{\lambda_0}\, c_1 \sqrt{\numKeys}
    \;=\; c_3 \sqrt{\numKeys}.
\]
Furthermore,
\[
    \|\target\|_2^2 \;=\; \sum_{i=1}^\numKeys \target_i^2 \;\le\; \numKeys\,B(\numKeys)^2,
    \qquad\Rightarrow\qquad
    \|\target\|_2 \;\le\; \sqrt{\numKeys}\,B(\numKeys) \;\le\; \poly(\numKeys).
\]

Let $\va_\star$ be the minimum–norm solution $\mMat a=\target$, so $\va_\star = \mMat^\dagger \target$ and
$\|\mMat^\dagger\|_{\mathrm{op}} = 1/s_{\min}(\mMat)$. Then
\[
    \|\va_\star\|_2
    \;=\; \|\mMat^\dagger \target\|_2
    \;\le\; \|\mMat^\dagger\|_{\mathrm{op}} \,\|\target\|_2
    \;=\; \frac{\|\target\|_2}{s_{\min}(\mMat)}
    \;\le\; \frac{\sqrt{\numKeys}\,B(\numKeys)}{c_3 \sqrt{\numKeys}}
    \;=\; \frac{B(\numKeys)}{c_3}
    \;\le\; \poly(\numKeys).
\]
This holds for each output coordinate $j$, and stacking the corresponding vectors $a_\star^{(j)}$
over $m=\poly(\numKeys)$ coordinates preserves a $\poly(\numKeys)$ bound on the encoder parameter norm.
\end{proof}

\subsubsection{Proof of \Cref{lm: enc_lipschitz}}
\begin{proof} \label{pf: enc_lipshitz}
Fix $i$ and write $\inVec := \key_i$. For fixed $\inVec$, view $\enc_\theta(\inVec)$ as a function
$\R^P\to\R$ of the parameter vector $\theta$. Its partial derivatives are, for each
$r\in[\hiddenDim]$ and $\ell\in[\tokenDim]$,
\[
  \frac{\partial \enc_\theta(\inVec)}{\partial \gateMat_{r\ell}}
  = \sigma(\langle \gaussRow{r},\inVec\rangle)\,\inVec_\ell,
  \qquad
  \frac{\partial \enc_\theta(\inVec)}{\partial \gaussMat_{r\ell}}
  = \sigma'(\langle \gaussRow{r},\inVec\rangle)\,\langle \avec_r,\inVec\rangle\,\inVec_\ell.
\]
On the parameter ball $\|\theta\|_2\le R_\theta(\numKeys)$ and with $\|\inVec\|\le R_\inVec(\numKeys)$ we have
$|\langle \gaussRow{r},\inVec\rangle| \le \|\gaussRow{r}\|\|\inVec\|\le R_\theta(\numKeys)R_\inVec(\numKeys) = B(\numKeys)$, so by assumption
$|\sigma(\langle \gaussRow{r},\inVec\rangle)| \le C_\sigma$ and
$|\sigma'(\langle \gaussRow{r},\inVec\rangle)| \le C'_\sigma$.
Moreover $|\inVec_\ell|\le R_\inVec(\numKeys)$ and
\[
  |\langle \avec_r,\inVec\rangle|
  \le \|\avec_r\|\,\|\inVec\|
  \le R_\theta(\numKeys)R_\inVec(\numKeys).
\]
Hence
\[
  \Big|\frac{\partial \enc_\theta(\inVec)}{\partial \gateMat_{r\ell}}\Big|
  \le C_\sigma\,R_\inVec(\numKeys),
  \qquad
  \Big|\frac{\partial \enc_\theta(\inVec)}{\partial \gaussMat_{r\ell}}\Big|
  \le C'_\sigma\,R_\theta(\numKeys)R_\inVec(\numKeys)^2.
\]

The gradient $\nabla_\theta \enc_\theta(\inVec)\in\R^P$ collects all these partial derivatives,
so its Euclidean norm satisfies
\[
  \|\nabla_\theta \enc_\theta(\inVec)\|_2^2
  \;\le\;
  P\cdot\big(\max\{C_\sigma R_\inVec(\numKeys),\,C'_\sigma R_\theta(\numKeys)R_\inVec(\numKeys)^2\}\big)^2
  \;\le\;
  C\,\poly(\numKeys)^2
\]
for some constant $C>0$, using $P\le\poly(\numKeys)$ and $R_\inVec(\numKeys),R_\theta(\numKeys)\le\poly(\numKeys)$.
Thus there exists $L(\numKeys)\le C^{1/2}\poly(\numKeys)$ such that
\[
  \|\nabla_\theta \enc_\theta(\inVec)\|_2 \;\le\; L(\numKeys)
  \quad\text{for all }\|\theta\|_2\le R_\theta(\numKeys).
\]

For any $\theta,\theta'$ with $\|\theta\|_2,\|\theta'\|_2\le R_\theta(\numKeys)$, the mean value
inequality in $\R^P$ yields
\[
  |\enc_\theta(\inVec) - \enc_{\theta'}(\inVec)|
  \;\le\;
  \sup_{\tilde\theta \in [\theta,\theta']}
  \|\nabla_\theta \enc_{\tilde\theta}(\inVec)\|_2
  \cdot \|\theta-\theta'\|_2
  \;\le\;
  L(\numKeys)\,\|\theta-\theta'\|_2.
\]
Since $L(\numKeys)\le\poly(\numKeys)$ by construction, this proves the claim.
TODO: Cite ??? to show that assumption 4 holds. 
\end{proof}

\subsubsection{Proof of \Cref{bounding_precision}}
\begin{proof} \label{pf: bounding_precision}
\textbf{Step 1: Allowed code noise.}
From \Cref{rad_noisy_decoding}, there is a constant $c_0>0$ such that,
if the code for fact $i$ is perturbed by at most $c_0\rho$ in
an appropriate sense (as in the theorem’s proof), then the score margin
remains positive:
\[
    s_{ii} - s_{ij} \;\ge\; \Omega(\rho)\,\|\vvec_i-\vvec_j\|\,\|\uvec_i\|.
\]
Thus the encoder codes are robust to perturbations of size
$\Theta(\rho)$. Using (i), we have
\[
    \rho \;\ge\; \frac{1}{\poly(\numKeys)},
\]
so the allowed code noise is at least $1/\poly(\numKeys)$.

\medskip
\noindent
\textbf{Step 2: From parameter perturbation to code perturbation.}
Let $\theta_\star$ be the ideal encoder parameters and
$\tilde\theta$ any other parameter vector. For each key $\key_i$,
define the code perturbation
\[
    \Delta_i := \enc_{\tilde\theta}(\key_i) - \enc_{\theta_\star}(\key_i).
\]
By the Lipschitz assumption (ii),
\[
    \|\Delta_i\|
    =
    \|\enc_{\tilde\theta}(\key_i) - \enc_{\theta_\star}(\key_i)\|
    \;\le\;
    L(\numKeys)\,\|\tilde\theta - \theta_\star\|
    \quad\forall i.
\]
To keep the codes within the robustness radius from Step~1, it suffices
to impose
\[
    \|\Delta_i\| \;\le\; c_0\rho
    \quad\forall i.
\]
A sufficient condition is therefore
\[
    \|\tilde\theta - \theta_\star\|
    \;\le\;
    \delta(\numKeys)
    := \frac{c_0\rho}{L(\numKeys)}.
\]
Using (i) and (ii), we obtain
\[
    \delta(\numKeys)
    \;\ge\;
    \frac{c_0}{\poly(\numKeys)\,\poly(\numKeys)}
    \;=\;
    \frac{1}{\poly(\numKeys)}.
\]
So there is a ball of radius at least $1/\poly(\numKeys)$ around $\theta_\star$
in parameter space such that any $\tilde\theta$ in this ball produces
codes that \Cref{rad_noisy_decoding} can tolerate.

\medskip
\noindent
\textbf{Step 3: Quantization and choice of grid size.}
Now quantize each coordinate of $\theta_\star$ to a grid of step size
$\Delta>0$, obtaining $\tilde\theta$. Each coordinate changes by at most
$\Delta/2$, so
\[
    \|\tilde\theta - \theta_\star\|_2
    \;\le\;
    \sqrt{P}\,\frac{\Delta}{2}.
\]
To guarantee $\|\tilde\theta - \theta_\star\|\le\delta(\numKeys)$, it is enough
to choose $\Delta$ so that
\[
    \sqrt{P}\,\frac{\Delta}{2}
    \;\le\;
    \delta(\numKeys)
    \quad\Longleftrightarrow\quad
    \Delta \;\le\; \frac{2\,\delta(\numKeys)}{\sqrt{P}}.
\]
Using $\delta(\numKeys)\ge 1/\poly(\numKeys)$ and $P\le\poly(\numKeys)$ from (iii), we get
\[
    \frac{2\,\delta(\numKeys)}{\sqrt{P}}
    \;\ge\;
    \frac{1}{\poly(\numKeys)}.
\]
Thus the admissible step size $\Delta$ can be as large as
$1/\poly(\numKeys)$. In particular, we may pick
\[
    \Delta := \numKeys^{-c}
\]
for some constant $c>0$ large enough so that $\Delta\le 2\delta(\numKeys)/\sqrt{P}$.
This ensures $\|\tilde\theta - \theta_\star\|\le\delta(\numKeys)$ and, by
Step~2, that the induced code perturbations are within the noise budget
of \Cref{rad_noisy_decoding}. Hence decoding remains correct.

\medskip
\noindent
\textbf{Step 4: Bit complexity.}
By \cref{prec-assm-mag}, each parameter lies in an interval of length at most
$\text{range} \le 2\,\poly(\numKeys)$. With grid spacing
$\Delta = \numKeys^{-c} = 1/\poly(\numKeys)$, the number of distinct values per
parameter is at most
\[
    \frac{\text{range}}{\Delta}
    \;\le\;
    \frac{\poly(\numKeys)}{1/\poly(\numKeys)}
    = \poly(\numKeys).
\]
Therefore the number of bits per parameter is
\[
    \log_2\Big(\frac{\text{range}}{\Delta}\Big)
    = O(\log \poly(\numKeys))
    = O(\log \numKeys).
\]
This proves that encoder parameters require only $O(\log \numKeys)$ bits of
precision.

\end{proof}

\end{document}